\documentclass[letterpaper]{article} 
\usepackage[preprint]{aaai2027}
\usepackage[hyphens]{url}  
\usepackage{graphicx} 
\usepackage{natbib}  
\usepackage{caption} 
\usepackage{algorithm}

\usepackage{newfloat}
\usepackage{listings}
\DeclareCaptionStyle{ruled}{labelfont=normalfont,labelsep=colon,strut=off} 
\floatstyle{ruled}
\newfloat{listing}{tb}{lst}{}
\floatname{listing}{Listing}

\usepackage{booktabs}

\usepackage{amsmath}
\usepackage{amssymb}
\usepackage{amsfonts}
\usepackage{booktabs}  
\usepackage{multirow}  
\usepackage{graphicx}  
\usepackage{multirow}
\usepackage{booktabs}
\usepackage{threeparttable}
\usepackage{siunitx}
\usepackage{algorithm}
\usepackage{algpseudocode}
\usepackage{amsmath,amssymb} 

\usepackage{xcolor}   
\usepackage{colortbl} 
\definecolor{TopOne}{RGB}{91,155,213}      
\definecolor{TopTwo}{RGB}{221,235,247}     
\definecolor{TopThree}{RGB}{217,217,217}   

\newcommand{\rankone}[1]{\cellcolor{TopOne}{#1}}
\newcommand{\ranktwo}[1]{\cellcolor{TopTwo}{#1}}
\newcommand{\rankthree}[1]{\cellcolor{TopThree}{#1}}

\usepackage{amsthm}
\usepackage{mathrsfs}
\newtheorem{assumption}{Assumption}[section]
\newtheorem{definition}[assumption]{Definition}
\newtheorem{lemma}[assumption]{Lemma}
\newtheorem{proposition}[assumption]{Proposition}
\newtheorem{theorem}[assumption]{Theorem}
\newtheorem{corollary}[assumption]{Corollary}
\theoremstyle{remark}
\newtheorem{remark}[assumption]{Remark}

\usepackage{booktabs}
\usepackage{array}
\usepackage{xurl}
\usepackage{tabularx}

\newcolumntype{L}[1]{>{\raggedright\arraybackslash}p{#1}}
\newcolumntype{C}{>{\centering\arraybackslash}X}
\newcommand{\smiles}[1]{\path{#1}}

\title{Don't Retrain, Just Reuse: \\Recovering Dual-Target Molecules from Single-Target Diffusion Models}
\author{
    Qingyuan Zeng\textsuperscript{\rm 1},
    Pengxiang Cai\textsuperscript{\rm 1},
    Zixin Guan\textsuperscript{\rm 3},
    Ziyang Chen\textsuperscript{\rm 1},\\
    Anglin Liu\textsuperscript{\rm 1},
    Xinyao LAI\textsuperscript{\rm 1},
    Jintai Chen\textsuperscript{\rm 1,\rm 2}
}
\affiliations{
    \textsuperscript{\rm 1}The Hong Kong University of Science and Technology (Guangzhou), Guangzhou, China\\
    \textsuperscript{\rm 2}The Hong Kong University of Science and Technology, Hong Kong, China\\
    \textsuperscript{\rm 3}Guangzhou University of Chinese Medicine, Guangzhou, China
}

\begin{document}

\maketitle

\begin{abstract}
Designing a single molecule that modulates two targets is a promising strategy for polypharmacology, but it remains substantially harder than standard single-target generation because one candidate must satisfy two binding requirements while preserving drug-likeness and synthesizability. Existing dual-target generative methods typically introduce dual-target capability by either retraining the generator or intervening in the diffusion process during sampling. The former can be costly and difficult to stabilize when dual-target supervision is sparse, while the latter may be sensitive to denoising-time target balancing and competing update directions. These limitations motivate a generator-preserving alternative that keeps the pretrained prior intact: \textit{can dual-target candidates instead be recovered from the input space of a frozen single-target diffusion model, without modifying its parameters or denoising dynamics?} We formulate this task as a constrained multi-objective optimization problem and propose REUSE, which evolves the input noise of a frozen diffusion generator rather than molecular structures. Each input is decoded multiple times and scored by the collective quality of the generated molecular family. Candidates are then screened progressively: lower-cost evaluations prioritize molecules satisfying chemical-feasibility criteria, full docking is reserved for a reduced frontier, and the survivors are jointly selected as a diverse panel with strong affinity to both targets. Experiments show that REUSE achieves stronger and more balanced dual-target recovery than prior dual-target baselines, improving Dual High Affinity by 21.1 percentage points over the strongest prior baseline while retaining QED and SA profiles consistent with commonly used chemical-feasibility criteria.
\end{abstract}

\section{Introduction}

Designing a single molecule that can modulate two targets is an important goal in drug discovery, especially for complex diseases where pathway redundancy, compensatory signaling, and acquired resistance often limit the effectiveness of single-target therapies \cite{stefan2024polypharmacology,ryszkiewicz2025polypharmacology,deshaies2025multispecific}. Compared with combination therapy, a dual-target molecule can coordinate activity at both targets within one chemical entity, potentially reducing mismatched pharmacokinetics, drug-drug interactions, and treatment complexity \cite{stefan2024polypharmacology,ryszkiewicz2025polypharmacology,deshaies2025multispecific}. In the structure-based setting, however, this problem is substantially harder than standard single-target generation: one molecule must be compatible with two binding sites while remaining drug-like and synthesizable \cite{zhou2024reprogramming,yuan2025mdrl}. Recent diffusion models have become powerful generative priors for 3D structure-based molecular design, but most existing methods still focus on one target at a time, leaving dual-target design comparatively underexplored \cite{schneuing2024diffsbdd,tang2024genai}.

\begin{figure}
    \centering
    \includegraphics[width=1\linewidth]{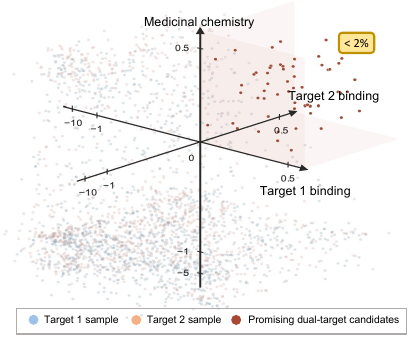}
    \caption{Latent dual-target potential in a frozen single-target diffusion model. Samples conditioned separately on Target 1 (blue) or Target 2 (orange) are evaluated by worst-case normalized margins for both target-binding objectives and medicinal chemistry; positive values indicate that the corresponding criterion is met. Fewer than 2\% satisfy all three criteria (red), suggesting that dual-target candidates are reachable but sparse under naive sampling.}
    \label{fig:motivation}
\end{figure}

Existing dual-target generative methods usually introduce dual-target capability in two ways: either by adapting the generator toward multi-target objectives or by intervening in the diffusion process during sampling. The former can be costly and difficult to scale when reliable target-pair training data is limited \cite{yuan2025mdrl,nguyen2026evosynth}. The latter avoids retraining, but it injects target-coupled control into the denoising trajectory, where competing update directions may bias generation toward one target and weaken balanced dual-target recovery \cite{yang2024paretodrug,zhou2024reprogramming}. This raises a natural question: \textit{can we preserve the pretrained single-target prior and its denoising dynamics, and instead recover dual-target candidates by searching only the input space of the frozen model?}

As single-target diffusion models are pretrained on diverse protein-ligand complexes, they have learnt a broad structure-based molecular prior that may be reusable beyond standard single-target generation. This raises the possibility that, for a given target pair, useful dual-target candidates can be recovered from the input space of a frozen model without changing its parameters or denoising dynamics \cite{gomezbombarelli2018automatic,winter2019continuous,abeer2024molso}. As shown in Figure~\ref{fig:motivation}, the frozen model can occasionally generate candidates satisfying both target-specific binding criteria even when conditioned on only one target at a time. However, realizing this potential remains challenging: fewer than 2\% of naive samples satisfy both binding criteria together with the medicinal-chemistry criterion, indicating that useful dual-target candidates occupy a sparse and highly constrained region of the model's input space \cite{winter2019continuous,abeer2024molso,isigkeit2024multitarget,ai2024mtmolgpt}.

Motivated by this, we formulate dual-target molecular design as a constrained multi-objective problem over the input space of a frozen single-target diffusion model and propose REUSE (\textbf{REU}sing frozen \textbf{S}ingle-target g\textbf{E}nerators). Unlike molecule-level evolutionary search, REUSE evolves generator inputs and assigns each a family-level score based on its decoded molecules. The decoded candidates are progressively narrowed using inexpensive affinity and chemistry signals before high-fidelity evaluation of a reduced frontier and final panel selection targeting strong affinity to both targets under chemical-feasibility and diversity constraints. The generator and its denoising dynamics remain unchanged throughout. Our contributions are as follows:
\begin{itemize}
    \item We present the first formulation of dual-target molecular design as constrained multi-objective search over the input space of a frozen single-target diffusion model. This decouples dual-target optimization from the diffusion process, reducing retraining cost and denoising-time balancing sensitivity while enabling explicit control over affinity, chemistry, and diversity.
    
    \item We develop a three-level recovery framework that evolves generator inputs using family-level evidence, progressively screens decoded molecules, and constructs diverse panels with strong affinity to both targets under QED/SA-based feasibility constraints.
    
    \item Extensive experiments show that REUSE recovers stronger and more balanced dual-target candidates than prior dual-target baselines, improving Dual High Affinity by 21.1 percentage points over the strongest prior baseline while retaining QED and SA profiles consistent with commonly used chemical-feasibility criteria.
\end{itemize}

\section{Related work}

\paragraph{Single-target structure-based molecular generation.}

Recent advances in structure-based molecular generation have substantially improved pocket-conditioned 3D design for single targets \cite{huang2024pmdm,jiang2024pocketflow,hu2025diffgui}. Pocket2Mol generates ligands autoregressively inside protein pockets under geometric constraints \cite{peng2022pocket2mol}. TargetDiff and DiffSBDD introduce equivariant diffusion frameworks for target-aware, pocket-conditioned molecule generation \cite{guan2023targetdiff,schneuing2024diffsbdd}. DecompDiff further improves diffusion-based SBDD with decomposed priors, bond diffusion, and validity guidance \cite{guan2023decompdiff}, while MolDiff addresses atom--bond inconsistency in 3D molecular diffusion \cite{peng2023moldiff}. DrugGPS improves generalization by learning subpocket prototypes and structural motifs for structure-based design \cite{zhang2023druggps}. Despite their strong performance, these methods are developed for the standard single-target setting and therefore do not directly address dual-target molecular design, where one ligand must simultaneously satisfy two distinct binding pockets.

\paragraph{Dual-target molecular generation.}

Recent dual-target generative methods mainly follow two directions. One
line adapts the generator itself, for example through reinforcement-learning-
based optimization or latent-space optimization for multi-target objectives
\cite{yuan2025mdrl,nguyen2026evosynth,munson2024polygon}. However, such
adaptation can be costly and unstable when reliable target-pair supervision
is sparse. The other line introduces target-coupled control during
inference-time generation or sampling
\cite{yang2024paretodrug,zhou2024reprogramming}. In this setting, competing
target preferences must be reconciled within the generation trajectory,
which may require delicate balancing and bias the resulting molecules
toward one target. These drawbacks motivate an alternative that keeps both
the pretrained generator and its denoising dynamics unchanged: recovering
dual-target molecules by searching the input space of a frozen single-target
diffusion model. Unlike conventional evolutionary search, REUSE scores each input through its decoded molecular family and separates input evolution from molecule screening and panel selection.

\begin{figure*}
    \centering
    \includegraphics[width=1\linewidth]{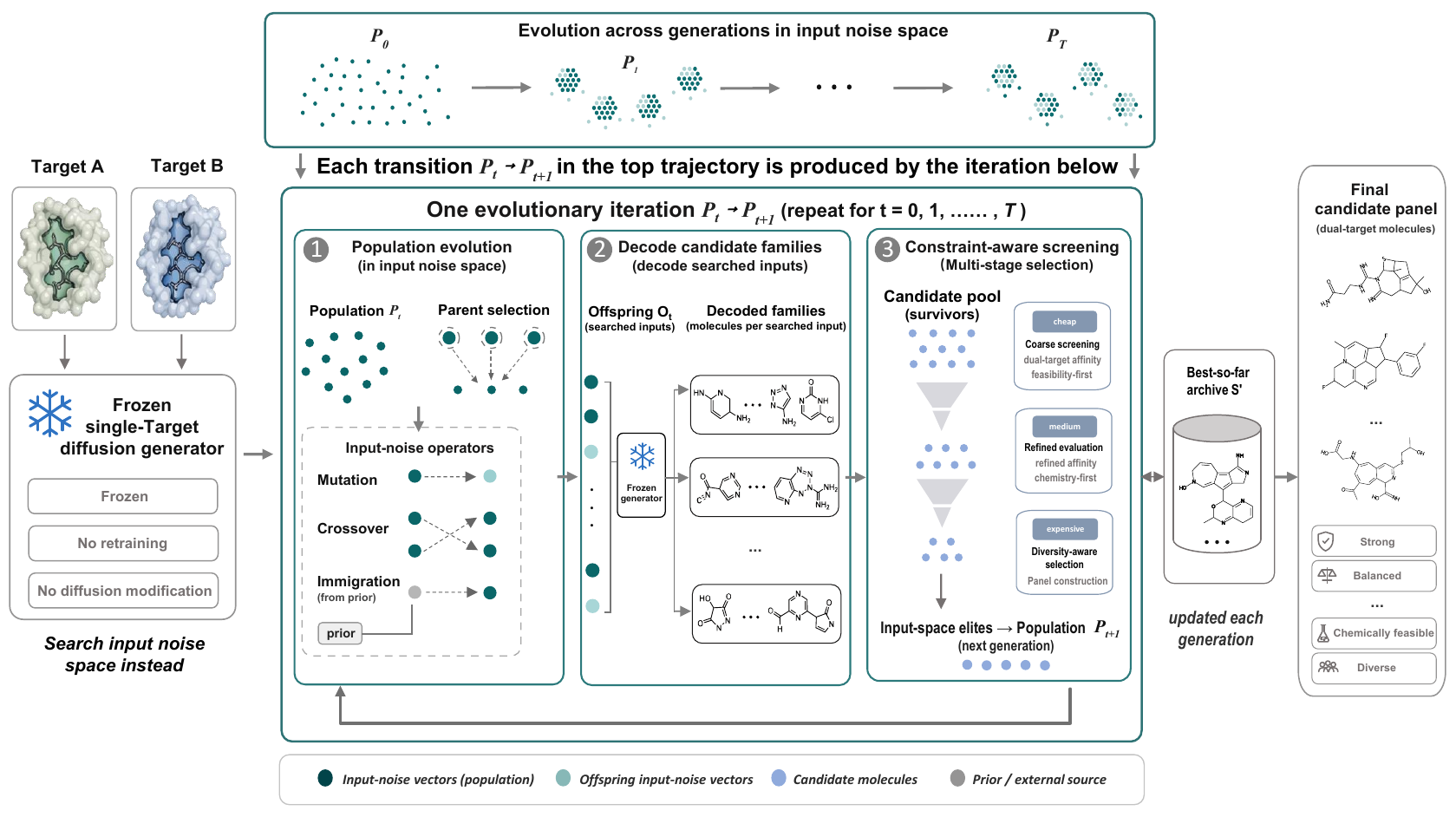}
    \caption{Overview of REUSE. REUSE evolves frozen-generator inputs using family-level scores, restricts full docking to a screened frontier, and selects a diverse panel with strong affinity to both targets under QED/SA-based chemical control.}
    \label{fig:pipeline}
\end{figure*}

\section{Methodology}

\subsection{Problem Formulation}

Let $G_{\phi}$ denote a pretrained \emph{single-target} diffusion generator with fixed parameters $\phi$.
Given a dual-target pair $p=(t_a,t_b)$, REUSE keeps $G_{\phi}$ fixed and searches over the \emph{input noise space} of the generator to recover dual-target candidates under the task defined by $p$. Let $z \in \mathcal{Z}$ denote an input noise realization supplied to the generator.
For a fixed task context $p$ and noise input $z$, the frozen generator induces a distribution over decoded molecules, from which we draw a finite family
\begin{equation}
\mathcal{M}(z,p)=\{m_1,\dots,m_K\}.
\label{eq:candidate_pool}
\end{equation}

Operationally, REUSE explores multiple noise inputs during search.
For any finite searched input set $\mathcal{Q}\subseteq\mathcal{Z}$, define the pooled decoded candidate universe
\begin{equation}
\mathcal{M}(\mathcal{Q},p)=\bigcup_{z\in\mathcal{Q}}\mathcal{M}(z,p).
\label{eq:pooled_family}
\end{equation}
Our goal is not to return a single molecule, but to recover a candidate panel $S \subseteq \mathcal{M}(\mathcal{Q},p)$ with $|S|=N$ that is jointly strong in dual-target affinity, chemically admissible, and structurally diverse. At conceptual level, we formulate this as a constrained multi-objective optimization problem:
\begin{equation}
\begin{aligned}
\min_{\substack{\mathcal{Q}\subseteq\mathcal{Z},\ |\mathcal{Q}|<\infty\\
S\subseteq\mathcal{M}(\mathcal{Q},p),\ |S|=N}}
&\ \Big(-F_{\mathrm{aff}}(S;p),-F_{\mathrm{chem}}(S),-F_{\mathrm{div}}(S)\Big)\\
\text{s.t.}\quad & c(m,p)=1,\qquad \forall m\in S.
\end{aligned}
\label{eq:multiobj}
\end{equation}
where $F_{\mathrm{aff}}$ is a larger-is-better utility measuring affinity to both targets and cross-target balance, $F_{\mathrm{chem}}$ is a QED/SA-based chemical-control score, $F_{\mathrm{div}}$ measures set-level diversity, and $c(m,p)$ encodes hard feasibility checks.
Eq.~\eqref{eq:multiobj} defines the panel-level objective, but REUSE does not optimize it directly because the search variables are noise inputs, not molecules. 
Each noise input is decoded into a molecular family and assigned a family-level score:
\begin{equation}
\mathcal{F}(z;p)=\Psi\big(\mathcal{M}(z,p);p\big),
\label{eq:input_search}
\end{equation}
where $\Psi$ summarizes the family's dual-target affinity, chemical feasibility, and diversity. 
This score is then used to rank and evolve noise inputs before final panel construction.

This formulation casts dual-target design as input-side capability recovery: REUSE probes a frozen single-target prior through its input noise space, rather than introducing dual-target behavior by retraining the model or modifying its inference dynamics.

\subsection{Overview of REUSE}

Figure~\ref{fig:pipeline} illustrates the overall framework.
At iteration $t$, REUSE maintains a population of noise-space individuals
\begin{equation}
\mathcal{P}_t=\{z_t^{(1)},\dots,z_t^{(B)}\}, \qquad z_t^{(i)} \in \mathcal{Z},
\label{eq:population}
\end{equation}
and evolves this population toward regions of the frozen prior that reliably decode to strong dual-target candidates. For each individual $z \in \mathcal{P}_t$, the frozen generator produces a decoded family $\mathcal{M}(z,p)$ for the target-pair design task $p$.
Rather than judging a noise input by a single decoded molecule, which is highly sensitive to sampling stochasticity, REUSE evaluates each noise input through the phenotypic evidence provided by its decoded molecular family:
\begin{equation}
\mathcal{F}(z;p)=\Psi\big(\mathcal{M}(z,p);p\big).
\label{eq:family_fitness}
\end{equation}
Parent selection is then utility-driven:
\begin{equation}
\mathcal{A}_t = \mathrm{Top}_{B_{\mathrm{par}}}\big(\mathcal{P}_{t-1};\mathcal{F}(\cdot;p)\big),
\label{eq:parent_selection_main}
\end{equation}
where $\mathrm{Top}_{B}(\mathcal{X};f)$ returns the top-$B$ elements of $\mathcal{X}$ ranked by scalar score $f$.
This encourages the search to favor regions of the input noise space that consistently yield high-quality candidates, rather than isolated outliers.
In our instantiation, $\Psi$ ranks each decoded family using the first-stage feasibility-first ordering $\succ_1$ defined later in Eq.~\eqref{eq:stage_order}, and averages the corresponding top-ranked $h_1(m;p)$ scores to reduce sensitivity to a single stochastic decode; see Appendix~A.3.

REUSE couples three decision levels. At the input level, family-level scores guide evolution. At the molecule level, candidates are progressively screened using affinity and chemical-control signals; balance-aware scoring penalizes candidates strong on only one target, and full docking is reserved for a reduced frontier. At the panel level, survivors are jointly selected for strong affinity to both targets and structural diversity under QED/SA-based chemical control, with the best-so-far panel retained across iterations. Algorithm 1 in Appendix A.2 summarizes the procedure, while Appendix A.7 relates it to Eq.~\eqref{eq:multiobj}.

\subsection{Variation Operators in Noise Space}

Given a selected parent set $\mathcal{A}_t \subseteq \mathcal{P}_{t-1}$, REUSE generates offspring according to the mixture proposal
\begin{equation}
\begin{aligned}
q_{\mathrm{off}}(z'\mid\mathcal{A}_t,p)
={}&\alpha_{\mathrm{mut}}q_{\mathrm{mut}}(z'\mid\mathcal{A}_t,p)\\
&+\alpha_{\mathrm{cross}}q_{\mathrm{cross}}(z'\mid\mathcal{A}_t,p)\\
&+\alpha_{\mathrm{imm}}\pi_0(z'\mid p),
\end{aligned}
\label{eq:offspring_mixture}
\end{equation}
where $\alpha_{\mathrm{mut}}+\alpha_{\mathrm{cross}}+\alpha_{\mathrm{imm}}=1$.
Here, $\pi_0(\cdot\mid p)$ denotes the random noise distribution used for population initialization and immigration. Samples are drawn independently from this distribution, while target-pair information enters through molecular decoding, dual-target evaluation, and downstream selection. The offspring set at iteration $t$ is then sampled as
\begin{equation}
\mathcal{O}_t=\{z'_1,\dots,z'_{B_{\mathrm{off}}}\},
\qquad
z'_i \sim q_{\mathrm{off}}(\cdot \mid \mathcal{A}_t,p),
\label{eq:offspring_set}
\end{equation}
where $B_{\mathrm{off}}$ is offspring budget. A continuous-space instantiation used in our experiments is
\begin{equation}
z'_{\mathrm{mut}} = z + \epsilon,
\qquad
\epsilon \sim \mathcal{N}(0,\sigma_{\mathrm{mut}}^2 I),
\label{eq:mutation}
\end{equation}
for mutation,
\begin{equation}
z'_{\mathrm{cross}} = \lambda z_i + (1-\lambda) z_j,
\qquad
\lambda \sim \mathcal{U}(0,1),
\qquad
z_i,z_j \in \mathcal{A}_t,
\label{eq:crossover}
\end{equation}
for crossover between two selected parents, and
\begin{equation}
z'_{\mathrm{imm}} \sim \pi_0(\cdot \mid p)
\label{eq:immigration}
\end{equation}
for immigration. Thus, mutation performs local exploration around promising inputs, crossover recombines successful parent lineages, and immigration draws fresh random noise inputs from $\pi_0(\cdot\mid p)$ to refresh the population and mitigate premature concentration.
All operators act directly in the input noise space and therefore leave the frozen generator unchanged.
Additional implementation details are given in Appendix~A.4.

\subsection{Constraint-Aware Multi-Stage Environmental Selection}

Let $\mathcal{O}_t$ denote the offspring set sampled from Eq.~\eqref{eq:offspring_mixture}.
The decoded offspring induce a pooled candidate set
\begin{equation}
\mathcal{C}_t^{(0)}=\bigcup_{z \in \mathcal{O}_t}\mathcal{M}(z,p).
\label{eq:pooled_candidates}
\end{equation}
REUSE then applies stage-wise environmental selection to this pool.
We index stage pools by $s=0,\dots,S$, where stage $s$ maps $\mathcal{C}_t^{(s-1)}$ to $\mathcal{C}_t^{(s)}$. At stage $s$, each candidate $m \in \mathcal{C}_t^{(s-1)}$ is assigned a scalar score
\begin{equation}
\begin{aligned}
h_s(m;p)={}&f_{\mathrm{aff}}^{(s)}(m;p)
+\beta_{\mathrm{chem}}q_{\mathrm{chem}}(m)\\
&+\beta_{\mathrm{div}}q_{\mathrm{div}}\!\left(m;\mathcal{C}_t^{(s-1)}\right),
\end{aligned}
\label{eq:stage_score}
\end{equation}
where $f_{\mathrm{aff}}^{(s)}$ is a stage-specific affinity term for both targets, $q_{\mathrm{chem}}$ is a soft chemical-control term, and $q_{\mathrm{div}}$ is a redundancy-reduction term. We instantiate these terms using progressively higher-fidelity docking signals, QED and SA, and fingerprint-based structural dissimilarity, respectively \cite{bickerton2012qed,ertl2009sascore,rogers2010ecfp}; see Appendices A.5 and A.6.
Here, $A_s(m,t)$ denotes the stage-$s$ affinity estimate of molecule $m$ against target $t$, represented on a common larger-is-better utility scale.
A simple balance-aware instantiation is
\begin{equation}
\begin{aligned}
f_{\mathrm{aff}}^{(s)}(m;p)={}&w_aA_s(m,t_a)+w_bA_s(m,t_b)\\
&-\lambda_{\mathrm{bal}}\big|A_s(m,t_a)-A_s(m,t_b)\big|,
\end{aligned}
\label{eq:balance_aff_main}
\end{equation}
so that candidates strong on only one target are down-weighted.
When a stage uses raw docking scores, we convert them to this utility scale by sign flip (or an equivalent monotone transformation) before applying Eq.~\eqref{eq:balance_aff_main}. Selection is feasibility-first.
Let $c(m,p)\in\{0,1\}$ indicate whether candidate $m$ satisfies the hard admissibility checks used by the pipeline.
Stage-wise preference is defined lexicographically as
\begin{equation}
\begin{aligned}
m_i\succ_s m_j\quad\Longleftrightarrow\quad
&\big(c(m_i,p),h_s(m_i;p)\big)\\
&>\big(c(m_j,p),h_s(m_j;p)\big),
\end{aligned}
\label{eq:stage_order}
\end{equation}
where $>$ denotes lexicographic comparison.
Thus, feasible candidates always dominate infeasible ones, and within the same feasibility class larger $h_s$ is preferred.
Accordingly, stage $s$ produces the top $B_s$ survivors under $\succ_s$:
\begin{equation}
\mathcal{C}_t^{(s)}=\mathrm{Top}_{B_s}\big(\mathcal{C}_t^{(s-1)};\succ_s\big),
\qquad s=1,\dots,S,
\label{eq:stage_selection}
\end{equation}
and we write $\mathcal{C}_t^{(\mathrm{final})}:=\mathcal{C}_t^{(S)}$. Finally, REUSE constructs the output at the set level:
\begin{equation}
\begin{aligned}
S_t \in \operatorname*{arg\,max}_{S}\quad
& J_p(S)\\
\text{s.t.}\quad
& S\subseteq\mathcal{C}_t^{(\mathrm{final})},\quad |S|=N,\\
& c(m,p)=1,\quad \forall\,m\in S,\\
& \Delta(m_i,m_j)\ge\tau,\quad 1\le i<j\le N.
\end{aligned}
\label{eq:final_diverse_set}
\end{equation}
where $J_p(S)$ combines affinity to both targets and diversity with QED/SA-based chemical control within the final panel.  If no feasible $N$-subset exists, the current iteration simply yields $S_t=\emptyset$. In such cases, the global search does not terminate; instead, the algorithm safely falls back to the best-so-far panel tracked from previous iterations. To drive the search forward into the next generation, the noise-space population is updated via elitist survival:
\begin{equation}
\mathcal{P}_t=\mathrm{Top}_{B}\big(\mathcal{P}_{t-1}\cup\mathcal{O}_t;\mathcal{F}(\cdot;p)\big).
\label{eq:population_update}
\end{equation}
This updated population $\mathcal{P}_t$ retains the most promising generator inputs and serves as the parent pool for the next iteration, thereby closing the evolutionary loop. Upon completing all $T$ iterations, REUSE returns the final best-so-far balanced and diverse candidate panel $S^\star$, rather than a winner-take-all molecule. Concrete instantiations of $\Psi$, $h_s$, and $J_p$ are given in Appendix A.3, Appendix A.5, and Appendix A.6, respectively.

\begin{figure*}[t!]
    \centering
    \includegraphics[width=1\linewidth]{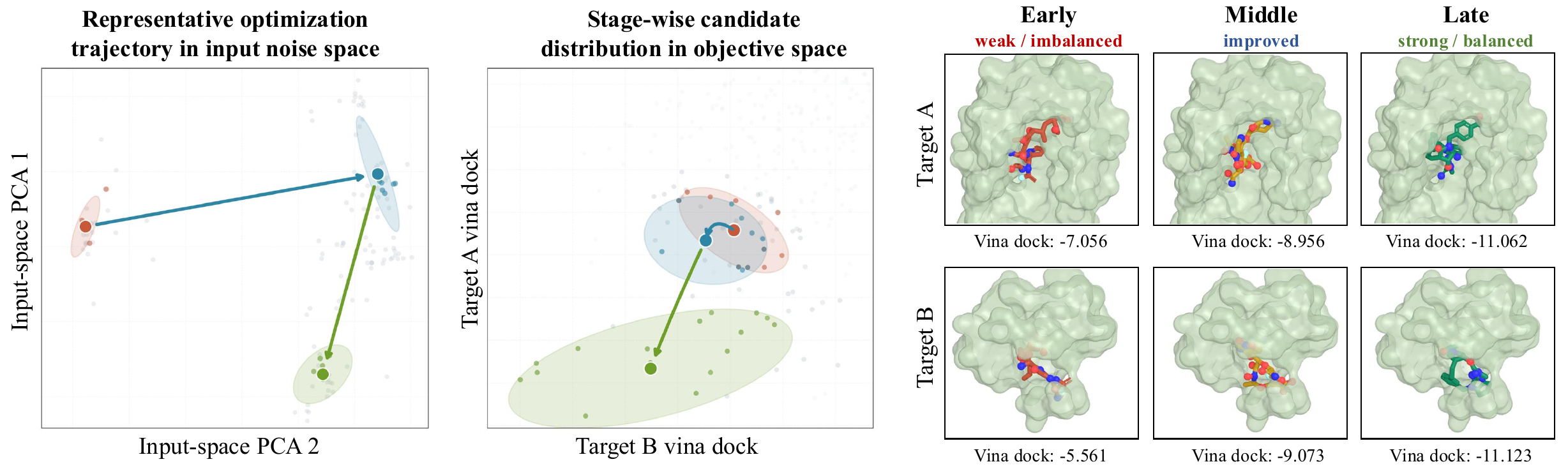}
    \caption{Visualization of the stage-wise search and optimization process in REUSE. (Left) Optimization trajectory in the input noise space. (Middle) Stage-wise candidate distribution in the objective space. (Right) Evolution of molecular binding poses across early, middle, and late stages.}
    \label{fig:2}
\end{figure*}

\begin{table*}[t]
\centering
\small
\setlength{\tabcolsep}{3pt}
\caption{
Quantitative comparison of REUSE and baselines for dual-target molecule
generation. All methods except MDRL use the same budget of 72 generated
candidates per target pair; MDRL follows its higher-budget converged
reinforcement-learning protocol.
Dark blue, light blue, and gray denote first, second, and third place,
respectively. The Reference row is excluded
from ranking.
}
\label{tab:formal20_merged}
\resizebox{\textwidth}{!}{%
\begin{tabular}{@{}l*{20}{c}@{}}
\toprule
\multirow{2}{*}{Method}
& \multicolumn{2}{c}{P-1 Vina Dock $\downarrow$}
& \multicolumn{2}{c}{P-2 Vina Dock $\downarrow$}
& \multicolumn{2}{c}{Max Vina Dock $\downarrow$}
& \multicolumn{2}{c}{Dual High Aff. $\uparrow$}
& \multicolumn{2}{c}{P-1 GNINA CNNaff. $\uparrow$}
& \multicolumn{2}{c}{P-2 GNINA CNNaff. $\uparrow$}
& \multicolumn{2}{c}{QED $\uparrow$}
& \multicolumn{2}{c}{SA $\uparrow$}
& \multicolumn{2}{c}{Diversity $\uparrow$}
& \multicolumn{2}{c}{Consensus Hit $\uparrow$} \\
\cmidrule(lr){2-3}
\cmidrule(lr){4-5}
\cmidrule(lr){6-7}
\cmidrule(lr){8-9}
\cmidrule(lr){10-11}
\cmidrule(lr){12-13}
\cmidrule(lr){14-15}
\cmidrule(lr){16-17}
\cmidrule(lr){18-19}
\cmidrule(lr){20-21}
& Avg. & Med.
& Avg. & Med.
& Avg. & Med.
& Avg. & Med.
& Avg. & Med.
& Avg. & Med.
& Avg. & Med.
& Avg. & Med.
& Avg. & Med.
& Avg. & Med. \\
\midrule

Reference
& -7.60 & -7.80
& -6.02 & -7.30
& -- & --
& -- & --
& 5.44 & 5.57
& 5.31 & 5.10
& 0.53 & 0.54
& 0.74 & 0.77
& -- & --
& -- & -- \\

\midrule

Pocket2Mol
& -4.73 & -4.68
& -4.65 & -4.69
& -4.48 & -4.47
& 0.3\% & 0.0\%
& 4.07 & 4.01
& 4.11 & 4.09
& 0.46 & 0.47
& \rankone{0.87} & \rankone{0.86}
& \rankone{0.85} & \rankone{0.82}
& 0.1\% & 0.0\% \\

TargetDiff
& -8.59 & -8.64
& -6.91 & -7.39
& -6.36 & -7.15
& 27.7\% & 19.8\%
& 5.91 & 5.95
& 4.52 & 4.59
& 0.50 & 0.52
& 0.55 & \rankthree{0.57}
& \ranktwo{0.70} & \rankthree{0.69}
& 26.7\% & 17.1\% \\

DiffLinker
& -7.02 & -7.77
& -7.19 & -7.84
& -5.81 & -7.10
& 24.3\% & 0.0\%
& 4.77 & 4.81
& 4.79 & 4.85
& 0.43 & 0.45
& 0.32 & 0.29
& 0.50 & 0.49
& 21.3\% & 0.0\% \\

LinkerNet
& -8.14 & -8.32
& -8.10 & -8.33
& -7.12 & -7.65
& 34.6\% & 0.0\%
& 5.31 & 5.32
& 5.01 & 5.05
& \rankthree{0.57} & \rankthree{0.62}
& \ranktwo{0.69} & \ranktwo{0.70}
& 0.34 & 0.34
& 30.8\% & 0.0\% \\

ParetoDrug
& -8.08 & \rankthree{-9.34}
& \rankthree{-8.83} & -8.71
& -6.82 & \rankthree{-8.42}
& 26.7\% & 16.7\%
& 5.19 & 5.45
& 5.16 & 5.04
& 0.41 & 0.40
& 0.52 & 0.53
& 0.54 & 0.60
& 8.3\% & 0.0\% \\

ComPDIFF
& -8.19 & -8.32
& -8.25 & -8.40
& -7.41 & -7.69
& 35.2\% & 29.7\%
& 5.83 & 5.82
& 5.69 & 5.70
& 0.53 & 0.54
& 0.54 & 0.54
& \rankthree{0.69} & \ranktwo{0.71}
& 31.1\% & 26.5\% \\

DualDiff
& -8.34 & -8.47
& -8.40 & -8.49
& -7.62 & -7.78
& 35.8\% & 29.7\%
& 5.87 & 5.85
& 5.73 & 5.76
& 0.52 & 0.51
& 0.54 & 0.56
& 0.63 & 0.64
& 32.6\% & 27.9\% \\

MDRL
& -8.32 & -8.42
& -8.45 & -8.66
& -7.91 & -8.10
& 37.2\% & 31.6\%
& 5.11 & 5.12
& 5.02 & 4.96
& \rankthree{0.57} & 0.58
& 0.55 & 0.56
& 0.44 & 0.43
& 35.6\% & 29.9\% \\

\midrule

w/o search
& -8.02 & -8.11
& -7.94 & -8.06
& -7.35 & -7.48
& 31.8\% & 26.4\%
& 5.14 & 5.18
& 5.10 & 5.13
& 0.56 & 0.58
& 0.55 & 0.56
& 0.45 & 0.44
& 30.9\% & 25.7\% \\

w/o balance
& \rankthree{-9.14} & -9.26
& \ranktwo{-9.02} & \ranktwo{-9.26}
& \rankthree{-8.41} & \rankone{-8.62}
& \ranktwo{55.2\%} & \rankthree{50.8\%}
& \rankthree{6.30} & \rankone{6.22}
& \ranktwo{6.06} & \rankthree{6.04}
& \ranktwo{0.61} & \ranktwo{0.63}
& 0.56 & 0.55
& 0.44 & 0.43
& \rankthree{48.3\%} & \rankthree{47.7\%} \\

w/o chemctrl
& \rankone{-9.60} & \rankone{-9.55}
& -8.81 & \rankthree{-8.92}
& \ranktwo{-8.47} & -8.32
& \rankthree{54.5\%} & \ranktwo{54.5\%}
& \rankone{6.36} & \ranktwo{6.19}
& \rankthree{6.01} & \ranktwo{6.05}
& 0.52 & 0.52
& 0.51 & 0.51
& 0.47 & 0.51
& \rankone{51.5\%} & \rankone{53.7\%} \\

REUSE
& \ranktwo{-9.41} & \ranktwo{-9.38}
& \rankone{-9.26} & \rankone{-9.64}
& \rankone{-8.64} & \ranktwo{-8.58}
& \rankone{58.3\%} & \rankone{57.3\%}
& \ranktwo{6.32} & \rankthree{6.15}
& \rankone{6.13} & \rankone{6.24}
& \rankone{0.62} & \rankone{0.67}
& \rankthree{0.57} & 0.56
& 0.48 & 0.46
& \ranktwo{50.1\%} & \ranktwo{48.8\%} \\

\bottomrule
\end{tabular}%
}
\end{table*}

\section{Experiments}
\subsection{Experimental Setup}
We evaluate on the full dual-target benchmark of Zhou et
al.~\cite{zhou2024reprogramming}, comprising 12,917 target pairs across 438 unique targets. Rather than random pairing, the benchmark derives
pharmacologically meaningful tasks from drug combinations with significant
synergistic effects~\cite{zheng2021drugcomb}. Each target is associated with
a reference molecule and a pocket-level protein--ligand complex. Available
complexes are taken from PDBBind, while missing structures are constructed
from PDB or AlphaFold DB using P2Rank and AutoDock Vina
\cite{li2015pdbbind,berman2000pdb,varadi2022alphafolddb,krivak2018p2rank,trott2010vina}. We compare with single-target generators (Pocket2Mol and TargetDiff),
linker-design methods (DiffLinker and LinkerNet), ParetoDrug, dual-target
methods (ComPDIFF, DualDiff, and MDRL), and REUSE ablations. Together, these
baselines span both diffusion-based and non-diffusion generation and
optimization paradigms. REUSE uses TargetDiff as its default frozen
backbone. Except for MDRL, all methods generate at most 72 candidates per
target pair before common scoring and panel selection. MDRL follows its
higher-budget converged reinforcement-learning protocol. Max Vina Dock is the less favorable of the two target-specific Vina scores,
while Dual High Affinity requires both scores to outperform their
corresponding references. GNINA CNNaffinity is a larger-is-better post hoc score that is not used during generation, search, or selection. Consensus Hit additionally requires non-negative GNINA CNNaffinity margins relative to both references, QED $\geq 0.5$, and SA $\geq 0.5$. Here, SA denotes the benchmark-normalized, larger-is-better transformation of the Ertl synthetic-accessibility score \cite{ertl2009sascore}. Diversity is the average pairwise Morgan-fingerprint dissimilarity within each reported panel. We report average and median results over the benchmark.

\subsection{Main Quantitative Results}

Table~\ref{tab:formal20_merged} reports the main comparison under the matched
candidate budget, except for the higher-budget MDRL baseline. REUSE achieves
the best average P-2 Vina Dock ($-9.26$), Max Vina Dock ($-8.64$), and Dual
High Affinity (58.3\%). Relative to MDRL, these metrics improve by 0.81,
0.73, and 21.1 percentage points, respectively. Importantly, post hoc GNINA
CNNaffinity provides an optimization-independent computational validation,
as it is not used during generation, search, or selection. REUSE achieves
P-1/P-2 GNINA scores of 6.32/6.13, exceeding the best prior-baseline scores
of 5.91/5.73 and showing that the affinity gains persist beyond the Vina
optimization oracle. With QED and SA serving as chemistry-control criteria,
REUSE obtains average QED/SA values of 0.62/0.57 and a Consensus Hit rate of
50.1\%.

\subsection{Ablation Study}

We ablate input-space search, balance-aware selection, and chemical control. Removing search causes the largest degradation: average P-1/P-2 Vina Dock worsens from $-9.41$/$-9.26$ to $-8.02$/$-7.94$, Dual High Affinity drops from 58.3\% to 31.8\%, and Consensus Hit drops from 50.1\% to 30.9\%.
This confirms that structured input-space exploration is the main driver of candidate recovery. Removing balance primarily weakens cross-target consistency, reducing Dual High Affinity to 55.2\% and worsening P-2 Vina Dock to $-9.02$. Removing chemical control reveals an affinity--quality
trade-off: P-1 Vina Dock improves to $-9.60$ and Consensus Hit increases slightly to 51.5\%, whereas average QED/SA decreases from 0.62/0.57 to 0.52/0.51. Overall, input-space search drives affinity recovery, balance-aware selection improves cross-target consistency, and chemical control shifts the selected candidates toward more favorable QED/SA profiles.

\begin{figure*}
    \centering
    \includegraphics[width=1\linewidth]{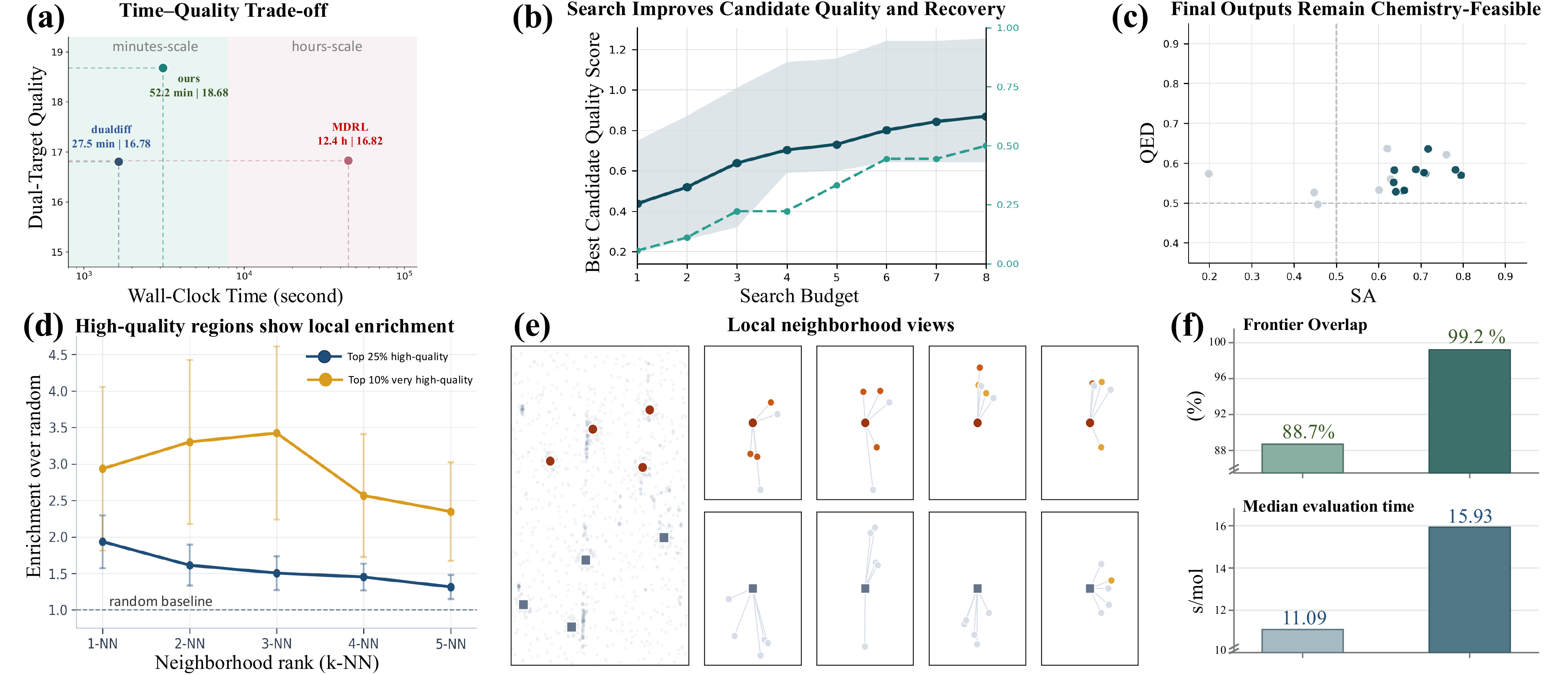}
    \caption{Additional analyses of REUSE. (a) Time--quality trade-off across methods. (b) Candidate quality and chemistry-feasible recovery as the search budget increases. (c) Distribution of final outputs in the QED--SA space. (d) Enrichment of high-quality candidates in local input-space neighborhoods. (e) Representative local neighborhoods around high- and lower-quality anchors. (f) Frontier preservation and evaluation cost of the cheap evaluator relative to the full evaluator.}
\label{fig:combined_experiments}
\end{figure*}


\subsection{Search Dynamics Over Iterations}

Figure~\ref{fig:2} presents the iterative behavior of REUSE from three aligned perspectives. In the input space (left), the trajectory departs from the initial region and progressively concentrates in a smaller set of more promising areas, indicating that search over the input space is gradually identifying and retaining favorable regions rather than repeatedly discovering isolated good points by chance. In the objective space (middle), the population shifts toward the lower-left region over time, showing that the search is improving the two target objectives jointly instead of trading one target off against the other in a purely one-sided manner. In the molecular view (right), representative candidates evolve from weak or imbalanced binding at early stages to stronger and more balanced binding at later stages. This molecular-level progression is consistent with the movement observed in input space and objective space, so the three panels together support the view that the search process is meaningful, cumulative, and not merely a by-product of stochastic sampling.


\subsection{Efficiency, Search Progression, and Chemical Validity}

Figure~\ref{fig:combined_experiments}(a--c) offers a complementary view of REUSE in terms of efficiency, search progression, and chemical validity. In panel (a), REUSE achieves a favorable time--quality trade-off: it obtains higher dual-target quality than MDRL with much lower runtime, and also outperforms the faster DualDiff in quality. For context, the reported MDRL runtime reflects a separate reinforcement-learning optimization for each target pair, so the 12.4-hour cost should be interpreted as per-pair optimization rather than one-time training amortized across pairs. DualDiff is faster, about 27.5 minutes, but its sampling-time dual-target guidance can be sensitive to competing target-induced updates, which may limit balanced recovery. REUSE runs in about 52 minutes; it also performs pair-specific search, but keeps the generator and denoising dynamics fixed and searches only the input space. This avoids the retraining or policy-update cost of generator adaptation, while also sidestepping the denoising-time target-balancing sensitivity introduced by sampling-time intervention. A numeric runtime summary is given in Appendix Table 17, with additional diagnostics in Appendix Tables 7 and 10. Panel (b) shows that increasing the search budget consistently improves both the best lightweight search score and the fraction of cases with at least one chemistry-feasible recovered molecule, suggesting cumulative rather than random exploration. Panel (c) further visualizes the QED–SA distribution of the final outputs relative to the chemistry-feasibility thresholds of QED $\geq 0.5$ and SA $\geq 0.5$. Overall, REUSE provides stronger dual-target recovery than DualDiff at moderate additional inference cost, while remaining far cheaper than the per-pair RL optimization used by MDRL.


\subsection{Local Neighborhood Structure in the Frozen Input Space}

Figure~\ref{fig:combined_experiments}(d--e) examines local neighborhood structure in the frozen input space from two complementary views. Here, an anchor is a point selected by its quality and used as the center for inspecting whether its local neighborhood also contains additional good points. In the panel (d), neighborhoods centered at high-quality anchors contain more high-quality neighbors than would be expected from randomly chosen centers, indicating that favorable solutions are not merely isolated outliers caused by noise. The enrichment is strongest for the top 10\% anchors, which further suggests that the best regions of the frozen input space tend to form exploitable local basins rather than single exceptional samples. The panel (e) makes the same pattern concrete through representative anchor-centered neighborhoods: higher-quality anchors are more often embedded among similarly strong neighbors, whereas lower-quality anchors are typically surrounded by more ordinary ones, consistent with the aggregate trend in the panel (d) rather than being an artifact of averaging alone. Taken together, these results suggest that good points cluster locally in the frozen input space, which in turn helps explain why mutation-based local exploration and neighborhood-preserving search can work reliably here instead of reducing to blind trial-and-error. Panel (f) shows that cheap evaluator
retains 88.7\% and 99.2\% of the full-evaluator top-5 and top-8 frontiers,
respectively. This supports its use as a high-recall screening stage before
full reranking; detailed agreement and efficiency statistics are provided
in Appendix C.6.



\subsection{Further Analyses}

Beyond aggregate benchmark scores, the Supplementary Material provides
complementary evidence on scorer robustness, input-space organization,
optimizer choice, chemical quality, screening efficiency, and backbone
transfer. Appendix~C.1 shows that the recovered candidates retain stable
Vina-family rankings, favorable relaxed-pose scores and pocket contacts,
and substantial residue-level overlap with reference ligands. Appendix~C.2
identifies structurally and objectively coherent local neighborhoods,
suggesting that favorable candidates form exploitable regions rather than
isolated points. Appendix~C.3 further shows how deeper search can improve
chemistry and reach chemistry-feasible regions. Under the matched
72-candidate budget, Appendix~C.4 finds that the default EA is best or tied
best in 17 of 18 metrics relative to CMA-ES and BO. Appendix~C.5 extends the evaluation beyond binding scores: 85.8\% of final
candidates jointly satisfy Ro5 and Veber, 95.8\% are PAINS-free, and 99.2\%
of their scaffolds remain novel relative to the complete benchmark
reference set. The panels also retain high within-panel diversity, with
predicted ADMET profiles provided for downstream prioritization.
Appendix~C.6 validates the coarse-to-fine design, with the cheap evaluator
preserving 88.7\% and 99.2\% of the full-evaluator top-5 and top-8
frontiers. Appendix~C.7 provides pose-level examples of the aggregate gains
across both targets, while Appendix~C.8 shows that compact
target-pair-specific motifs recur across globally distinct molecules
without reducing the panel to a single scaffold. Appendix~C.9 shows that
both the best-so-far search score and chemistry-floor recovery improve with
search budget. Appendix~C.10 confirms that final selection raises Consensus
Hit from 41.8\% to 50.1\% and Dual High Affinity from 50.1\% to 58.3\%,
while improving conservative dual-target docking. Appendix~C.11 quantifies
the time--performance trade-off: REUSE achieves the strongest Combined Dock
result with moderate overhead over DualDiff, yet remains far below MDRL's
12.4-hour per-pair optimization cost. Finally, Appendix~C.12 demonstrates
transfer across three frozen single-target diffusion backbones without
retraining or modifying their denoising dynamics.

\section{Conclusion}

In this work, we investigate whether a frozen single-target diffusion model can be reused for dual-target molecular design without retraining or altering its denoising dynamics. We propose REUSE, which evolves generator inputs using family-level evidence aggregated over repeated stochastic decodes, progressively screens the resulting candidates, and jointly constructs diverse panels with strong affinity to both targets under QED/SA-based chemical control. Across the full dual-target benchmark, REUSE improves Dual High Affinity by 21.1 percentage points over the strongest prior baseline while maintaining QED and SA profiles consistent with commonly used feasibility criteria. These results show that structured input-space search can recover useful dual-target capability already encoded in a frozen single-target prior. More broadly, REUSE offers a general strategy for discovering and organizing sparse, task-relevant capabilities in pretrained molecular generators, opening new possibilities for multi-target and constrained molecular design.

\clearpage
\appendix
\setcounter{secnumdepth}{2}

\section*{Overview}
In this appendix, we provide additional details and supplementary analyses for \textsc{REUSE}, our framework for recovering dual-target molecules by searching the input space of a frozen single-target diffusion model. The appendix is organized into four parts. First, we present additional methodological details, including notation, family-level utility, evolutionary search operators, stage-wise environmental selection, and set-level panel construction. Second, we summarize implementation and reproducibility details, including the concrete experimental instantiation, hyperparameter settings, and runtime environment. Third, we provide supplementary experimental analyses and visualizations that examine oracle and structural sanity checks, input-space optimizer choices, local search dynamics, extended physicochemical and predicted ADMET profiles, structural-alert patterns, benchmark-reference novelty and intra-panel diversity, and qualitative pose- and motif-level properties of the recovered candidates. Fourth, we present mathematical properties and proofs for the exact-selector instantiation analyzed in this paper.

\section*{Reproducibility}

To facilitate reproducibility, this appendix documents the concrete settings
used in all reported experiments, including the frozen generator,
input-space search hyperparameters, docking configuration, chemistry
thresholds, diversity constraints, and runtime environment. We also provide
explicit definitions of the family-level search utility, stage-wise
selection criteria, evaluation metrics, and set-level panel utility. The
source code, configuration files, and evaluation scripts will be made
publicly available upon acceptance.

\section{Method Details}

\subsection{Notation Summary}
\label{app:notation}

Table~\ref{tab:notation} summarizes the main mathematical symbols used in REUSE.

\begin{table*}[t]
\centering
\small
\begin{tabularx}{\textwidth}{@{}L{0.22\textwidth}X@{}}
\toprule
Symbol & Meaning \\
\midrule
$G_{\phi}$ & Pretrained frozen single-target diffusion generator with fixed parameters $\phi$ \\
$p=(t_a,t_b)$ & Dual-target pair \\
$\mathcal{Z}$ & Input noise space searched by REUSE \\
$z \in \mathcal{Z}$ & Input noise realization \\
$\mathcal{P}_t$ & Population of noise-space individuals at iteration $t$ \\
$\mathcal{A}_t$ & Selected parent set at iteration $t$ \\
$\mathcal{O}_t$ & Offspring set at iteration $t$ \\
$\mathcal{M}(z,p)$ & Decoded molecular family generated from noise input $z$ under target pair $p$ \\
$\mathcal{M}(\mathcal{Q},p)$ & Pooled decoded candidate universe induced by a searched input set $\mathcal{Q}$ \\
$\mathcal{F}(z;p)$ & Family-level utility assigned to input $z$ \\
$\mathcal{C}_t^{(s)}$ & Candidate pool after stage $s$ of iteration $t$ \\
$f_{\mathrm{aff}}^{(s)}(m;p)$ & Stage-$s$ dual-target affinity score \\
$q_{\mathrm{chem}}(m)$ & Soft chemistry-quality term \\
$q_{\mathrm{div}}(m;\mathcal{C}_t^{(s-1)})$ & Diversity / redundancy-reduction term evaluated relative to the incoming pool at stage $s$ \\
$h_s(m;p)$ & Stage-$s$ scalar score for candidate $m$ \\
$c(m,p)$ & Hard feasibility indicator \\
$\succ_s$ & Stage-$s$ feasibility-first preference relation \\
$\lambda_{\mathrm{bal}}$ & Balance-penalty weight in the affinity score \\
$J_p(S)$ & Set-level utility for the final candidate panel \\
$\Delta(m_i,m_j)$ & Pairwise dissimilarity between molecules $m_i$ and $m_j$ \\
$\tau$ & Minimum diversity threshold for the final panel \\
$B$ & Population size \\
$B_{\mathrm{par}}$ & Number of selected parents \\
$B_{\mathrm{off}}$ & Number of offspring sampled per iteration \\
$S$ & Number of environmental-selection stages \\
$B_s$ & Survivor budget at stage $s$ \\
$N$ & Final panel size, assumed to satisfy $N\ge 2$ \\
\bottomrule
\end{tabularx}
\caption{Mathematical notation used in REUSE.}
\label{tab:notation}
\end{table*}

For clarity, $\mathrm{Top}_B(\mathcal{X};\cdot)$ returns the
$\min\{B,|\mathcal{X}|\}$ highest-ranked available elements of
$\mathcal{X}$, with the ranking induced by either a scalar score or a
preference relation.
We also set $J_p(\emptyset):=-\infty$ for the empty panel.

\begin{figure}[t]
    \centering
    \includegraphics[width=1\linewidth]{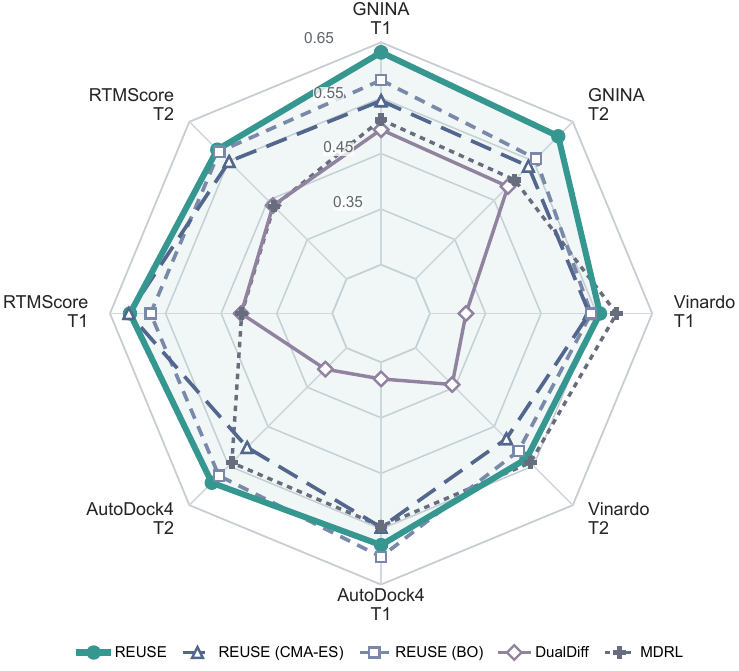}
    \caption{
Post hoc cross-scorer comparison on the two targets. Each spoke reports
the method-level mean for one scorer--target combination. For visualization,
all scorer outputs are oriented so that larger values are better and
monotonically rescaled to a common display scale, preserving within-scorer
rankings. A shared radial range of 0.25--0.65 is used without per-axis
min--max normalization. REUSE denotes the default EA configuration.
}
    \label{fig:cross_scorer_radar}
\end{figure}

\subsection{Overall Procedure}
\label{app:overall_procedure}

Algorithm~\ref{alg:REUSE} summarizes the budgeted implementation of
REUSE. The initialized population is treated as the first latent
generation. At each generation, newly evaluated inputs are decoded into
molecular families and added to the cumulative candidate pool. REUSE then
applies coarse-to-full environmental selection, constructs a feasible
candidate panel, updates the latent population by elitist survival, and
retains the best-so-far panel as the incumbent. After $T$ generations, the
final incumbent $S_T^\star$ is returned.

The components used in Algorithm~\ref{alg:REUSE} are instantiated in the
subsequent appendix sections: Appendix~\ref{app:family_utility} defines the
family-level utility, Appendix~\ref{app:evolution_search} details the
evolutionary operators, Appendix~\ref{app:stage_selection} specifies
stage-wise environmental selection, and Appendix~\ref{app:panel_selection}
gives the final panel-construction rule.

\begin{algorithm*}
\caption{\textsc{REUSE}: Evolutionary Search in the Input Noise Space of a Frozen Single-Target Diffusion Model}
\label{alg:REUSE}
\footnotesize
\begin{algorithmic}[1]

\Require Frozen generator $G_\phi$, target pair $p$, and initialization
distribution $\pi_0(\cdot\mid p)$; sizes $B$, $B_{\mathrm{par}}$, and
$B_{\mathrm{off}}$; latent generations $T$ (including initialization)
\Require Decodes per input $M_{\mathrm{eval}}$; stages $S\geq2$ and caps
$\{B_s\}_{s=1}^{S}$; panel size $N\leq B_S$ and threshold $\tau$
\Require Fixed rescue rule $\mathsf{Rescue}$ satisfying
$\mathsf{Rescue}(D)\subseteq D$ (possibly empty)
\Ensure Best-so-far exact-$N$ panel $S_T^\star$, or $\emptyset$

\State $\mathcal C_0^{(0)}\gets\emptyset$,
$S_0^\star\gets\emptyset$; set $J_p(\emptyset):=-\infty$
\State Sample $\mathcal P_1=\{z_1^{(i)}\}_{i=1}^{B}$ from
$\pi_0(\cdot\mid p)$

\For{$t=1$ to $T$}
    \If{$t=1$}
        \State $\mathcal O_t\gets\mathcal P_1$
    \Else
        \State $\mathcal A_t\gets
        \mathrm{Top}_{B_{\mathrm{par}}}
        (\mathcal P_{t-1};\mathcal F(\cdot;p))$
        \Comment{Eq.~\eqref{eq:parent_selection}}
        \State Sample
        $\mathcal O_t=\{z'_i\}_{i=1}^{B_{\mathrm{off}}}$ from
        $q_{\mathrm{off}}(\cdot\mid\mathcal A_t,p)$
        \Comment{Eq.~\eqref{eq:app_offspring}}
    \EndIf

    \State Decode each $z\in\mathcal O_t$ for $M_{\mathrm{eval}}$ runs
    and evaluate $\mathcal F(z;p)$
    \Comment{Eq.~\eqref{eq:app_family_utility}}
    \State $\mathcal C_t^{(0)}\gets
    \mathcal C_{t-1}^{(0)}
    \cup\displaystyle\bigcup_{z\in\mathcal O_t}\mathcal M(z,p)$
    \Comment{Eq.~\eqref{eq:app_pool}}

    \For{$s=1$ to $S-1$}
        \State $\mathcal C_t^{(s)}\gets
        \mathrm{Top}_{B_s}
        (\mathcal C_t^{(s-1)};\succ_s)$
        \Comment{Eq.~\eqref{eq:app_stage_survivor}}
    \EndFor

    \State $\mathcal L_t\gets
    \mathrm{Top}_{B_S}
    (\mathcal C_t^{(S-1)};\succ_{S-1})$
    \State Fully evaluate all uncached molecules in $\mathcal L_t$
    \State $\mathcal E_t\gets\mathcal L_t$
    \State $\mathcal C_t^{(S)}\gets
    \mathrm{Top}_{B_S}(\mathcal E_t;\succ_S)$

    \If{$\mathfrak{F}_N(\mathcal C_t^{(S)};p,\tau)=\emptyset$}
        \State $\mathcal R_t\gets
        \mathsf{Rescue}
        (\mathcal C_t^{(S-1)}\setminus\mathcal L_t)$
        \State Fully evaluate all uncached molecules in $\mathcal R_t$
        \State $\mathcal E_t\gets\mathcal E_t\cup\mathcal R_t$
        \State $\mathcal C_t^{(S)}\gets
        \mathrm{Top}_{B_S}(\mathcal E_t;\succ_S)$
    \EndIf

    \State $\mathcal V_t\gets
    \mathfrak{F}_N(\mathcal C_t^{(S)};p,\tau)$
    \If{$\mathcal V_t\neq\emptyset$}
        \State Choose
        $S_t\in
        \operatorname*{arg\,max}_{Q\in\mathcal V_t}J_p(Q)$
        \Comment{Eq.~\eqref{eq:app_final_panel}}
    \Else
        \State $S_t\gets\emptyset$
    \EndIf

    \If{$t\geq2$}
        \State $\mathcal P_t\gets
        \mathrm{Top}_{B}
        (\mathcal P_{t-1}\cup\mathcal O_t;
        \mathcal F(\cdot;p))$
        \Comment{Eq.~\eqref{eq:app_population_update}}
    \EndIf

    \If{$J_p(S_t)>J_p(S_{t-1}^\star)$}
        \State $S_t^\star\gets S_t$
    \Else
        \State $S_t^\star\gets S_{t-1}^\star$
    \EndIf
\EndFor

\State \Return $S_T^\star$
\end{algorithmic}
\end{algorithm*}

\subsection{Family-Level Utility in Noise Space}
\label{app:family_utility}

For each input $z \in \mathcal{Z}$, the decoded family is denoted by
\begin{equation}
\mathcal{M}(z,p)=\{m_1,\dots,m_K\}.
\end{equation}
In the reported implementation, $K=M_{\mathrm{eval}}=6$, so each evaluated
input contributes six stochastic molecular decodes.
The family-level search utility is defined as
\begin{equation}
\mathcal{F}(z;p)=\Psi(\mathcal{M}(z,p);p).
\label{eq:app_family_utility}
\end{equation}

To reduce sensitivity to stochastic decoding, $\Psi$ is instantiated as an aggregation over the strongest portion of the decoded family rather than over a single best molecule.
Let
\begin{equation}
r_1 \succ_1 r_2 \succ_1 \cdots \succ_1 r_K
\end{equation}
denote the ranking of molecules in $\mathcal{M}(z,p)$ under the stage-1 feasibility-first ordering.
When this ordering is used inside $\Psi$, the diversity term is evaluated relative to the decoded family itself, i.e., $q_{\mathrm{div}}(m;\mathcal{M}(z,p))$.
A practical instantiation is
\begin{equation}
\Psi(\mathcal{M}(z,p);p)
=
\frac{1}{L}\sum_{\ell=1}^{L} h_1(r_\ell;p),
\qquad 1 \le L \le K,
\label{eq:psi_topl}
\end{equation}
where $L$ controls how much of the high-quality tail is used to score a family.
This scoring rule favors noise inputs whose decoded families are robustly strong rather than accidentally good.

\subsection{Evolutionary Search Operators}
\label{app:evolution_search}

At iteration $t$, parent selection is performed from the previous population $\mathcal{P}_{t-1}$ according to family-level utility:
\begin{equation}
\mathcal{A}_t = \mathrm{Top}_{B_{\mathrm{par}}}\big(\mathcal{P}_{t-1};\mathcal{F}(\cdot;p)\big),
\label{eq:parent_selection}
\end{equation}
where $\mathcal{F}(z;p)=\Psi(\mathcal{M}(z,p);p)$.

Variation is implemented through mutation, crossover, and random injection in the input noise space. A conceptual mixture-based view is
\begin{equation}
\begin{aligned}
q_{\mathrm{off}}(z' \mid \mathcal{A}_t,p)
={}& \alpha_{\mathrm{mut}}\, q_{\mathrm{mut}}(z' \mid \mathcal{A}_t,p) \\
&+ \alpha_{\mathrm{cross}}\, q_{\mathrm{cross}}(z' \mid \mathcal{A}_t,p) \\
&+ \alpha_{\mathrm{imm}}\, \pi_0(z' \mid p),
\end{aligned}
\label{eq:app_offspring}
\end{equation}
with $\alpha_{\mathrm{mut}}+\alpha_{\mathrm{cross}}+\alpha_{\mathrm{imm}}=1$.
In our implementation, $\pi_0(\cdot\mid p)$ denotes the random noise distribution used for population initialization and immigration. Samples are drawn independently from this distribution, while target-pair information enters through molecular decoding, dual-target evaluation, and downstream selection.

In the final implementation, this variation step is instantiated by crossover with rate $\alpha_{\mathrm{cross}}$, Gaussian mutation with scale $\sigma_{\mathrm{mut}}$ and decay $\gamma_{\mathrm{mut}}$, and $n_{\mathrm{imm}}$ random injections per generation; see Table~\ref{tab:key_hparams}.

A continuous-space instantiation is:
\begin{align}
q_{\mathrm{mut}} &: \quad z' = z + \epsilon, \qquad \epsilon \sim \mathcal{N}(0,\sigma_{\mathrm{mut}}^2 I), \\
q_{\mathrm{cross}} &: \quad z' = \lambda z_i + (1-\lambda) z_j, \notag\\
&\qquad \lambda \sim \mathcal{U}(0,1),\quad z_i,z_j \sim \mathcal{A}_t, \\
q_{\mathrm{imm}} &: \quad z' \sim \pi_0(\cdot \mid p).
\end{align}

Let $\mathcal{O}_t$ denote the input set evaluated at generation $t$, with
$\mathcal{O}_1=\mathcal{P}_1$. The cumulative molecular pool is defined
recursively as
\begin{equation}
\mathcal{C}_t^{(0)}
=
\begin{cases}
\displaystyle
\bigcup_{z\in\mathcal{O}_1}\mathcal{M}(z,p),
& t=1, \\[6pt]
\displaystyle
\mathcal{C}_{t-1}^{(0)}
\cup
\bigcup_{z\in\mathcal{O}_t}\mathcal{M}(z,p),
& t\geq 2.
\end{cases}
\label{eq:app_pool}
\end{equation}
Thus, previously decoded molecules remain available for subsequent
environmental selection and panel construction.

After panel construction, the next noise-space population is updated by elitist survival:
\begin{equation}
\mathcal{P}_t=\mathrm{Top}_{B}\big(\mathcal{P}_{t-1}\cup\mathcal{O}_t;\mathcal{F}(\cdot;p)\big).
\label{eq:app_population_update}
\end{equation}
This yields a standard population-based evolutionary loop in the input noise space.

\subsection{Stage-Wise Environmental Selection}
\label{app:stage_selection}

We index candidate pools by $s=0,\dots,S$, where stage $s$ maps $\mathcal{C}_t^{(s-1)}$ to $\mathcal{C}_t^{(s)}$.
At stage $s$, each candidate is assigned a scalar score
\begin{equation}
\begin{aligned}
h_s(m;p)
={}& f_{\mathrm{aff}}^{(s)}(m;p)
+ \beta_{\mathrm{chem}} q_{\mathrm{chem}}(m) \\
&+ \beta_{\mathrm{div}} q_{\mathrm{div}}(m;\mathcal{C}_t^{(s-1)}).
\end{aligned}
\label{eq:app_stage_score}
\end{equation}

\paragraph{Balance-aware affinity term.}
Let $A_s(m,t_a)$ and $A_s(m,t_b)$ denote stage-$s$ affinity estimates for the two targets, represented on a common larger-is-better utility scale.
We instantiate
\begin{equation}
\begin{aligned}
f_{\mathrm{aff}}^{(s)}(m;p)
={}& w_a A_s(m,t_a)+w_b A_s(m,t_b) \\
&-\lambda_{\mathrm{bal}}\big|A_s(m,t_a)-A_s(m,t_b)\big|,
\end{aligned}
\label{eq:app_balance_aff}
\end{equation}
where $\lambda_{\mathrm{bal}}>0$ penalizes asymmetric candidates.
When a stage is based on raw docking scores, these scores are first converted to the same utility scale by sign flip (or an equivalent monotone transformation).
In our implementation, the stage-specific estimators $\{A_s\}_{s=1}^{S}$ are organized as a coarse-to-fine docking pipeline.
Early stages use inexpensive docking-derived surrogates, including rapid docking-score evaluation and local pose refinement, while the final stage uses a full docking search on the reduced frontier.
Thus, REUSE allocates high-fidelity evaluation only to a compact set of survivors selected by evolution.

\paragraph{Chemistry term.}
The chemistry term $q_{\mathrm{chem}}(m)$ is a soft quality term that rewards chemically desirable outputs.
In our experimental instantiation, it is computed from standard molecular quality indicators, specifically QED and SA \cite{bickerton2012qed,ertl2009sascore}. Before aggregation, these quantities are transformed to a common larger-is-better scale so that candidates with better drug-likeness and synthetic accessibility receive larger scores.

\paragraph{Diversity term.}
A simple diversity instantiation is
\begin{equation}
q_{\mathrm{div}}(m;\mathcal{C}_t^{(s-1)})
=
\min_{m' \in \mathcal{C}_t^{(s-1)}\setminus\{m\}} \Delta(m,m').
\label{eq:app_div_score}
\end{equation}
If $|\mathcal{C}_t^{(s-1)}|=1$, we set $q_{\mathrm{div}}(m;\mathcal{C}_t^{(s-1)})=0$.
In our implementation, the pairwise dissimilarity is defined by Morgan-fingerprint distance,
\begin{equation}
\Delta(m_i,m_j)
=
1-\mathrm{Tan}\!\big(\mathrm{FP}(m_i),\mathrm{FP}(m_j)\big),
\label{eq:app_morgan_distance}
\end{equation}
where $\mathrm{FP}(\cdot)$ denotes the Morgan fingerprint and $\mathrm{Tan}(\cdot,\cdot)$ is Tanimoto similarity \cite{rogers2010ecfp}.

\paragraph{Feasibility-first preference.}
Selection at stage $s$ is lexicographic:
\begin{equation}
\begin{aligned}
m_i \succ_s m_j
\quad \Longleftrightarrow \quad
&\big(c(m_i,p),\, h_s(m_i;p)\big) \\
&> \big(c(m_j,p),\, h_s(m_j;p)\big),
\end{aligned}
\label{eq:app_lexicographic}
\end{equation}
where $>$ denotes lexicographic comparison.
Thus feasible candidates always dominate infeasible ones, and within the same feasibility class larger $h_s$ is preferred.
In our implementation, $c(m,p)$ includes molecular validity and chemistry-admissibility checks, together with the availability of valid affinity estimates required at the corresponding stage.
For candidates arising from local noise-space perturbations, feasibility additionally enforces bounded structural deviation relative to the corresponding parent molecule, implemented through similarity- and size-based constraints, so that the search remains local and chemically controlled.
The preliminary stages retain
\begin{equation}
\mathcal{C}_t^{(s)}
=
\mathrm{Top}_{B_s}
\big(\mathcal{C}_t^{(s-1)};\succ_s\big),
\qquad s=1,\dots,S-1.
\label{eq:app_stage_survivor}
\end{equation}

The terminal stage first forms the base shortlist
\[
\mathcal{L}_t
=
\mathrm{Top}_{B_S}
\big(\mathcal{C}_t^{(S-1)};\succ_{S-1}\big)
\]
and fully evaluates its uncached members. Let
$\mathcal{E}_t=\mathcal{L}_t$ initially. If the resulting terminal
survivor pool contains no feasible exact-$N$ panel, the fixed rescue rule
selects additional candidates
$\mathcal{R}_t\subseteq
\mathcal{C}_t^{(S-1)}\setminus\mathcal{L}_t$.
After fully evaluating their uncached members, we update
$\mathcal{E}_t\gets\mathcal{E}_t\cup\mathcal{R}_t$ and recompute
\[
\mathcal{C}_t^{(S)}
=
\mathrm{Top}_{B_S}\big(\mathcal{E}_t;\succ_S\big).
\]
Thus, $B_S$ caps the terminal survivor pool rather than the number of
full-evaluation calls.

\begin{figure*}[t]
    \centering
    \includegraphics[width=0.95\linewidth]{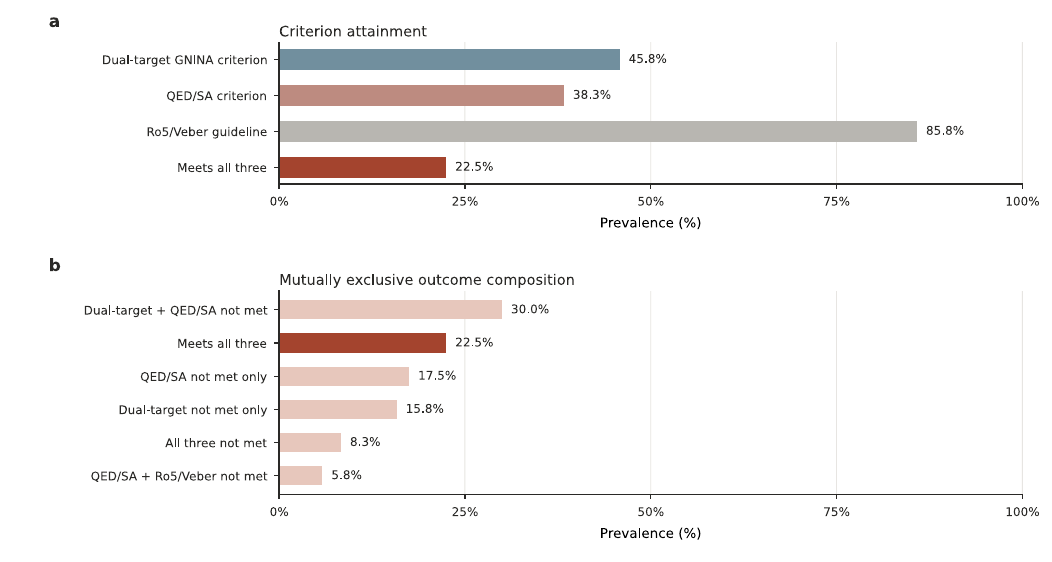}
    \caption{
    Complementary binding and medicinal-chemistry criterion attainment
    within the pre-selection \textsc{REUSE} candidate pool, pooled across
    random seeds. This pool is analyzed before hard-admissibility filtering
    and final set-level panel construction.
    (a) The criteria are evaluated independently rather than as a
    sequential filtering funnel. The dual-target GNINA criterion requires
    non-negative reference-normalized GNINA margins for both targets; the
    QED/SA criterion requires QED $\geq 0.5$ and normalized SA $\geq 0.5$;
    and the Ro5/Veber criterion requires both guidelines to be satisfied.
    ``Meets all three'' denotes the intersection of these criteria.
    (b) Mutually exclusive observed outcome profiles. Zero-frequency
    profiles are omitted. ``Meets all three'' in this figure does not
    include PAINS or Brenk structural-alert filtering.
    }
    \label{fig:criterion}
\end{figure*}

\subsection{Set-Level Final Panel Construction}
\label{app:panel_selection}

The final output is selected from the terminal pool $\mathcal{C}_t^{(\mathrm{final})}:=\mathcal{C}_t^{(S)}$ at the \emph{set} level.
We define a panel utility
\begin{equation}
\begin{aligned}
J_p(S)
={}& \frac{1}{|S|}\sum_{m\in S}
\left[
\eta_{\mathrm{aff}} f_{\mathrm{aff}}^{(S)}(m;p)
+ \eta_{\mathrm{chem}} q_{\mathrm{chem}}(m)
\right] \\
&+ \eta_{\mathrm{div}} D(S),
\end{aligned}
\label{eq:app_panel_utility}
\end{equation}
where
\begin{equation}
D(S)=\frac{1}{|S|(|S|-1)}
\sum_{m_i \neq m_j \in S}\Delta(m_i,m_j)
\label{eq:app_panel_div}
\end{equation}
is the average pairwise dissimilarity within the panel, with $\Delta$ instantiated as the Morgan-fingerprint distance in Eq.~\eqref{eq:app_morgan_distance}.

Define the feasible panel family
\begin{equation}
\begin{aligned}
\mathfrak{F}_N\!\left(\mathcal{C}_t^{(\mathrm{final})};p,\tau\right)
&:= \Bigl\{\, S \;\Bigm| \\
&\quad S \subseteq \mathcal{C}_t^{(\mathrm{final})},\quad |S|=N, \\
&\quad c(m,p)=1\ \forall m\in S, \\
&\quad \Delta(m_i,m_j)\ge\tau\ \forall m_i\neq m_j\in S
\,\Bigr\}.
\end{aligned}
\label{eq:app_feasible_family}
\end{equation}

If $\mathfrak{F}_N(\mathcal{C}_t^{(\mathrm{final})};p,\tau)\neq\emptyset$, the final panel is given by
\begin{equation}
S_t
\in
\arg\max_{S \in \mathfrak{F}_N(\mathcal{C}_t^{(\mathrm{final})};p,\tau)}
J_p(S).
\label{eq:app_final_panel}
\end{equation}
If $\mathfrak{F}_N(\mathcal{C}_t^{(\mathrm{final})};p,\tau)=\emptyset$, we set $S_t=\emptyset$ and use the convention $J_p(\emptyset)=-\infty$.

In practice, the nonempty case can be solved greedily or by exact subset search when the terminal pool is sufficiently small.

\begin{figure*}[t]
    \centering
    \includegraphics[width=1\linewidth]{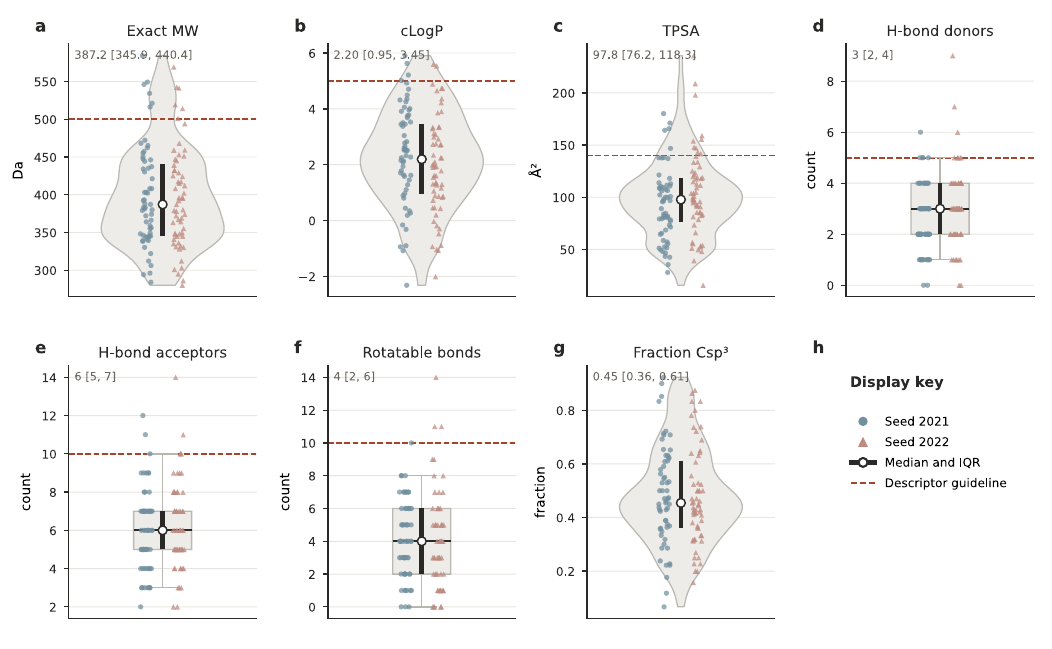}
    \caption{
    Extended physicochemical distributions of the final \textsc{REUSE}
    candidates. Blue circles and terracotta triangles denote the two random
    seeds. The open black circle and thick black interval indicate the
    pooled median and interquartile range, respectively; the corresponding
    numerical values are printed in each panel. Dashed lines indicate the
    descriptor guidelines used for the rule-based summaries: molecular
    weight $\leq 500$ Da, cLogP $\leq 5$, TPSA $\leq 140$~\AA$^2$,
    hydrogen-bond donors $\leq 5$, hydrogen-bond acceptors $\leq 10$, and
    rotatable bonds $\leq 10$. Fraction Csp$^3$ is shown descriptively
    without a hard threshold.
    }
    \label{fig:physchem}
\end{figure*}

\subsection{Conceptual Objective vs. Operational Procedure}
\label{app:objective_connection}

Eq.(3) in the main text defines the conceptual constrained multi-objective optimization target over a searched input set $\mathcal{Q}$ and a final panel $S$ drawn from the pooled decoded universe $\mathcal{M}(\mathcal{Q},p)$.
REUSE approximates this target through four coupled components:
\begin{enumerate}
    \item \textbf{Noise-space family optimization:} score each searched input $z$ using $\Psi(\mathcal{M}(z,p);p)$;
    \item \textbf{Variation in noise space:} generate new offspring from selected parents through mutation, crossover, and immigration;
    \item \textbf{Environmental selection:} filter decoded molecules under feasibility, affinity, chemistry, and diversity criteria;
    \item \textbf{Panel construction and incumbent tracking:} select a feasible and diverse subset maximizing $J_p(S)$ and keep the best-so-far panel across iterations.
\end{enumerate}
Operationally, iteration $t$ applies environmental selection and panel
construction to the cumulative pool of molecules decoded up to that point.
The resulting panel $S_t$ is compared with the previous incumbent, allowing
REUSE to approximate the conceptual pooled objective through incremental
search and best-so-far incumbent tracking.
This decomposition preserves the multi-objective optimization perspective while avoiding direct optimization over the stochastic decoding map of the frozen generator.

\section{Implementation and Reproducibility}

\subsection{Implementation Details}
\label{app:implementation_details}

\begin{table*}[t]
\centering
\small
\renewcommand{\arraystretch}{1.03}
\setlength{\tabcolsep}{3.5pt}

\begin{tabularx}{\textwidth}{@{}lXr@{\qquad}lXr@{}}
\toprule
\multicolumn{3}{c}{Search and decoding settings} &
\multicolumn{3}{c}{Selection and panel settings} \\
\cmidrule(lr){1-3}
\cmidrule(lr){4-6}
Symbol & Meaning & Value & Symbol & Meaning & Value \\
\midrule

$B$
& population size
& 4
& $S$
& number of selection stages
& 2 \\

$B_{\mathrm{par}}$
& parent pool size
& 3
& $B_1$
& stage-1 survivor cap
& 40 \\

$T$
& total latent generations, including initialization
& 3
& $B_2$
& base-shortlist and terminal-survivor cap
& 20 \\

$B_{\mathrm{off}}$
& offspring size per generation
& 4
& $N$
& final panel size
& 10 \\

$d_z$
& input-space dimensionality
& 10
& $\beta_{\mathrm{chem}}^{\mathrm{search}}$
& search-stage chemistry weight
& 0.40 \\

$\alpha_{\mathrm{cross}}$
& crossover rate
& 0.35
& $\beta_{\mathrm{chem}}^{\mathrm{rerank}}$
& reranking chemistry weight
& 0.60 \\

$\sigma_{\mathrm{mut}}$
& initial mutation scale
& 0.42
& $\beta_{\mathrm{div}}^{\mathrm{subset}}$
& subset diversity weight
& 0.05 \\

$\gamma_{\mathrm{mut}}$
& mutation-scale decay
& 0.82
& $\beta_{\mathrm{bal}}^{\mathrm{subset}}$
& subset balance weight
& 0.02 \\

$n_{\mathrm{imm}}$
& random injections per generation
& 1
& $\lambda_{\mathrm{bal}}^{\mathrm{proxy}}$
& proxy-stage balance weight
& 0.25 \\

$M_{\mathrm{eval}}$
& stochastic decodes per evaluated input
& 6
& $\tau_{\mathrm{QED}}^{\mathrm{hard}}$
& hard QED admissibility floor
& 0.50 \\

$B_{\mathrm{dec}}$
& total raw decoded-candidate budget
& 72
& $\tau_{\mathrm{SA}}^{\mathrm{hard}}$
& hard SA admissibility floor
& 0.50 \\

$T_{\mathrm{train}}$
& diffusion horizon during training
& 1000
& $\tau_{\mathrm{QED}}^{\mathrm{good}}$
& preferred QED quality floor
& 0.55 \\

$T_{\mathrm{infer}}$
& reverse denoising steps during inference
& 100
& $\tau_{\mathrm{SA}}^{\mathrm{good}}$
& preferred SA quality floor
& 0.55 \\

& &
&
$\rho_{\max}$
& maximum pairwise Tanimoto similarity
& 0.85 \\

& &
&
$\tau$
& minimum pairwise Morgan dissimilarity
& 0.15 \\

\bottomrule
\end{tabularx}

\caption{
Key hyperparameters and implementation settings for the default
\textsc{REUSE} configuration used in the main experiments.
Variant-specific settings are described in the corresponding supplementary
sections.
}
\label{tab:key_hparams}
\end{table*}

Here $T=3$ includes the initialized population as the first latent
generation. Thus, REUSE evaluates
$B+(T-1)B_{\mathrm{off}}=12$ inputs and decodes each input
$M_{\mathrm{eval}}=6$ times, yielding 72 raw decoded candidates per target
pair. At iteration $t$, all molecules decoded up to that point remain
available in the cumulative pool $\mathcal{C}_t^{(0)}$. The rescue rule
only adds full evaluations of already decoded candidates and therefore
does not increase this decoding budget. Stage-$S$ evaluations are cached by molecule and reused across
generations.

Unless otherwise stated, we instantiate REUSE on top of a frozen pretrained single-target diffusion generator, \emph{TargetDiff} \cite{guan2023targetdiff}, and keep the generator fixed throughout the main experiments.
The model is loaded from a pretrained checkpoint and reconstructed from the checkpoint-embedded training configuration.
The diffusion horizon is 1000 during training, while inference uses 100 reverse denoising steps.

For search in the input-noise space, we use a population-based evolutionary
procedure with mutation, crossover, and random injection. The initialized
population counts as the first of three latent generations, and every
evaluated input is decoded six times. Newly decoded molecules are retained
in a cumulative candidate pool, while the best-so-far panel is tracked
across generations. The final panel size is fixed to $N=10$, and the initial
population is sampled from $\pi_0(\cdot\mid p)$.

For candidate generation, ligand size, initial coordinates, and atom types follow the standard initialization procedure of the frozen generator.

Environmental selection follows a coarse-to-fine docking pipeline.
An initial screening stage uses fast docking-based signals for both targets, followed by full docking on a reduced shortlist.
An optional rescue stage adds a small number of additional fully docked candidates when necessary.
The final panel is selected by reusing the previously computed docking results rather than launching an additional docking round.
Unless otherwise stated, the final evaluation stages use docking exhaustiveness 32.

For chemistry-aware ranking, we use a weighted chemistry-control score built from QED and SA together with auxiliary medicinal-chemistry penalties and rewards.
For final panel construction, we additionally enforce hard chemistry admissibility thresholds and a maximum pairwise Tanimoto similarity of 0.85.
Following the main text, structural diversity is measured using fingerprint-level molecular dissimilarity.

All experiments are run in the \texttt{dual\_target} environment with fixed runtime settings for reproducibility.
Our main runs are conducted on a server with 8 NVIDIA RTX 4090 GPUs.

\begin{figure*}
    \centering
    \includegraphics[width=1\linewidth]{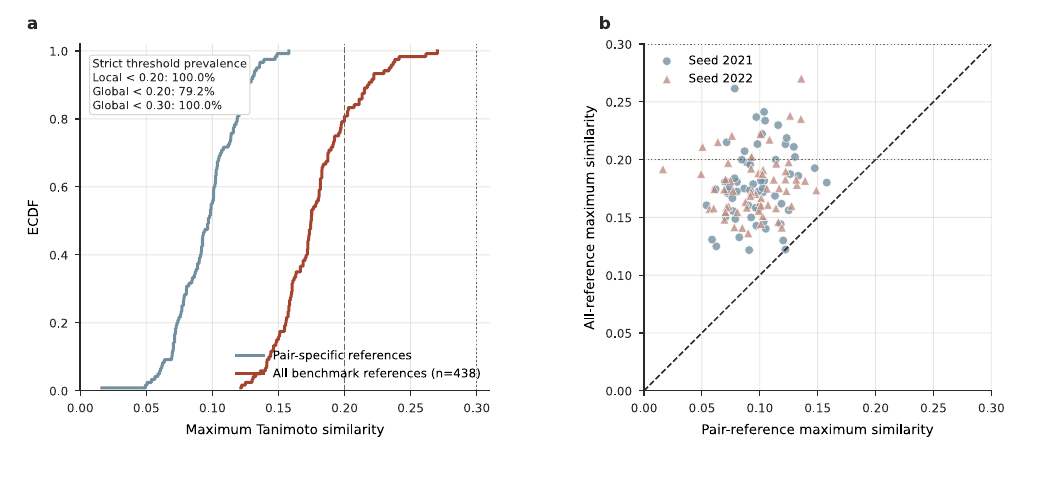}
    \caption{
    Structural similarity of the final \textsc{REUSE} candidates to
    benchmark references, computed using Tanimoto similarity with radius-2,
    2048-bit Morgan fingerprints \cite{rogers2010ecfp}.
    (a) Empirical cumulative distributions of each candidate's maximum
    similarity to the two references associated with its target pair and
    to the complete set of 438 benchmark references. Vertical lines mark
    similarity values of 0.20 and 0.30. The annotation reports strict
    threshold proportions using $<$, whereas the ECDF follows the standard
    $\leq$ convention.
    (b) Candidate-level comparison of pair-specific and all-benchmark
    maximum similarity. Marker shape denotes the random seed, and the
    diagonal is the identity line. Because the complete reference set
    contains the pair-specific references, the all-benchmark maximum cannot
    be smaller than the pair-specific maximum.
    }
    \label{fig:novelty}
\end{figure*}

\subsection{Experimental Instantiation and Reproducibility Details}
\label{app:reproducibility}

For reproducibility, we summarize here the concrete experimental instantiation of REUSE.
The frozen generator $G_{\phi}$, the random noise prior $\pi_0(\cdot \mid p)$, the stage-wise docking pipeline, and all optimization hyperparameters are fixed throughout each experiment and reported explicitly in the corresponding experimental subsection or summary table.

\paragraph{Benchmark protocol.}
We follow the dual-target benchmark protocol of Zhou et al.~\cite{zhou2024reprogramming} and evaluate on the full set of 12,917 target pairs over 438 unique targets. We do not subsample target pairs for the main benchmark. Each target pair is associated with two pocket-level protein--ligand complex contexts and one benchmark reference molecule for each target, as provided by the benchmark construction. Available complex structures are taken from curated protein--ligand complexes, while missing structures are completed by the benchmark using PDB or AlphaFold DB structures together with pocket detection and docking-based pocket construction.

For each target in a pair, the binding site and docking region are defined by the benchmark-provided pocket specification. All generated candidates are independently docked or rescored against both target pockets under the same receptor-preparation and scoring pipeline. This ensures that performance reflects dual-target compatibility under a consistent evaluation protocol.

All methods are evaluated under the same fixed-size panel protocol. For each target pair, a method is expected to return a candidate panel of size $N=10$. If a method produces more than 10 valid candidates, the reported panel is selected using the common feasibility, dual-target affinity, chemistry, and diversity criteria described in the main text and Appendix~\ref{app:stage_selection}. If a method produces fewer than 10 valid candidates, we retain all valid candidates and treat the missing panel slots as unsuccessful recoveries for panel-level feasibility and dual-hit statistics. Target pairs for which docking or molecule validity fails are not silently removed from the benchmark; instead, invalid or unevaluable candidates receive failed feasibility status. Thus, all reported averages and medians are computed over the same benchmark task set, and differences between methods reflect recovery quality rather than changes in evaluation coverage.

\paragraph{Evaluation metrics.}
For a target pair $p=(t_1,t_2)$, let $r_i$ denote the benchmark reference
molecule associated with target $t_i$. We denote the Vina Dock score and
GNINA CNNaffinity of candidate $m$ against target $t_i$ by $V_i(m)$ and
$G_i(m)$, respectively. Lower Vina scores and higher GNINA CNNaffinity
scores indicate stronger predicted binding. The target-specific scores
reported as P-1 and P-2 are $V_1(m)$ and $V_2(m)$, or $G_1(m)$ and
$G_2(m)$, respectively. The conservative dual-target scores are
\begin{equation}
\begin{aligned}
V_{\max}(m;p)
&=\max\{V_1(m),V_2(m)\},\\
G_{\min}(m;p)
&=\min\{G_1(m),G_2(m)\}.
\end{aligned}
\end{equation}
where Max Vina represents the weaker of the two target-specific docking
scores, while Min GNINA represents the weaker of the two CNNaffinity
scores.

We define the reference-normalized Vina and GNINA margins as
\begin{equation}
\begin{aligned}
\Delta_i^{\mathrm{Vina}}(m)
&=V_i(r_i)-V_i(m), \\
\qquad
\Delta_i^{\mathrm{GNINA}}(m)
&=G_i(m)-G_i(r_i).
\end{aligned}
\end{equation}
Thus, $\Delta_i^{\mathrm{GNINA}}(m)\geq 0$ means that candidate $m$
achieves a GNINA CNNaffinity score no lower than the benchmark reference
ligand for target $t_i$. Because lower Vina scores indicate stronger
predicted binding, $\Delta_i^{\mathrm{Vina}}(m)>0$ means that the candidate
strictly outperforms the corresponding reference under Vina.

The Dual High Affinity indicator is
\begin{equation}
H_{\mathrm{dual}}(m;p)
=
\mathbf{1}\!\left[
\Delta_1^{\mathrm{Vina}}(m)>0
\land
\Delta_2^{\mathrm{Vina}}(m)>0
\right],
\end{equation}
and its reported rate is the percentage of candidates that outperform both
target-specific references according to Vina.

QED is reported on its standard $[0,1]$ scale. We convert the original
synthetic-accessibility score, for which lower values indicate easier
synthesis, to a larger-is-better scale as
\begin{equation}
\mathrm{SA}(m)
=
\operatorname{round}\!\left(
\frac{10-\mathrm{SA}_{\mathrm{raw}}(m)}{9},
2
\right).
\end{equation}
The resulting value is rounded to two decimal places before applying the
$\mathrm{SA}(m)\geq 0.5$ criterion.

The Consensus Hit indicator jointly requires reference-beating Vina scores,
non-negative reference-normalized GNINA CNNaffinity margins on both targets,
and the QED and SA criteria:
\begin{equation}
\begin{aligned}
H_{\mathrm{cons}}(m;p)
={}&H_{\mathrm{dual}}(m;p)\\[-0.2em]
&\times\mathbf{1}[\Delta_1^{\mathrm{GNINA}}(m)\geq 0]
\\[-0.2em]
&\times\mathbf{1}[\Delta_2^{\mathrm{GNINA}}(m)\geq 0]\\[-0.2em]
&\times
\mathbf{1}[\mathrm{QED}(m)\geq 0.5]\,
\mathbf{1}[\mathrm{SA}(m)\geq 0.5].
\end{aligned}
\end{equation}
The Consensus Hit rate is the percentage of candidates satisfying all of
these conditions simultaneously.

For the optimizer comparison, let $R_i(m)$ denote the Vinardo score of
candidate $m$ against target $t_i$. Because lower Vinardo scores are better,
its reference-normalized margin is
\begin{equation}
\Delta_i^{\mathrm{Vinardo}}(m)
=
R_i(r_i)-R_i(m).
\end{equation}
The scorer-specific dual-target hit indicators are
\begin{equation}
\begin{aligned}
H_{\mathrm{GNINA}}(m;p)
&=
\mathbf{1}\!\left[
\Delta_1^{\mathrm{GNINA}}(m)\geq 0
\right.\\[-0.2em]
&\qquad\left.\land\;\Delta_2^{\mathrm{GNINA}}(m)\geq 0
\right],\\
H_{\mathrm{Vinardo}}(m;p)
&=
\mathbf{1}\!\left[
\Delta_1^{\mathrm{Vinardo}}(m)\geq 0
\right.\\[-0.2em]
&\qquad\left.\land\;\Delta_2^{\mathrm{Vinardo}}(m)\geq 0
\right].
\end{aligned}
\end{equation}
Let $H_s(m;p)$ denote the corresponding dual-target hit indicator for
scorer $s\in\{\mathrm{GNINA},\mathrm{Vinardo}\}$.

We further define
\begin{equation}
H_{\mathrm{chem}}(m)
=
\mathbf{1}\!\left[
\mathrm{QED}(m)\geq 0.5
\land
\mathrm{SA}(m)\geq 0.5
\right],
\end{equation}
and let $H_{\mathrm{alert}}(m)$ indicate that the candidate triggers neither
a PAINS nor a Brenk alert. Let $H_{\mathrm{Ro5}}(m)$ and
$H_{\mathrm{Veber}}(m)$ denote the corresponding rule-pass indicators.
Lipinski Ro5 pass allows at most one violation among molecular weight
$\leq 500$ Da, cLogP $\leq 5$, hydrogen-bond donors $\leq 5$, and
hydrogen-bond acceptors $\leq 10$. Veber pass requires TPSA
$\leq 140$~\AA$^2$ and at most 10 rotatable bonds.

The three optimizer-comparison indicators are
\begin{equation}
\begin{aligned}
H_{\mathrm{feasible},s}(m;p)
&= H_s(m;p)H_{\mathrm{chem}}(m),\\
H_{\mathrm{alert\text{-}free},s}(m;p)
&= H_s(m;p)H_{\mathrm{alert}}(m),\\
H_{\mathrm{qualified},s}(m;p)
&= H_s(m;p)
   H_{\mathrm{chem}}(m)
   H_{\mathrm{Ro5}}(m)\\
&\quad {}\times
   H_{\mathrm{Veber}}(m)
   H_{\mathrm{alert}}(m).
\end{aligned}
\end{equation}
Thus, a feasible hit satisfies the scorer-specific dual-target binding
criterion together with QED and SA thresholds; an alert-free hit satisfies
the binding criterion without triggering either structural-alert catalog;
and a qualified hit additionally satisfies the QED, SA, Ro5, Veber,
PAINS-free, and Brenk-free criteria. The reported rates are the percentages
of candidates satisfying the corresponding indicators. PAINS-free and
Brenk-free mean that no alert from the respective catalog is detected.

Structural diversity is computed from radius-2 Morgan fingerprints. For
two molecules,
\begin{equation}
\Delta(m_i,m_j)
=
1-\mathrm{Tan}\!\left(
\mathrm{FP}(m_i),\mathrm{FP}(m_j)
\right),
\end{equation}
and the diversity of a reported panel $S$ is
\begin{equation}
D(S)
=
\frac{1}{|S|(|S|-1)}
\sum_{m_i\neq m_j\in S}
\Delta(m_i,m_j).
\end{equation}
Higher values therefore indicate lower structural redundancy within the
selected panel. Invalid, unevaluable, or missing required panel entries are
treated as failures when computing indicator-based rates, following the
fixed-panel protocol described above.

For the cheap-to-full evaluator analysis, let
$\mathcal{F}_{\mathrm{cheap}}^{(k)}$ and
$\mathcal{F}_{\mathrm{full}}^{(k)}$ denote the top-$k$ frontiers selected by
the cheap and full evaluators, respectively. Frontier overlap is defined as
\begin{equation}
\mathrm{Ov}@k
=
\frac{
\left|
\mathcal{F}_{\mathrm{cheap}}^{(k)}
\cap
\mathcal{F}_{\mathrm{full}}^{(k)}
\right|
}{k}.
\end{equation}
Spearman, Pearson, and Kendall measure global ranking or score agreement
over the complete candidate pool. Per-molecule cost is the elapsed
evaluation time divided by the number of evaluated molecules, and the
reported runtime ratio is
\begin{equation}
R_{\mathrm{time}}
=
\frac{T_{\mathrm{full}}}{T_{\mathrm{cheap}}}.
\end{equation}

\paragraph{Frozen generator and random noise prior.}
In each experiment, REUSE operates on a fixed pretrained single-target generator $G_{\phi}$ without any retraining or modification of its denoising dynamics.
The random noise prior $\pi_0(\cdot \mid p)$ defines the distribution used for population initialization and immigration and remains fixed throughout the run.

\paragraph{Evolutionary hyperparameters.}
The final implementation instantiates the latent search with population size $B$, parent pool size $B_{\mathrm{par}}$, number of generations $T$, latent dimensionality $d_z$, crossover rate $\alpha_{\mathrm{cross}}$, mutation scale $\sigma_{\mathrm{mut}}$, mutation decay $\gamma_{\mathrm{mut}}$, and random injections per generation $n_{\mathrm{imm}}$.
Additional implementation settings include the number of latent evaluation samples, chemistry weights used during search and reranking, subset-level diversity and balance weights, chemistry admissibility thresholds, and the final diversity threshold.
Key numerical settings used in the reported experiments are summarized in
Table~\ref{tab:key_hparams}.

\paragraph{Ablation configurations.}
All ablated variants use the same frozen generator, target-pair inputs, candidate-panel size, docking pipeline, and evaluation metrics as the full REUSE configuration. They differ only in the component being ablated. In \emph{w/o search}, we remove evolutionary feedback in the input-noise space: parent selection, mutation, crossover, and elitist population update are disabled, and the same candidate budget is instead obtained by drawing independent noise inputs from the random noise prior $\pi_0(\cdot \mid p)$. The resulting decoded candidates are still processed by the same stage-wise environmental selection and final panel-construction procedure. This variant tests whether structured input-space search is necessary beyond budget-matched sampling from the frozen generator.

In \emph{w/o balance}, we keep input-space search, chemistry control, diversity control, and all docking stages unchanged, but remove explicit cross-target balance penalties by setting the balance weights to zero, including $\lambda_{\mathrm{bal}}^{\mathrm{proxy}}=0$ in the stage-wise affinity score and $\beta_{\mathrm{bal}}^{\mathrm{subset}}=0$ in final panel construction. The variant still evaluates both targets, but no longer penalizes asymmetric candidates that bind strongly to only one target. This isolates the contribution of balance-aware selection.

In \emph{w/o chemctrl}, we keep the search procedure and balance-aware affinity terms unchanged, but remove chemistry-control terms from ranking and selection by setting the chemistry weights to zero and disabling the hard QED/SA admissibility floors during panel construction. Validity checks and docking-result availability are still enforced, and the same diversity constraint is retained. This variant tests whether the chemistry-control component is needed to maintain drug-like and synthetically accessible candidates, rather than merely improving docking scores.

\paragraph{Baseline reproduction protocol.}
We evaluate all methods on the same dual-target benchmark using a common
downstream scoring and panel-reporting protocol. For methods covered by the
released benchmark of Zhou et al.~\cite{zhou2024reprogramming}, we use the
official evaluation pipeline and released outputs whenever available. Methods
requiring local generation are reproduced using their public implementations
and recommended pretrained checkpoints or default hyperparameters, without
tuning on the evaluation target pairs. All generated molecules, including
those from the baselines and REUSE, are subsequently evaluated using the same
dual-target docking and molecular-property pipeline.

Except for MDRL, each method is assigned a generation budget of at most 72
decoded candidate molecules per target pair. MDRL follows its original
higher-budget converged reinforcement-learning protocol, since its policy
optimization requires a substantially larger sampling budget to converge.
This exception applies only to the generation budget; all methods use the same
downstream scoring and fixed-panel reporting procedure.

For each target pair, the generated candidate pool is scored against both
targets, after which a panel of $N=10$ candidates is constructed using the
common feasibility, dual-target affinity, chemistry, and diversity criteria.
If fewer than $N$ valid candidates are available, all valid candidates are
retained and the missing panel entries are treated as failures for
indicator-based statistics. This prevents methods producing invalid or
one-sided candidates from receiving an evaluation advantage.

For single-target structure-based generators such as Pocket2Mol and
TargetDiff, candidates are generated separately for the two pockets, pooled,
and then re-evaluated against both targets. Linker-design methods such as
DiffLinker and LinkerNet use the benchmark-provided pocket and reference
information before the same dual-target rescoring procedure is applied.
ParetoDrug and the dual-target methods ComPDIFF and DualDiff are evaluated
from their generated candidate pools under the 72-candidate budget. MDRL is
evaluated using the outputs of its higher-budget converged reinforcement-learning
procedure. All candidate pools are subsequently processed by the same docking,
chemistry assessment, diversity filtering, and fixed-panel reporting pipeline.

\paragraph{Stage-wise docking pipeline.}
The stage-specific affinity estimators $\{A_s\}_{s=1}^{S}$ are instantiated by a fixed coarse-to-fine docking pipeline.
Early stages use inexpensive docking-derived surrogates for rapid screening, intermediate stages apply lightweight structural refinement when used, and the final stage applies a full docking search on the reduced frontier.
The number of stages $S$, the associated software configuration, and the corresponding runtime settings are reported together with the experimental protocol.

\paragraph{Feasibility checks.}
The hard feasibility indicator $c(m,p)$ includes molecular validity, chemistry-admissibility checks, and the availability of valid stage-specific affinity estimates.
For candidates arising from local noise-space perturbations, feasibility may additionally enforce bounded structural deviation relative to the corresponding parent molecule so as to keep the search chemically controlled.

For clarity, the key numerical values used in the reported experiments are summarized in Table~\ref{tab:key_hparams}.

\section{Additional Experiments and Analyses}

\subsection{Cross-Scorer Robustness and Structural Sanity Checks}
\label{app:oracle_structural_sanity_checks}

Beyond the optimization-independent GNINA CNNaffinity results in the main
comparison, we evaluate the compared final candidate panels under a common
post hoc protocol using Vinardo, AutoDock4, and RTMScore. None of the
post hoc scores summarized here is used during REUSE generation, search,
or final selection.

Figure~\ref{fig:cross_scorer_radar} reports performance separately on T1
and T2. REUSE outperforms DualDiff on all eight scorer--target combinations
and MDRL on six of eight, with MDRL leading only under Vinardo. The default EA configuration also exceeds CMA-ES and BO on seven of eight combinations in each comparison. Thus, the overall advantage remains visible across complementary empirical and learned scoring functions rather than being confined to the Vina optimization objective.

At the candidate-ranking level, we further examine whether the original
Vina ordering remains stable under alternative Vina-family configurations.
We compare the original Vina ranking with Vinardo and the Vina-like score
reported internally by GNINA. The latter, referred to as GNINA-internal
Vina, is distinct from the GNINA CNNaffinity score reported in the main
comparison and is used here only as an alternative Vina-family scoring
output. As shown in Table~\ref{tab:oracle_cross_scorer}, Vina is highly
consistent with Vinardo, with Spearman correlations of 0.959 and 0.994 on
the two targets and 100\% top-5 overlap on both targets. GNINA-internal Vina shows similarly high agreement, with Spearman correlations of 0.997 and 0.985. These results demonstrate that the recovered ranking is stable
across the evaluated Vina-family scoring configurations and implementations.

\begin{table}[t]
\centering
\small
\setlength{\tabcolsep}{3pt}
\begin{tabular}{@{}llccc@{}}
\toprule
Comparison & Target & $\rho$ & $\tau$ & Ov@5 \\
\midrule
Vina--Vinardo & T1 & 0.959 & 0.883 & 100\% \\
Vina--Vinardo & T2 & 0.994 & 0.967 & 100\% \\
Vina--GNINA-internal Vina & T1 & 0.997 & 0.983 & 100\% \\
Vina--GNINA-internal Vina & T2 & 0.985 & 0.950 & 80\% \\
\bottomrule
\end{tabular}
\caption{
Vina-family rank consistency for the recovered candidate panel. The
original Vina ranking is compared with Vinardo and GNINA-internal Vina.
The latter denotes GNINA's Vina-like scoring output and is distinct from
the GNINA CNNaffinity score reported in the main comparison. Here $\rho$
and $\tau$ denote Spearman and Kendall rank correlations, respectively,
and Ov@5 denotes top-5 overlap.
}
\label{tab:oracle_cross_scorer}
\end{table}

Second, we evaluate the benchmark reference ligand for each target under the
same docking setup. The T1 reference ligand receives Vina and Vinardo scores
of $-10.960$ and $-7.331$, respectively, and forms 13 residue-level contacts
in the pocket. The T2 reference ligand similarly receives Vina and Vinardo
scores of $-5.483$ and $-4.907$, respectively, with 11 detected pocket
contacts. These results support the validity of the docking regions and
receptor-preparation settings used for the corresponding target pockets.

Third, we compare the highest-ranked recovered candidates with within-pool
background candidates. The selected candidates remain stronger under the
alternative Vina-family scores: their mean conservative dual Vinardo
strength is 2.780, compared with $-2.944$ for the background candidates,
and their mean conservative dual GNINA-internal Vina strength is 5.206,
compared with 0.832. The selected candidates also form more balanced pocket
contacts, with a higher mean minimum contact count across the two targets
(8.600 versus 6.273). These comparisons are summarized in
Table~\ref{tab:oracle_reference_background} and
Figure~\ref{fig:oracle_background}.

\begin{table*}[t]
\centering

\small
\setlength{\tabcolsep}{3pt}
\renewcommand{\arraystretch}{1.08}

\begin{tabular*}{\textwidth}{
    @{\extracolsep{\fill}}
    l
    *{8}{c}
    @{}
}
\multicolumn{9}{@{}l}{%
    \textbf{(a) Dual-target binding performance}
}\\[2pt]
\toprule
&
\multicolumn{4}{c}{GNINA CNNaffinity $\uparrow$}
&
\multicolumn{4}{c}{Vinardo $\downarrow$}
\\
\cmidrule(lr){2-5}
\cmidrule(lr){6-9}
Method
& P-1
& P-2
& Min
& Hit (\%)
& P-1
& P-2
& Max
& Hit (\%)
\\
\midrule

\textbf{REUSE} (default; EA)
& \textbf{6.324}
& \textbf{6.136}
& \textbf{5.672}
& \textbf{41.7}
& \textbf{-6.442}
& -6.253
& \textbf{-5.876}
& \textbf{51.7}
\\

REUSE (CMA-ES)
& 6.042
& 5.913
& 5.501
& 38.3
& -6.362
& -6.089
& -5.678
& 41.7
\\

REUSE (BO)
& 6.152
& 6.002
& 5.531
& 28.3
& -6.402
& \textbf{-6.265}
& -5.820
& 48.3
\\
\bottomrule
\end{tabular*}

\vspace{0.8em}

\begin{tabular*}{\textwidth}{
    @{\extracolsep{\fill}}
    l
    *{10}{c}
    @{}
}
\multicolumn{11}{@{}l}{%
    \textbf{(b) Chemical feasibility, structural alerts, and qualified-candidate rates}
}\\[2pt]
\toprule
&
\multicolumn{2}{c}{Feasible hit rate (\%)}
&
\multicolumn{2}{c}{Alert-free hit rate (\%)}
&
\multicolumn{2}{c}{Qualified hit rate (\%)}
&
\multicolumn{4}{c}{Drug-likeness and filter pass rate (\%)}
\\
\cmidrule(lr){2-3}
\cmidrule(lr){4-5}
\cmidrule(lr){6-7}
\cmidrule(lr){8-11}
Method
& GNINA
& Vinardo
& GNINA
& Vinardo
& GNINA
& Vinardo
& Ro5
& Veber
& PAINS-free
& Brenk-free
\\
\midrule

\textbf{REUSE} (default; EA)
& \textbf{26.7}
& \textbf{25.0}
& \textbf{18.3}
& \textbf{8.3}
& \textbf{15.0}
& \textbf{8.3}
& \textbf{98.3}
& \textbf{93.3}
& \textbf{98.7}
& \textbf{30.0}
\\

REUSE (CMA-ES)
& 20.0
& 21.7
& 11.7
& 5.0
& 10.0
& 5.0
& 93.3
& 88.3
& 96.7
& 25.0
\\

REUSE (BO)
& 13.3
& \textbf{25.0}
& 5.0
& 6.7
& 3.3
& 6.7
& 91.7
& 91.7
& 96.7
& 16.7
\\
\bottomrule
\end{tabular*}

\caption{
Comparison of input-space optimizer choices within REUSE.
The row labeled \textbf{REUSE} (default; EA) corresponds to the proposed
REUSE method used throughout the main experiments, where evolutionary
algorithm (EA) search serves as the default input-space optimizer.
REUSE (CMA-ES) and REUSE (BO) replace the EA update rule with CMA-ES and
Bayesian optimization (BO), respectively. All three configurations are
initialized from the same random noise prior $\pi_0(\cdot\mid p)$ under
the matched random-seed setting. The frozen generator, search objective,
decoded-candidate budget, scoring pipeline, chemistry and diversity
controls, and downstream selection and evaluation protocols are otherwise
unchanged.
(a) Dual-target binding performance evaluated using GNINA CNNaffinity and
Vinardo. For GNINA, higher values are better, and Min denotes the
conservative dual-target score; for Vinardo, lower values are better, and
Max denotes the conservative dual-target score.
(b) Chemical-feasibility, structural-alert-free, and qualified-candidate
rates reported separately according to the GNINA- and Vinardo-based hit
definitions. Ro5, Veber, PAINS-free, and Brenk-free denote the corresponding
medicinal-chemistry and structural-filter pass rates.
Boldface numerical entries indicate the best result in each column; tied
best results are all highlighted. All evaluated outputs were valid.
}
\label{tab:optimizer_comparison}

\end{table*}

\begin{table*}[t]
\centering
\small
\setlength{\tabcolsep}{4pt}
\begin{tabular}{@{}llccc@{}}
\toprule
Check & Metric & Selected/reference & Background & Evidence \\
\midrule
T1 reference ligand & Vina / Vinardo / contacts & -10.960 / -7.331 / 13 & -- & reasonable reference binding \\
T2 reference ligand & Vina / Vinardo / contacts & -5.483 / -4.907 / 11 & -- & reasonable reference binding \\
Selected vs. background & Dual Vinardo strength & 2.780 & -2.944 & $p=0.0005$ \\
Selected vs. background & Dual GNINA-internal Vina strength & 5.206 & 0.832 & $p=0.0005$ \\
Selected vs. background & Dual minimum contacts & 8.600 & 6.273 & $p=0.0111$ \\
Selected vs. background & Total contacts & 26.400 & 22.182 & $p=0.0230$ \\
Reference overlap & Mean residue-level Jaccard & 0.763 & -- & shared pocket residues \\
\bottomrule
\end{tabular}
\caption{Reference and background sanity checks for the recovered candidate panel. Reference ligands are evaluated under the same docking setup as generated molecules. For the background comparison, high-ranked recovered candidates are compared against within-pool background candidates. Larger values are better for the dual-strength and contact-count metrics.}
\label{tab:oracle_reference_background}
\end{table*}

\begin{figure*}[t]
\centering
\includegraphics[width=0.95\textwidth]{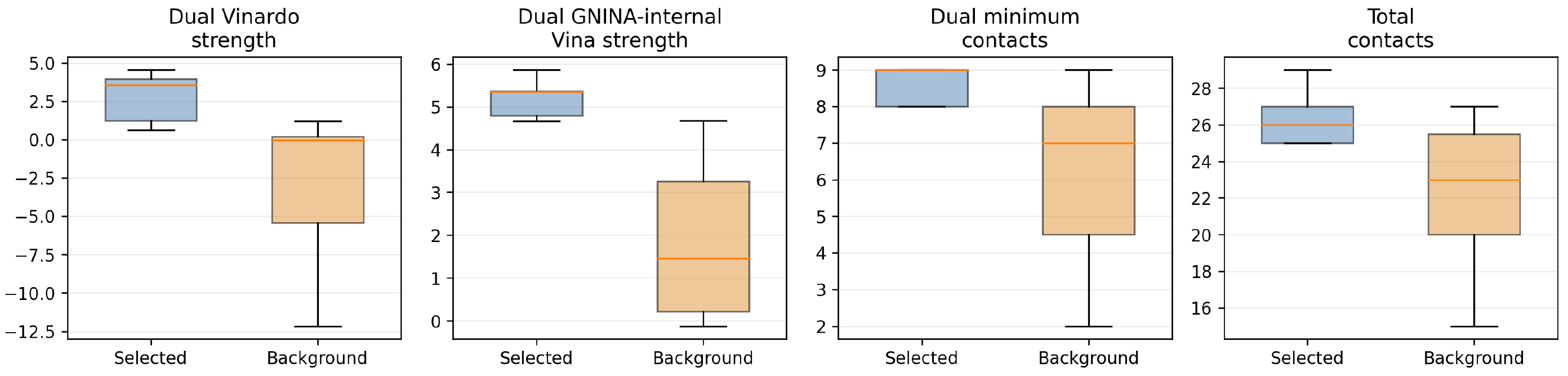}
\caption{
Selected recovered candidates compared with within-pool background molecules
under alternative Vina-family rescoring and contact-based structural
plausibility metrics.
}
\label{fig:oracle_background}
\end{figure*}

Fourth, we evaluated whether favorable scores are preserved after local pose
relaxation and whether the resulting poses exhibit severe geometric artifacts.
The selected candidates preserve favorable relaxed-pose scores, maintain pocket
contacts, and show no severe protein--ligand clashes under conservative
geometry checks. These results are summarized in
Table~\ref{tab:oracle_pose_geometry} and
Figure~\ref{fig:oracle_pose_geometry}.

\begin{table*}[t]
\centering
\small
\setlength{\tabcolsep}{4pt}
\begin{tabular}{@{}llccc@{}}
\toprule
Check & Metric & Selected & Background & Evidence \\
\midrule
Relaxed scoring & Dual relaxed Vina strength & 6.215 & 5.101 & $p=0.0343$ \\
Pose geometry & Severe protein--ligand clashes & 0.000 & 0.000 & no severe clashes \\
Pocket contacts & Dual minimum relaxed contacts & 8.600 & 6.273 & $p=0.0111$ \\
Pocket contacts & Mean contact retention & 0.863 & 0.694 & $p=0.0045$ \\
Pocket contacts & Relaxed/raw contact ratio & 1.095 & 0.894 & $p=0.0043$ \\
\bottomrule
\end{tabular}
\caption{Relaxed-pose and geometry sanity checks for the recovered candidate panel. Scores are evaluated after local Vina pose relaxation, and contact and clash statistics are computed from the resulting protein--ligand geometries. Larger values are better for relaxed strength and contact metrics; lower values are better for severe clashes.}
\label{tab:oracle_pose_geometry}
\end{table*}

\begin{figure*}[t]
\centering
\includegraphics[width=0.95\textwidth]{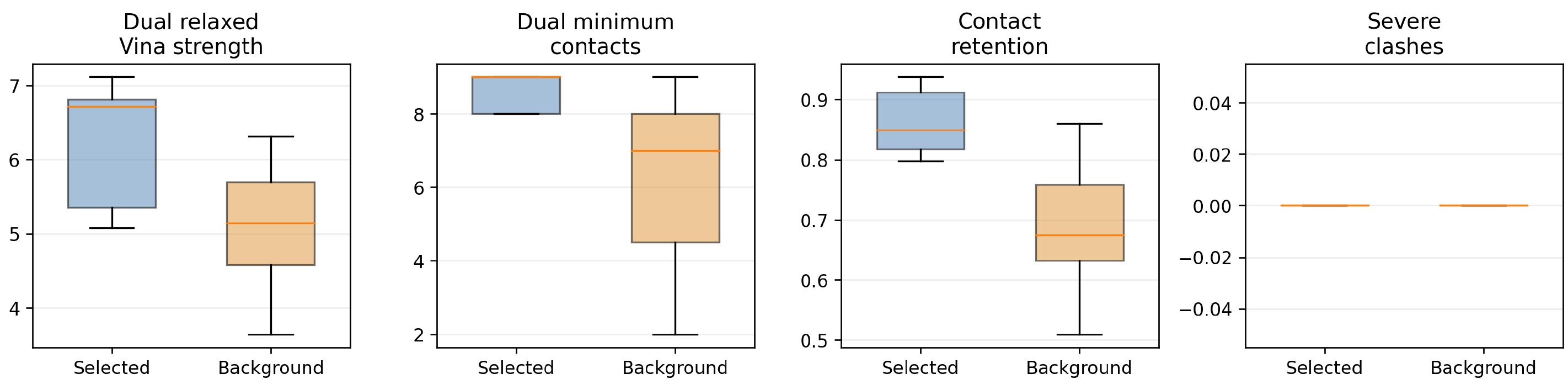}
\caption{
Relaxed-pose and geometry sanity checks for selected recovered candidates
against within-pool background molecules. The selected candidates preserve
favorable relaxed-pose scores and pocket contacts without severe
protein--ligand clashes.
}
\label{fig:oracle_pose_geometry}
\end{figure*}

Finally, we compared residue-level interaction fingerprints between the selected
candidates and the corresponding benchmark reference ligands. The selected
candidates share substantial pocket-level contact patterns with the reference
ligands, with a mean residue-level Jaccard overlap of 0.763 across the two
targets (Figure~\ref{fig:oracle_reference_overlap}). Taken together, these
checks support that the favorable scores of the recovered candidates correspond
to interpretable binding-site contacts and are not solely due to nonspecific
pocket filling. We emphasize that this analysis remains a computational
structural plausibility check rather than wet-lab validation of dual-target
activity.

\begin{figure*}[t]
\centering
\includegraphics[width=0.72\textwidth]{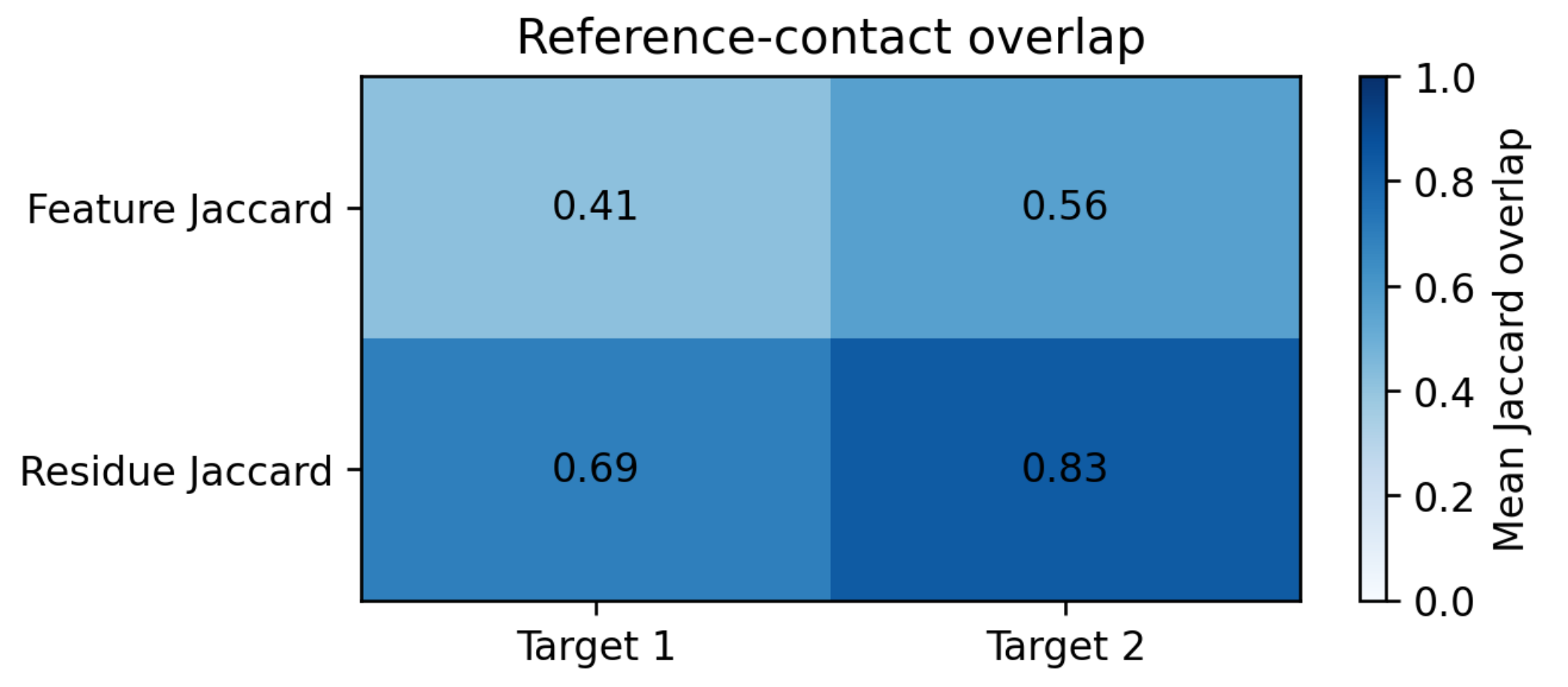}
\caption{
Pocket-interaction overlap between selected generated candidates and the
corresponding benchmark reference ligands. The selected candidates share
residue-level contact patterns with the reference ligands on both targets.
}
\label{fig:oracle_reference_overlap}
\end{figure*}

\subsection{Local Consistency in the Shared Input Space}
\label{sec:appendix_local_consistency}

\begin{figure*}[t]
    \centering
    \includegraphics[width=0.88\textwidth]{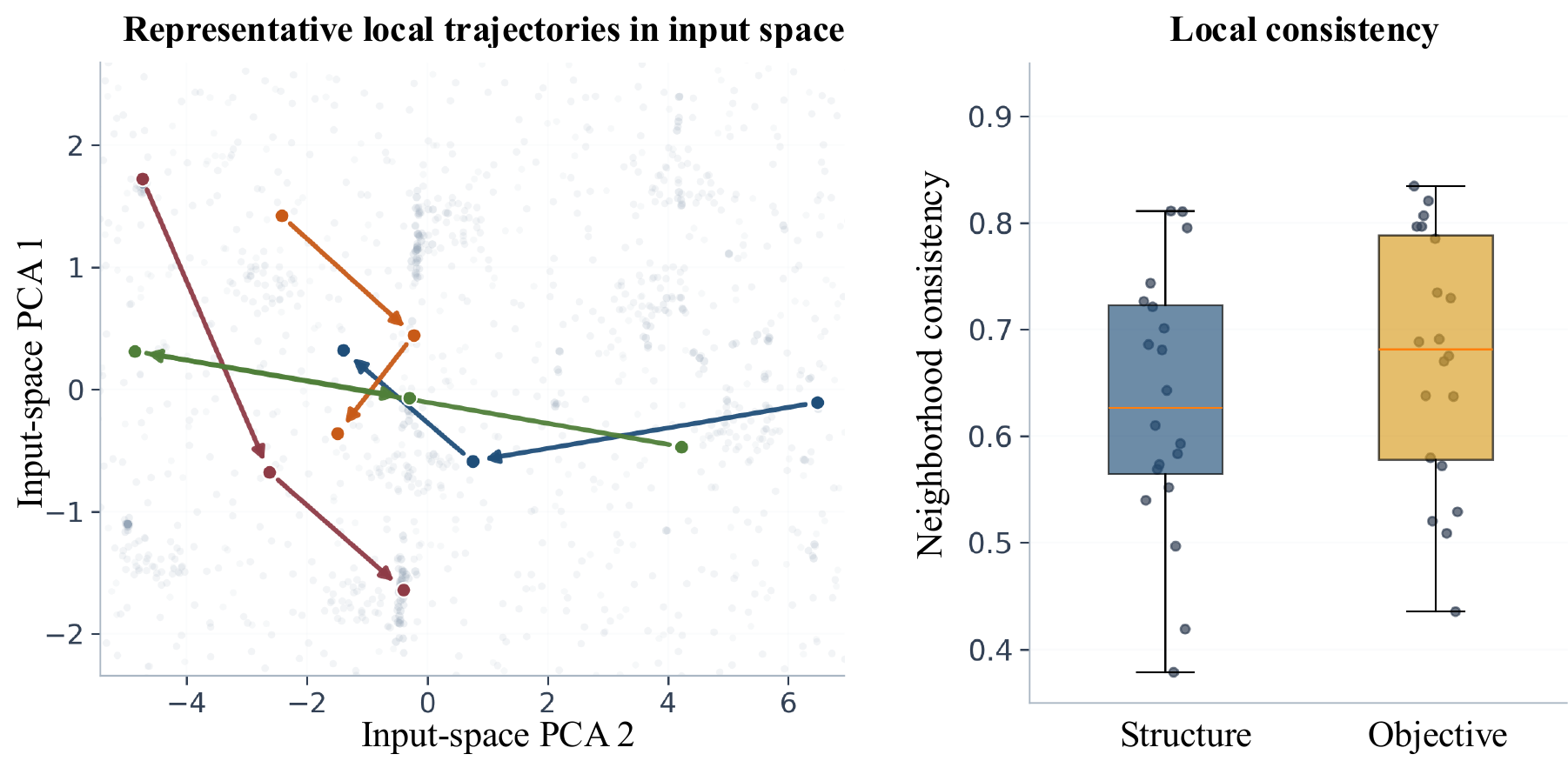}
    \caption{Local consistency in the shared input space. Left: multi-stage trajectories projected onto a shared PCA view of the pooled latent representations. Right: distributions of case-level structural consistency and objective-region consistency across local neighborhoods. Higher values indicate stronger local coherence in the frozen input space.}
    \label{fig:S1}
\end{figure*}

Figure \ref{fig:S1} visualizes local organization in the shared input space. In the left panel, we project multi-stage trajectories onto a global PCA view of the pooled latent representations. The gray points show the pooled background distribution over all candidates, while the colored arrows trace trajectory directions across stages. The plot shows how the search paths move relative to the overall candidate distribution.

The right panel quantifies this local structure. Let $x_i \in \mathbb{R}^d$ denote the latent representation of candidate $i$, and let $\mathcal{N}_k(i)$ denote its $k$ nearest neighbors in input noise space. Each candidate is associated with a structural-group label $c_i$ from the preprocessing family assignment, and with a binary objective-region label
\[
y_i = \mathbf{1}[o_i \ge \tau],
\]
where $o_i$ is the pooled objective score and $\tau$ is the fixed threshold used to define the high-objective region. For each candidate $i$, we define structural consistency at neighborhood size $k$ as
\[
S_i^{(k)} = \frac{1}{k} \sum_{j \in \mathcal{N}_k(i)} \mathbf{1}[c_j = c_i],
\]
and objective consistency as
\[
O_i^{(k)} = \frac{1}{k} \sum_{j \in \mathcal{N}_k(i)} \mathbf{1}[y_j = y_i].
\]
Thus, $S_i^{(k)}$ measures how often nearby candidates belong to the same structural group, whereas $O_i^{(k)}$ measures how often nearby candidates lie in the same objective region.

To obtain the case-level summaries shown in the boxplots, we first average these quantities over all candidates within each evaluation case, and then average over a predefined range of neighborhood sizes:
\begin{align*}
\bar S_m
&= \frac{1}{|\mathcal{K}|} \sum_{k \in \mathcal{K}}
\left(
\frac{1}{|\mathcal{I}_m|} \sum_{i \in \mathcal{I}_m} S_i^{(k)}
\right), \\
\bar O_m
&= \frac{1}{|\mathcal{K}|} \sum_{k \in \mathcal{K}}
\left(
\frac{1}{|\mathcal{I}_m|} \sum_{i \in \mathcal{I}_m} O_i^{(k)}
\right).
\end{align*}
where $\mathcal{I}_m$ denotes the set of candidates in evaluation case $m$, and $\mathcal{K}$ denotes the neighborhood sizes included in the summary. The two boxplots in the right panel show the distributions of $\bar S_m$ and $\bar O_m$ across evaluation cases. Nearby points are therefore not arranged arbitrarily: local neighborhoods tend to be coherent both structurally and in terms of objective-region membership.

Overall, Figure \ref{fig:S1} suggests that favorable regions in the shared input space are not randomly scattered. Instead, they appear as locally organized neighborhoods, which helps explain why search can benefit from local exploration in the frozen latent space.

\begin{table}[t]
\centering
\small
\setlength{\tabcolsep}{3pt}
\begin{tabularx}{\columnwidth}{@{}L{0.44\columnwidth}rrrr@{}}
\toprule
Stage
& \shortstack{Median \\ (min)}
& \shortstack{Mean \\ (min)}
& \shortstack{IQR \\ (min)}
& \shortstack{Share \\ (\%)} \\
\midrule
Search, Decoding, and Rapid Screening & 40.8 & 40.3 & 17.9 & 77.2 \\
Reduced-Frontier Full Docking & 11.3 & 11.8 & 3.4 & 22.6 \\
Final Panel Construction and Incumbent Update & 0.1 & 0.1 & 0.0 & 0.2 \\
Total & 52.2 & 52.2 & 21.3 & 100.0 \\
\bottomrule
\end{tabularx}
\caption{Stage-wise wall-clock decomposition of REUSE under the official timed evaluation setting used in the main paper. The stage split follows the paper-level operational procedure rather than implementation-specific profiler boundaries.}
\label{tab:appendix_runtime_stagewise}
\end{table}

\begin{figure*}[t]
    \centering
    \includegraphics[width=0.95\textwidth]{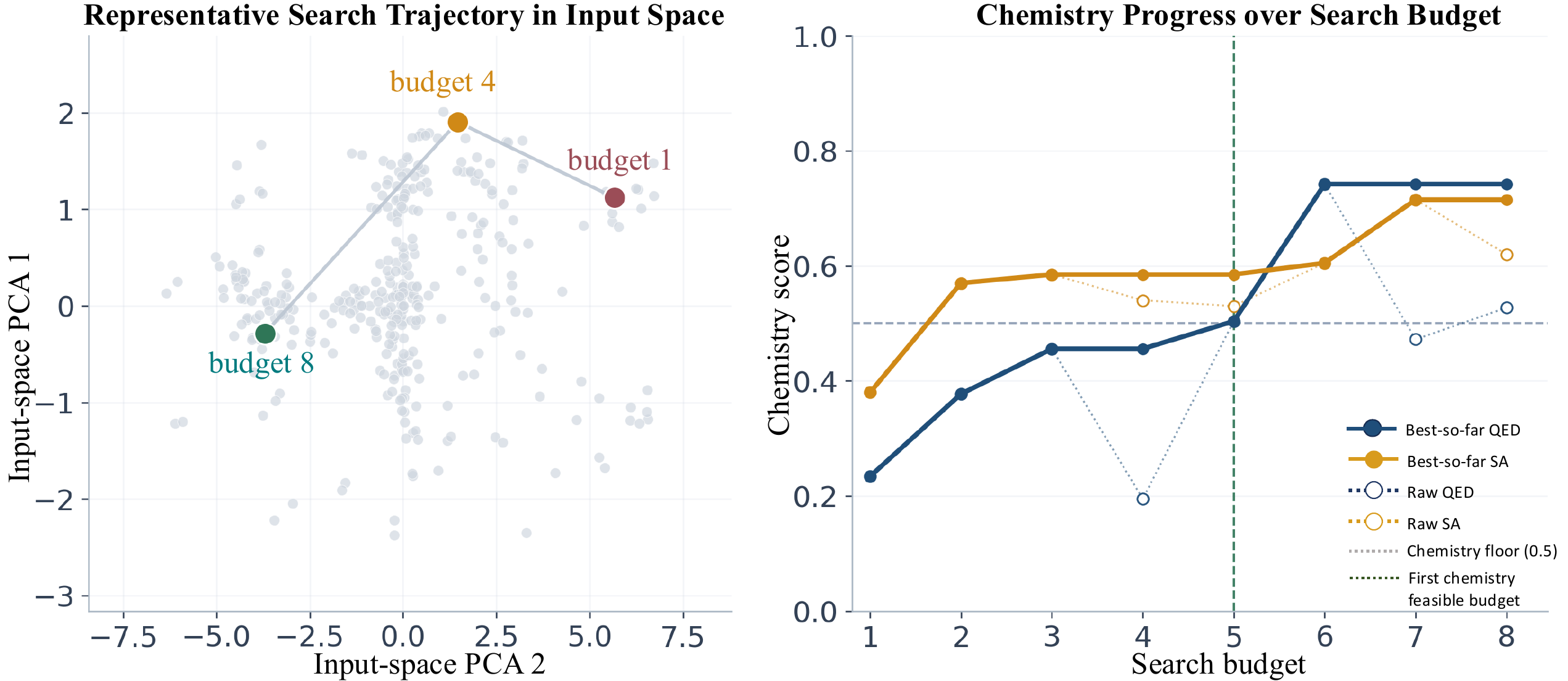}
    \caption{Search trajectory and chemistry progression in the frozen input space. Left: trajectory in the frozen input-space PCA view, with gray points showing other sampled cases and colored markers denoting milestone budgets. Right: raw and best-so-far chemistry scores (QED and SA) over search budget. The horizontal dashed line indicates the chemistry floor (0.5), and the vertical dashed line marks the first chemistry-feasible budget.}
    \label{fig:SA_QED}
\end{figure*}

\begin{table*}[t]
\centering
\small
\setlength{\tabcolsep}{3.5pt}

\begin{minipage}[t]{0.61\textwidth}
\centering
\textit{(a) Physicochemical descriptors}
\vspace{2pt}

\begin{tabularx}{\linewidth}{@{}Xrrr@{}}
\toprule
Property
& \shortstack{Mean\\$\pm$ SD}
& \shortstack{Median\\{[IQR]}}
& Range \\
\midrule

Exact MW (Da)
& 397.7 $\pm$ 67.2
& 387.2 [345.9--440.4]
& 280.1--584.2 \\

cLogP
& 2.18 $\pm$ 1.78
& 2.20 [0.95--3.45]
& $-2.32$--5.91 \\

TPSA (\AA$^2$)
& 98.8 $\pm$ 38.3
& 97.8 [76.2--118.3]
& 15.7--235.7 \\

H-bond donors
& 2.87 $\pm$ 1.47
& 3 [2--4]
& 0--9 \\

H-bond acceptors
& 6.13 $\pm$ 2.19
& 6 [5--7]
& 2--14 \\

Rotatable bonds
& 4.08 $\pm$ 2.67
& 4 [2--6]
& 0--14 \\

Fraction Csp$^3$
& 0.480 $\pm$ 0.183
& 0.455 [0.360--0.609]
& 0.067--0.923 \\

\bottomrule
\end{tabularx}
\end{minipage}
\hfill
\begin{minipage}[t]{0.36\textwidth}
\centering
\textit{(b) Medicinal-chemistry compliance}
\vspace{2pt}

\begin{tabularx}{\linewidth}{@{}Xr@{}}
\toprule
Criterion & Pass rate (\%) \\
\midrule

MW $\leq 500$ Da
& 90.0 \\

cLogP $\leq 5$
& 95.0 \\

HBD $\leq 5$
& 96.7 \\

HBA $\leq 10$
& 96.7 \\

Lipinski Ro5
& 96.7 \\

Rotatable bonds $\leq 10$
& 97.5 \\

TPSA $\leq 140$ \AA$^2$
& 87.5 \\

Veber rule
& 86.7 \\

Ro5 and Veber
& 85.8 \\

\bottomrule
\end{tabularx}
\end{minipage}

\caption{
Physicochemical profile and medicinal-chemistry rule compliance of
\textsc{REUSE} candidates pooled across  random seeds.
Continuous descriptors are reported as mean $\pm$ standard deviation,
median [interquartile range], and range.
Rule-based outcomes are reported as pooled percentages.
}
\label{tab:appendix_physchem_druglikeness}
\end{table*}

\begin{table*}[t]
\centering
\small
\setlength{\tabcolsep}{5pt}

\begin{tabularx}{\textwidth}{@{}Xrr@{}}
\toprule
Endpoint
& Median [IQR]
& Rule pass (\%) \\
\midrule

\multicolumn{3}{@{}l}{\textit{Absorption}} \\
\addlinespace[1pt]

\hspace{0.75em}Human intestinal absorption (HIA)
& 0.975 [0.749, 0.999]
& 90.8 \\

\hspace{0.75em}Oral bioavailability
& 0.591 [0.482, 0.740]
& 72.5 \\

\hspace{0.75em}Aqueous solubility, log(mol/L)
& $-2.658$ [$-3.644$, $-1.926$]
& -- \\

\hspace{0.75em}Caco-2 permeability
& $-5.694$ [$-6.352$, $-5.061$]
& -- \\

\hspace{0.75em}PAMPA permeability
& 0.584 [0.216, 0.841]
& 53.3 \\

\hspace{0.75em}P-gp inhibition
& 0.123 [0.012, 0.404]
& 81.7 \\

\addlinespace[3pt]
\multicolumn{3}{@{}l}{\textit{Distribution}} \\
\addlinespace[1pt]

\hspace{0.75em}Blood--brain barrier penetration (BBB)
& 0.507 [0.329, 0.725]
& -- \\

\hspace{0.75em}Plasma protein binding (\%)
& 59.04 [38.69, 75.73]
& -- \\

\addlinespace[3pt]
\multicolumn{3}{@{}l}{\textit{Metabolism}} \\
\addlinespace[1pt]

\hspace{0.75em}CYP1A2 inhibition
& 0.052 [0.005, 0.277]
& 85.0 \\

\hspace{0.75em}CYP2C19 inhibition
& 0.204 [0.056, 0.615]
& 70.8 \\

\hspace{0.75em}CYP2C9 inhibition
& 0.107 [0.017, 0.438]
& 81.7 \\

\hspace{0.75em}CYP2D6 inhibition
& 0.225 [0.052, 0.446]
& 78.3 \\

\hspace{0.75em}CYP3A4 inhibition
& 0.304 [0.048, 0.667]
& 65.0 \\

\addlinespace[3pt]
\multicolumn{3}{@{}l}{\textit{Excretion}} \\
\addlinespace[1pt]

\hspace{0.75em}Half-life prediction (exploratory)
& $-0.816$ [$-8.940$, 9.578]
& -- \\

\hspace{0.75em}Hepatocyte-clearance prediction (exploratory)
& $-9.103$ [$-21.24$, 13.86]
& -- \\

\hspace{0.75em}Microsomal-clearance prediction (exploratory)
& 19.52 [6.671, 35.12]
& -- \\

\addlinespace[3pt]
\multicolumn{3}{@{}l}{\textit{Toxicity}} \\
\addlinespace[1pt]

\hspace{0.75em}hERG blockade
& 0.514 [0.275, 0.816]
& 48.3 \\

\hspace{0.75em}Ames mutagenicity
& 0.380 [0.202, 0.607]
& 65.0 \\

\hspace{0.75em}Drug-induced liver injury (DILI)
& 0.396 [0.203, 0.669]
& 60.0 \\

\hspace{0.75em}ClinTox clinical toxicity
& 0.127 [0.072, 0.208]
& 100.0 \\

\bottomrule
\end{tabularx}

\caption{
Predicted ADMET profile of \textsc{REUSE} candidates pooled across both
random seeds. Predictions were obtained using the pretrained
endpoint-specific models distributed with ADMET-AI v2.0.1, with Chemprop
2.3.0 as the modeling backend. Values are reported as median
[interquartile range]. Classification endpoints are represented by
predicted probabilities, whereas regression endpoints are reported as raw
model outputs without post hoc clipping. A rule pass requires a predicted
probability $\geq 0.5$ for human intestinal absorption, oral
bioavailability, and PAMPA permeability. For P-gp inhibition, CYP1A2,
CYP2C19, CYP2C9, CYP2D6, and CYP3A4 inhibition, hERG blockade, Ames
mutagenicity, drug-induced liver injury, and ClinTox toxicity, a rule pass
requires a predicted probability $<0.5$. Blood--brain barrier penetration,
plasma protein binding, and continuous regression endpoints are reported
descriptively without pass thresholds. The nominal units specified by
ADMET-AI are hours for half-life, $\mu$L\,min$^{-1}$ per $10^{6}$ cells
for hepatocyte clearance, and $\mu$L\,min$^{-1}$\,mg$^{-1}$ for
microsomal clearance. Because the regression outputs are not constrained
to physically admissible ranges, the excretion endpoints are treated as
exploratory model predictions. ``--'' denotes an endpoint without a
classification threshold. Each rule-pass rate is computed separately;
no composite ADMET pass criterion is used. These results are computational
predictions rather than experimental measurements.
}
\label{tab:appendix_admet_profile}
\end{table*}

\subsection{Iterative Search Trajectory and Chemistry Progression in the Frozen Input Space}
\label{sec:appendix_search_trajectory}

Figure~\ref{fig:SA_QED} visualizes the search trajectory, chemistry progression, and entry into the chemistry-feasible region in the frozen input space.

The left panel places the trajectory in an input-space PCA view, with other sampled cases shown in gray for reference. The colored markers denote three milestone budgets. Although this is only a two-dimensional projection of a higher-dimensional process, the trajectory is still informative: it follows a structured path through the frozen input space, moving from an earlier region associated with weaker chemistry toward later regions containing stronger and more qualified candidates.

The right panel tracks chemistry as the search budget increases. At each budget, we report both the raw chemistry of the current candidate and the best-so-far chemistry accumulated up to that point. This distinction matters because the chemistry of individual candidates is not monotonic. Intermediate steps can temporarily weaken one aspect even when the search is improving overall. The best-so-far curves therefore give the more relevant view of progression, showing that deeper search increases the chance of recovering candidates with stronger chemistry profiles.

The chemistry-feasible transition makes this progression concrete. The horizontal dashed line marks the chemistry floor used in our analysis, and the vertical dashed line marks the first budget at which the trajectory enters this chemistry-feasible region. Additional search budget therefore does not only improve the proxy search path; it also enables the recovery of candidates that satisfy the basic chemistry criterion that earlier budgets failed to reach.

Overall, Figure~\ref{fig:SA_QED} is consistent with the aggregate results in the main text. The trajectory shows progressive movement toward better regions in the frozen input space, and the budget curve shows how that movement can translate into stronger chemistry and eventual entry into the chemistry-feasible regime.


\begin{table}[t]
\centering
\small
\begin{tabular}{@{}lr@{}}
\toprule
Quantity & Value \\
\midrule
Decoded candidate budget & 72 \\
Reduced frontier & 32 \\
Fully evaluated set & 23 \\
Reported output panel & 10 \\
Total wall-clock time (min) & 52.2 \\
Full-docking time (min) & 11.8 \\
\bottomrule
\end{tabular}
\caption{Runtime and intermediate set sizes under the official timed evaluation setting used in the main paper. This table complements Table~\ref{tab:appendix_runtime_stagewise} by reporting the candidate pool sizes at different stages of the pipeline.}
\label{tab:appendix_runtime_sizes}
\end{table}

\begin{table*}[t]
\centering
\small
\setlength{\tabcolsep}{4pt}
\begin{tabular}{@{}l l c c c c c c c@{}}
\toprule
Row & Shared motif & Atoms & Rings & Min sim & Avg sim & Max sim & Murcko scaf. & Motif frac. \\
\midrule
A & Phenyl ring            & 6 & 1 & 0.076 & 0.084 & 0.098 & 3 & 0.214--0.222 \\
B & Pyridyl ring           & 6 & 1 & 0.106 & 0.117 & 0.126 & 3 & 0.194--0.211 \\
C & Small unsaturated ring & 5 & 1 & 0.064 & 0.082 & 0.103 & 3 & 0.200--0.208 \\
D & Conjugated chain motif & 5 & 0 & 0.076 & 0.082 & 0.092 & 3 & 0.185--0.192 \\
E & Aminoalkyl chain motif & 7 & 0 & 0.070 & 0.083 & 0.108 & 3 & 0.280--0.292 \\
F & Aryl-linked N motif    & 6 & 0 & 0.074 & 0.081 & 0.090 & 3 & 0.214--0.222 \\
\bottomrule
\end{tabular}
\caption{Row-level structural summary for the six illustrative molecule trios shown in Figure~\ref{fig:appendix_small_shared_motif}. Whole-molecule similarity is measured by Morgan fingerprint Tanimoto similarity \cite{rogers2010ecfp}. `Murcko scaf.' counts the number of distinct Bemis--Murcko scaffolds among the three molecules in each row \cite{bemis1996frameworks}. `Motif frac.' denotes the range of shared-motif atoms divided by heavy-atom count across the three molecules in the row.}
\label{tab:appendix_motif_row_summary}
\end{table*}

\begin{table*}[t]
\centering
\small
\setlength{\tabcolsep}{4pt}
\begin{tabularx}{\textwidth}{
@{}
c
c
>{\raggedright\arraybackslash}X
S[table-format=1.2,round-mode=places,round-precision=2]
S[table-format=1.2,round-mode=places,round-precision=2]
@{}
}
\toprule
Row & Mol. & SMILES & {SA} & {QED} \\
\midrule
A & 1 & \path{N=C1CC(C(=O)C2OC(c3ccccc3)CNCCC(=O)OCC2O)CC1=N} & 0.570 & 0.573 \\
A & 2 & \path{Cc1cc(N)c2c(c1)N(CC1=POONC1=N)CCC2C1=C(F)CC=C(F)C1} & 0.540 & 0.398 \\
A & 3 & \path{CON(c1ccccc1)N1Nc2ccccc2N2NCN2C2C=C[C]=CN21} & 0.580 & 0.807 \\
B & 1 & \path{N=c1[nH]c2c(c(=O)o1)OCCC1N(c3ccccc3)C(C3NOc4ccncc43)CN1} & 0.590 & 0.669 \\
B & 2 & \path{O=C1ONCC2(c3ccncc3)NN3CN4CCN2N1C43} & 0.500 & 0.511 \\
B & 3 & \path{CC(=O)NC1=CC(c2ccncc2)CC(CS(=O)N2C(CO)CCC2CC2CCOC2)N1} & 0.520 & 0.528 \\
C & 1 & \path{CC1CNC(C2=CCCC2)CC2CC(Cl)C(CN3CCCC3(C)C)C21} & 0.460 & 0.704 \\
C & 2 & \path{OC1N(CC2=CCCC2)C(F)=NC2CCCCC2NC1=S} & 0.510 & 0.566 \\
C & 3 & \path{FC1=CC=C(C2Cc3ccccc3C3CCN4C=CC=C4N23)C(F)=C1} & 0.560 & 0.737 \\
D & 1 & \path{OC1CN2C=CC3=C2C2=C(NC=C2)C2N=CC=C2OC13} & 0.490 & 0.785 \\
D & 2 & \path{CC(=O)C1=CC(=O)NCC=C(C(=O)O)SCCO=N1} & 0.560 & 0.326 \\
D & 3 & \path{NCC1=CC(F)=C(CO)C2OC(c3ccccc3)NC12} & 0.510 & 0.499 \\
E & 1 & \path{CC1(F)CC(=O)N2C(CCNC(N)=O)SC(CNCC3=CC(F)=CC(F)=C3)C2=N1} & 0.600 & 0.776 \\
E & 2 & \path{CC1(F)COC(CCCCN2C(=O)NCC2C2CCC3C(C2)SC3N)=N1} & 0.490 & 0.290 \\
E & 3 & \path{S=C1NC(c2ccccc2)C2CCN(CCc3ccncc3)C2C1CCNC1CCNC=C1} & 0.600 & 0.830 \\
F & 1 & \path{O[S](F)(F)=O.C1=CC=C2NC(C3CCNNC3c3ncccc3)CCC2=C1} & 0.440 & 0.617 \\
F & 2 & \path{Cc1cc2c(c1)c1ccccc1N(CC1=POONC1=N)CCC2} & 0.550 & 0.558 \\
F & 3 & \path{Nc1cccc2c1N1NC(c3ccccc3)NN(C(=O)OCc3ccccc3)C1CC2} & 0.600 & 0.675 \\
\bottomrule
\end{tabularx}
\caption{Molecule-level details for the 18 molecules shown in Figure~\ref{fig:appendix_small_shared_motif}. Molecules are listed in the same row-wise order as in the figure.}
\label{tab:appendix_motif_molecule_details}
\end{table*}

\begin{table*}[t]
\centering
\small
\setlength{\tabcolsep}{3pt}
\captionsetup{width=\textwidth}

\begin{tabularx}{\textwidth}{
@{}
>{\raggedright\arraybackslash}X
r
@{}
}
\toprule
Metric & Pooled result \\
\midrule

\multicolumn{2}{@{}l}{\textit{Molecular and scaffold uniqueness}} \\
\addlinespace[1pt]

\hspace{0.75em}Canonical-molecule uniqueness
& 100\% \\

\hspace{0.75em}Bemis--Murcko scaffold uniqueness
& 100\% \\

\addlinespace[3pt]
\multicolumn{2}{@{}l}{\textit{Novelty relative to benchmark references}} \\
\addlinespace[1pt]

\hspace{0.75em}Target-1 reference similarity
& $0.0809 \pm 0.0310$ \\

\hspace{0.75em}Target-2 reference similarity
& $0.0833 \pm 0.0265$ \\

\hspace{0.75em}Maximum pair-reference similarity
& $0.0962 \pm 0.0242$ \\

\hspace{0.75em}Share with pair-reference similarity $<0.20$
& 100\% \\

\hspace{0.75em}Maximum benchmark-reference similarity
& $0.1782 \pm 0.0290$ \\

\hspace{0.75em}Share with benchmark-reference similarity $<0.20$
& 79.2\% \\

\hspace{0.75em}Share with benchmark-reference similarity $<0.30$
& 100\% \\

\hspace{0.75em}Novelty distance to benchmark references
& $0.8218 \pm 0.0290$ \\

\hspace{0.75em}Exact pair-reference match rate
& 0\% \\

\hspace{0.75em}Exact benchmark-reference match rate
& 0\% \\

\hspace{0.75em}Pair-reference scaffold novelty rate
& 100\% \\

\hspace{0.75em}Benchmark-reference scaffold novelty rate
& 99.2\% \\

\addlinespace[3pt]
\multicolumn{2}{@{}l}{\textit{Intra-panel structural diversity}} \\
\addlinespace[1pt]

\hspace{0.75em}Within-panel Morgan diversity
& $0.8978 \pm 0.0158$ \\

\hspace{0.75em}Within-panel maximum Tanimoto
& $0.1219 \pm 0.0247$ \\

\hspace{0.75em}Distinct-scaffold panel rate
& 100\% \\

\bottomrule
\end{tabularx}

\caption{
Benchmark-reference novelty and intra-panel structural diversity of
\textsc{REUSE} candidates pooled across  random seeds.
Molecular similarity is measured by Tanimoto similarity using radius-2,
2048-bit Morgan fingerprints \cite{rogers2010ecfp}, while scaffold novelty
is evaluated using Bemis--Murcko scaffolds \cite{bemis1996frameworks}.
``Pair reference'' denotes the two benchmark-associated reference ligands
for the corresponding target pair, whereas ``benchmark references'' denotes
the complete benchmark reference set.
Within-panel quantities are computed among candidates belonging to the same
selected output panel.
Continuous values are reported as mean $\pm$ standard deviation;
uniqueness, threshold, exact-match, and scaffold-novelty statistics are
reported as percentages.
}
\label{tab:appendix_reference_novelty_diversity}
\end{table*}

\subsection{Comparison of Input-Space Optimizers}
\label{sec:appendix_optimizer_comparison}

Table~\ref{tab:optimizer_comparison} compares three derivative-free
input-space optimizers under a controlled \textsc{REUSE} protocol. All
variants use the same target-pair inputs, random-seed setting, frozen
generator, input representation, random noise prior
$\pi_0(\cdot\mid p)$, search objective, docking and scoring protocols,
chemistry and diversity controls, and downstream selection procedure. The
variants differ only in how the input-space optimizer updates and proposes
new points after initialization. For each target pair, EA, CMA-ES, and BO
are each limited to the same budget of 72 decoded candidate molecules. The
budget is matched by the number of evaluated input points and stochastic
molecular decodes rather than by optimizer iterations, since the three
optimizers follow different update schedules.

\paragraph{Optimizer-specific configurations.}
All three optimizers were initialized from the same random noise prior
$\pi_0(\cdot\mid p)$ under the matched random-seed setting. No
pair-specific warm start or optimizer-specific initialization information
was used.

CMA-ES operated in the 10-dimensional input space with population size
$\lambda=4$, parent number $\mu=2$, three generations, and initial step
size $\sigma_0=0.42$. Its initial population was sampled from
$\pi_0(\cdot\mid p)$, and its initial mean was set to the arithmetic
centroid of these samples. Candidate points were restricted to
$[-8,8]^{10}$. After each generation, the mean, covariance matrix, and
global step size were updated from the two highest-scoring individuals
using a rank-$\mu$ CMA-ES update.

Bayesian optimization used four initial observations independently sampled
from $\pi_0(\cdot\mid p)$, followed by eight sequential proposals. Its
surrogate was a Gaussian process with a fixed Matérn-$5/2$ kernel of unit
length scale and a fixed white-noise term of $0.05$. The response values
were normalized, with numerical jitter $10^{-8}$. Each proposal maximized
expected improvement with $\xi=0.01$ over 4,096 scrambled Sobol candidates
in the bounded domain $[-8,8]^{10}$.

All three optimizers evaluated 12 input points under the matched budget,
with each point decoded six times, resulting in
$12\times6=72$ decoded candidates per target pair. Thus, the comparison
matches the initialization information, number of evaluated input points,
total number of stochastic molecular decodes, and downstream evaluation
pipeline.

For dual-target binding, \textsc{REUSE}-EA achieves the highest P-1, P-2,
and conservative minimum GNINA CNNaffinity scores, together with the highest
GNINA hit rate. It also obtains the best P-1 Vinardo score, conservative
maximum Vinardo score, and Vinardo hit rate. \textsc{REUSE}-BO yields a
slightly better P-2 Vinardo score than \textsc{REUSE}-EA
($-6.265$ versus $-6.253$), but this local improvement does not translate
into a better conservative dual-target score ($-5.820$ versus $-5.876$) or
hit rate (48.3\% versus 51.7\%).

\textsc{REUSE}-CMA-ES provides a competitive but less consistent alternative.
It achieves a higher GNINA hit rate than \textsc{REUSE}-BO
(38.3\% versus 28.3\%), but trails \textsc{REUSE}-EA across all reported
binding metrics. Compared with \textsc{REUSE}-BO, CMA-ES also gives weaker
continuous GNINA and Vinardo scores and a lower Vinardo hit rate. Thus, the
advantage of the default configuration cannot be attributed merely to the
use of an arbitrary population-based derivative-free optimizer.

The chemical-feasibility results exhibit a similar pattern.
\textsc{REUSE}-EA is best or tied best in every column of
Table~\ref{tab:optimizer_comparison}(b). Under the GNINA-based hit
definition, CMA-ES consistently outperforms BO in feasible, alert-free, and
qualified hit rates (20.0\%, 11.7\%, and 10.0\%, respectively, compared
with 13.3\%, 5.0\%, and 3.3\%). By contrast, BO performs better under the
corresponding Vinardo-based definitions, reaching 25.0\%, 6.7\%, and 6.7\%,
compared with 21.7\%, 5.0\%, and 5.0\% for CMA-ES. CMA-ES also exceeds BO
in the Ro5 and Brenk-free pass rates, whereas BO obtains a higher Veber pass
rate; the two configurations have the same PAINS-free rate.

The separate post hoc analysis in Figure~\ref{fig:cross_scorer_radar} supports the same optimizer-level conclusion: the default EA exceeds CMA-ES and BO on seven of eight scorer--target combinations in each comparison, while the isolated exceptions occur under different scorers.

Across the binding and chemical-quality evaluations,
\textsc{REUSE}-EA is best or tied best in 17 of the 18 reported columns.
These results indicate that EA provides the most consistently balanced
performance among the tested optimizers under the matched pipeline and
evaluation budget, supporting its use as the default input-space optimizer
in \textsc{REUSE}. This comparison is specific to the evaluated setting and
is not intended to establish the universal superiority of EA over other
derivative-free optimization methods.

\subsection{Extended Characterization of Candidate Pools and Final Panels}
\label{sec:appendix_extended_characterization}

\paragraph{Criterion attainment.}
Figure~\ref{fig:criterion} characterizes three complementary properties of
the pre-selection candidate pool: dual-target binding, joint QED/SA quality,
and compliance with the Lipinski and Veber guidelines. Ro5/Veber compliance
is relatively common, with 85.8\% of candidates satisfying both guidelines.
The dual-target GNINA and joint QED/SA criteria are more selective, with
attainment rates of 45.8\% and 38.3\%, respectively, while 22.5\% of
candidates satisfy all three criteria simultaneously. The outcome profiles
further show that joint failure of the dual-target and QED/SA criteria is
the most frequent shortfall. Together, these results characterize the
candidate landscape entering final panel construction.

\paragraph{Physicochemical profile and rule-based compliance.}
Figure~\ref{fig:physchem} and
Table~\ref{tab:appendix_physchem_druglikeness} provide complementary
candidate-level and pooled summaries of the physicochemical profile.
The pooled medians are 387.2 Da for molecular weight, 2.20 for cLogP,
97.8~\AA$^2$ for TPSA, three hydrogen-bond donors, six hydrogen-bond
acceptors, and four rotatable bonds. Lipinski compliance reaches 96.7\%,
while 86.7\% satisfy the Veber rule and 85.8\% satisfy both.

The distributions are centered largely within conventional early-stage
medicinal-chemistry ranges, and the two random seeds show broadly similar
profiles. At the same time, the distribution tails include candidates with
higher molecular weight, polarity, lipophilicity, or flexibility. These
results characterize the chemical space represented by the final panels
rather than implying uniform downstream developability.

\paragraph{Predicted ADMET profile.}
Table~\ref{tab:appendix_admet_profile} extends the physicochemical
characterization with computational predictions from ADMET-AI v2.0.1,
using its pretrained endpoint-specific Chemprop 2.3.0 models. Classification
outputs are summarized as predicted probabilities and prespecified rule-pass
rates. The continuous excretion endpoints are retained as exploratory raw regression outputs and are not used to define candidate pass or failure.

Pass rates are relatively high for predicted human intestinal absorption
and ClinTox, at 90.8\% and 100.0\%, respectively. Oral bioavailability
reaches 72.5\%, while the CYP-inhibition criteria range from 65.0\% to
85.0\%. PAMPA permeability and hERG blockade are more selective, with pass
rates of 53.3\% and 48.3\%, respectively. Overall, the variation across endpoints identifies both favorable predicted properties and potential liabilities that can inform subsequent candidate prioritization.

\paragraph{Structural-alert landscape.}

Figure~\ref{fig:alert} characterizes structural alerts by motif and target
pair. The PAINS analysis is particularly informative for candidate
prioritization because PAINS motifs are widely used as early filters for
compounds susceptible to nonspecific assay interference and false-positive
readouts. Overall, 95.8\% of the candidates are PAINS-free. The two observed PAINS
motifs occur only at low overall frequencies (2.5\% and 1.7\%). This sparse PAINS profile
indicates that the candidate panels are not dominated by canonical nuisance
substructures, which is favorable for subsequent compound review, synthesis
prioritization, and experimental follow-up.

The complementary Brenk screen shows that structural alerts are
heterogeneously distributed across target pairs rather than being driven by
any single pair. These alerts therefore provide compound-level flags for
further review and prioritization, but should not be interpreted as
experimental evidence of toxicity or assay interference.

\paragraph{Benchmark-reference novelty and intra-panel diversity.}
Figure~\ref{fig:novelty} and
Table~\ref{tab:appendix_reference_novelty_diversity} evaluate structural
novelty at two complementary reference scopes, together with diversity
within each selected panel. The pair-specific comparison measures whether
the search primarily recovers close analogs of the two reference molecules
associated with the current target pair. At this scope, the mean maximum
similarity is only $0.0962 \pm 0.0242$, and every candidate has a maximum
similarity strictly below 0.20. The individual similarities to the two
target-specific references are similarly low, with means of
$0.0809 \pm 0.0310$ and $0.0833 \pm 0.0265$, respectively.

The all-reference comparison is more extensive because each candidate is
matched against the complete set of 438 benchmark references. As expected,
expanding the reference pool increases the probability of finding a closer
structural neighbor, raising the mean maximum similarity to
$0.1782 \pm 0.0290$. Nevertheless, 79.2\% of candidates remain strictly
below 0.20, and every candidate remains below 0.30, corresponding to a mean
novelty distance of $0.8218 \pm 0.0290$. The cumulative distributions in
Figure~\ref{fig:novelty}(a) show that the pair-specific similarities are
consistently concentrated at lower values. Figure~\ref{fig:novelty}(b)
further compares the two reference scopes candidate by candidate: the
all-reference maximum is necessarily at least as large as the pair-specific
maximum because the complete benchmark set contains the pair-specific
references, but the resulting similarities remain low overall.

The same pattern is observed at the level of molecular identity and core
scaffolds. All canonical molecules and all candidate Bemis--Murcko
scaffolds are unique within the evaluated set, and no candidate exactly
matches either a pair-specific reference or any molecule in the complete
benchmark reference set. All candidate scaffolds are novel relative to
their pair-specific references, and 99.2\% remain scaffold-novel relative
to the complete reference set. Thus, the observed differences extend
beyond molecular-string uniqueness and, in nearly all cases, also involve
distinct core structural frameworks.

Global uniqueness alone does not guarantee useful diversity within the
small panel selected for an individual target pair. We therefore additionally
evaluate the composition of each final panel. The panels retain high
within-panel Morgan diversity ($0.8978 \pm 0.0158$) and low maximum
pairwise Tanimoto similarity ($0.1219 \pm 0.0247$), and every panel contains
N distinct Bemis--Murcko scaffolds. These results indicate that the
selection procedure does not concentrate each panel around a single local
chemotype, but instead returns multiple structurally differentiated
alternatives for each target pair.

Taken together, the fingerprint-, identity-, and scaffold-level results show
that the recovered candidates are structurally distinct from the benchmark
reference molecules while retaining substantial diversity within each
selected panel.

\subsection{Cheap-to-Full Evaluator Frontier Preservation}
\label{sec:appendix_frontier_filtering}

The cheap evaluator is intended to preserve promising candidates for
subsequent full evaluation rather than reproduce the exact ranking of the
entire candidate pool. Table~\ref{tab:appendix_evaluator_stats} evaluates
this screening behavior through frontier overlap, global score agreement,
and computational cost.

\begin{table}[t]
\centering
\small
\renewcommand{\arraystretch}{1.05}
\setlength{\tabcolsep}{8pt}
\begin{tabular}{@{}lcc@{}}
\toprule
Group & Metric & Value \\
\midrule
\multirow{2}{*}{Preservation}
& Ov@5 $\uparrow$ & 0.887 \\
& Ov@8 $\uparrow$ & 0.992 \\
\midrule
\multirow{3}{*}{Agreement}
& Spearman $\uparrow$ & 0.356 \\
& Pearson $\uparrow$ & 0.428 \\
& Kendall $\uparrow$ & 0.302 \\
\midrule
\multirow{3}{*}{Efficiency}
& Cheap s/mol $\downarrow$ & 11.09 \\
& Full s/mol $\downarrow$ & 15.93 \\
& Runtime ratio $\uparrow$ & 2.17 \\
\bottomrule
\end{tabular}
\caption{
Stage-1 frontier preservation, global score agreement, and computational
cost of the cheap evaluator relative to full evaluation. Ov@$k$ denotes
the overlap between their top-$k$ frontiers. Correlations are computed over
the complete candidate pool. Cheap and Full s/mol denote per-molecule
evaluation costs, whereas Runtime ratio denotes the end-to-end full/cheap
runtime ratio.
}
\label{tab:appendix_evaluator_stats}
\end{table}

The cheap evaluator retains 88.7\% of the full-evaluator top-5 frontier and
99.2\% of the top-8 frontier. Although its global rank and score correlations
with full evaluation are moderate, the high frontier overlap shows that the
candidates most relevant to downstream selection are largely preserved.
The per-molecule cost decreases from 15.93 to 11.09 seconds, with an
end-to-end full/cheap runtime ratio of 2.17. These results support using the
cheap evaluator as a high-recall screening stage before full evaluation of the reduced frontier.

\begin{figure*}[t]
    \centering
    \includegraphics[width=0.58\textwidth]
    {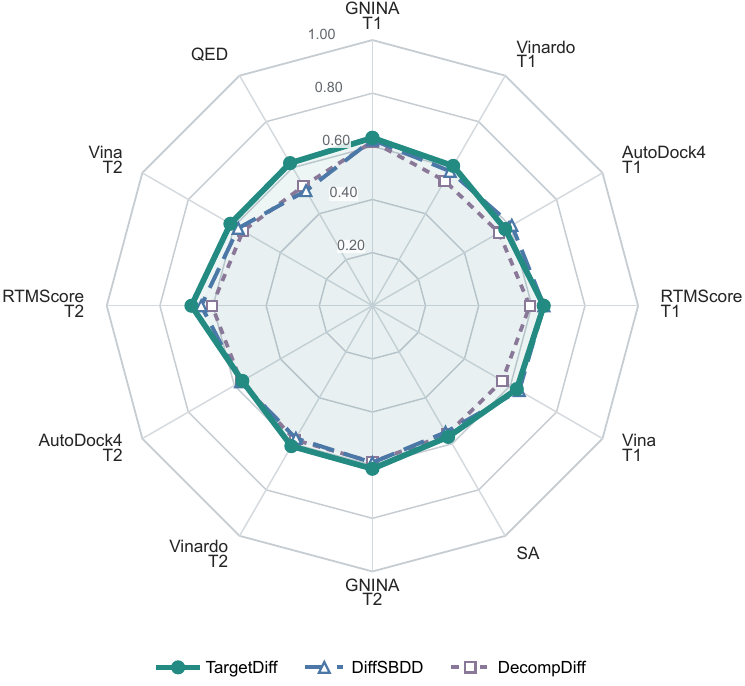}
    \caption{
    Cross-scorer backbone-transfer profiles of \textsc{REUSE} with
    TargetDiff, DiffSBDD, and DecompDiff. T1 and T2 results for each
    affinity metric are placed on opposite spokes, with QED and SA forming
    the remaining pair. All axes are oriented such that farther outward
    indicates better performance. Scores are mapped using fixed metric-wise
    linear transformations (GNINA$/10$, $-\mathrm{Vinardo}/10$,
    $-\mathrm{AutoDock4}/15$, RTMScore$/20$, and $-\mathrm{Vina}/15$),
    while QED and SA are unchanged. A shared zero-based radial range of
    0--1 is used without per-axis min--max normalization. The broadly
    similar profiles indicate that \textsc{REUSE} transfers across frozen
    backbones, while TargetDiff provides the most consistently strong
    overall profile.
    }
    \label{fig:backbone_radar}
\end{figure*}

\subsection{Additional Pose-Level Comparisons}
\label{sec:appendix_more_examples}
Figure~\ref{fig:appendix_binding_visualization} gives three additional qualitative comparisons on representative dual-target pairs. For each pair, we compare molecules from REUSE, DualDiff, and MDRL, and visualize their poses on both targets under a matched viewing angle for each target.

The same pattern appears across all three pairs. REUSE more often achieves favorable docking scores on both targets at the same time, which is the key requirement in the dual-target setting. Its poses are also more often centered within the pocket and better situated in the binding region, whereas the baseline molecules more frequently appear shifted, partially buried, or less stably aligned with the pocket geometry. This contrast is especially clear in the first and third target pairs.

These examples are useful because they show how the quantitative advantage appears at the pose level. Even when the gap on one target is only moderate, the REUSE pose still tends to look more coherent across the pair, consistent with better cross-target compatibility rather than one-sided improvement. Overall, Figure~\ref{fig:appendix_binding_visualization} supports the main quantitative results by showing that REUSE more reliably produces favorable and visually plausible binding configurations on both targets.

\subsection{Context-flexible recurrence of compact local motifs within a target pair}
\label{sec:appendix_local_motif_recurrence}

Figure~\ref{fig:appendix_small_shared_motif} highlights a qualitative
pattern within intermediate candidate pools produced during REUSE search.
In each row, the three molecules are distinct decoded candidates for the
same target-pair design problem, yet they share one compact exact motif,
shown in blue, while differing visibly in their larger architectures, ring
organizations, and substituent layouts. The central observation is local
reuse despite global diversity: for a fixed target pair, the same small
motif can recur across otherwise different molecules.

The six rows are representative candidate-pool examples for structural
inspection. The recurring motifs include both ring-based and non-ring
fragments, so the effect is not tied to a single motif class. The same
compact motif remains recognizable even when the surrounding molecular
context changes substantially. This pattern is more consistent with local
motif reuse within a target-pair candidate pool than with repetition of a
single whole-molecule template.

Table~\ref{tab:appendix_motif_row_summary} supports the same reading quantitatively. Across all six rows, whole-molecule Morgan fingerprint similarity remains low, each row contains three distinct Bemis--Murcko scaffolds, and the shared motif occupies only a limited fraction of heavy atoms in the full molecules. In other words, what recurs is a small local fragment, while the larger molecular frameworks remain diverse.

Table~\ref{tab:appendix_motif_molecule_details} records the 18 displayed structures explicitly, including canonical SMILES together with SA and QED values. This table mainly serves as a molecule-level record of the examples shown in the figure.

Overall, this analysis points to a target-pair-specific pattern in which our method recovers multiple globally distinct candidate molecules that nevertheless reuse the same compact local motif. A natural interpretation is that the search process identifies local chemical preferences that can be realized in several different larger molecular contexts for the same target pair.

\begin{table}[t]
\centering
\small
\setlength{\tabcolsep}{3pt}
\begin{tabular}{@{}rcccc@{}}
\toprule
\multirow{2}{*}{Budget} & \multicolumn{3}{c}{Search score} & \multirow{2}{*}{\shortstack{Chem-floor \\ recovery}} \\
\cmidrule(lr){2-4}
& Mean & P25 & P75 & \\
\midrule
1 & 0.437 & 0.193 & 0.750 & 0.056 \\
2 & 0.519 & 0.261 & 0.873 & 0.121 \\
3 & 0.638 & 0.321 & 1.012 & 0.226 \\
4 & 0.703 & 0.591 & 1.139 & 0.221 \\
5 & 0.731 & 0.600 & 1.157 & 0.337 \\
6 & 0.801 & 0.643 & 1.245 & 0.447 \\
7 & 0.844 & 0.653 & 1.245 & 0.451 \\
8 & 0.869 & 0.643 & 1.256 & 0.500 \\
\bottomrule
\end{tabular}
\caption{Budget sensitivity for Full REUSE. We report the best-so-far
search score and the chemistry-floor attainment rate at each search budget.}
\label{tab:budget_sensitivity}
\end{table}

\subsection{Budget Sensitivity}
\label{sec:appendix_budget_sensitivity}

Table~\ref{tab:budget_sensitivity} shows how REUSE progresses as the search budget increases from 1 to 8. The trend is clear in both summaries: The best-so-far mean search score rises from 0.437 at budget 1 to
0.869 at budget 8, while the quartiles generally shift upward. The
chemistry-floor attainment rate also generally increases from 0.056 to
0.500, with a small fluctuation at budget 4. This means that larger budgets more reliably surface stronger candidates under the internal search score and more often reach chemically acceptable regions, consistent with the search-progression evidence in the main text.

\begin{table}[t]
\centering
\scriptsize
\setlength{\tabcolsep}{2.5pt}
\renewcommand{\arraystretch}{1.08}

\begin{tabular*}{\columnwidth}{
    @{\extracolsep{\fill}}
    l
    *{6}{r}
    @{}
}
\toprule
\multirow{2}{*}{Metric}
& \multicolumn{3}{c}{Pre-refinement}
& \multicolumn{3}{c}{Final selected}
\\
\cmidrule(lr){2-4}
\cmidrule(lr){5-7}
& Mean & P25 & P75
& Mean & P25 & P75
\\
\midrule

Consensus Hit (\%) $\uparrow$
& 41.8 & 0.0 & 86.5
& 50.1 & 0.0 & 100.0
\\

Dual High Affinity (\%) $\uparrow$
& 50.1 & 19.8 & 91.7
& 58.3 & 33.3 & 100.0
\\

Max Vina Dock $\downarrow$
& -8.270 & -8.987 & -7.801
& -8.635 & -9.266 & -8.087
\\

\bottomrule
\end{tabular*}

\caption{
Effect of final selection under the main \textsc{REUSE} configuration.
Mean, P25, and P75 statistics are compared between the pre-refinement
candidate set and the final selected set.
}
\label{tab:final_selection_wo_balance_detailed}
\end{table}

\subsection{Effect of Final Selection}
\label{sec:appendix_final_selection}

Table~\ref{tab:final_selection_wo_balance_detailed} compares the candidate
set immediately before the final refinement step with the reported output
set. Final selection increases the mean Consensus Hit rate from 41.8\% to
50.1\% and the mean Dual High Affinity rate from 50.1\% to 58.3\%, while
improving Max Vina Dock from $-8.270$ to $-8.635$. The corresponding
quartiles generally shift in the same direction, showing that final
selection concentrates candidates with stronger and more consistently
balanced dual-target profiles.

\begin{table}[t]
\centering
\small
\setlength{\tabcolsep}{3pt}
\begin{tabular}{@{}lccc@{}}
\toprule
Method & \shortstack{Wall-clock \\ time} & \shortstack{Combined \\ Dock $\downarrow$} & \shortstack{Relative \\ cost} \\
\midrule
Ours & 52.2 min & -18.678 & 1.90$\times$ \\
DualDiff & 27.5 min & -16.806 & 1.00$\times$ \\
MDRL & 12.4 h & -16.828 & 27.05$\times$ \\
\bottomrule
\end{tabular}
\caption{Numeric summary of the time--performance trade-off in Figure 4(a) of the main paper. Wall-clock time reports the overall method cost, and Combined Dock denotes the official combined full docking score, where lower is better.}
\label{tab:time_performance_summary}
\end{table}

\subsection{Time--Performance Summary and Runtime Decomposition}
\label{sec:appendix_time_performance}

Table~\ref{tab:time_performance_summary} gives the numeric
counterpart to Figure 4(a) in the main paper. REUSE achieves the strongest final docking performance, but now incurs a moderate runtime increase relative to DualDiff because of the added search overhead. Even so, it remains vastly cheaper than MDRL, which operates in a full adaptation regime rather than frozen-prior reuse.

To make this runtime profile more concrete, Tables~\ref{tab:appendix_runtime_stagewise} and~\ref{tab:appendix_runtime_sizes} provide a finer-grained breakdown under the same official timed evaluation setting used in the main paper. The two tables describe the same setting from complementary perspectives. Table~\ref{tab:appendix_runtime_stagewise} reports the stage-wise wall-clock decomposition, while Table~\ref{tab:appendix_runtime_sizes} reports the corresponding intermediate candidate-set sizes together with the total runtime and the full-docking runtime.

Several points are worth highlighting. First, most of the runtime is spent in the search, decoding, and rapid-screening stage rather than in the reduced-frontier full-docking stage. This is consistent with the design of REUSE: the method invests computation in iterative exploration of the frozen input space, while reserving the most expensive high-fidelity evaluation for a much smaller shortlist. Second, the reduced-frontier full-docking stage still accounts for a substantial but clearly smaller portion of total runtime, which shows that stage-wise pruning is effective in containing the cost of full evaluation. Third, the final panel-construction and incumbent-update step contributes negligibly to the total wall-clock time, indicating that the dominant cost comes from candidate generation and evaluation rather than from the final subset-selection logic itself.

The comparison with DualDiff mainly comes down to where computation is spent. DualDiff performs target-coupled sampling-time intervention followed by dual-target evaluation, and therefore remains relatively lightweight at inference time. REUSE, by contrast, explicitly searches the frozen input space over multiple iterations, which introduces additional runtime beyond one-shot generation. At the same time, REUSE still uses a coarse-to-fine evaluation pipeline and applies the most expensive full docking only to a reduced frontier, which keeps the overall cost far below adaptation-heavy optimization.

MDRL is expensive for a different reason. Unlike REUSE and DualDiff, it does not only run inference with a fixed generator; it also performs a full adaptation procedure. In the full proxy-RL configuration, it first builds an initial population and then alternates between large-batch molecule generation and policy updates over many epochs. The dominant cost is therefore the repeated sample--train loop itself, which is substantially more expensive than frozen-prior reuse. This places MDRL in a much higher cost regime even though its final docking quality remains below REUSE.

\begin{figure*}
    \centering
    \includegraphics[width=1\linewidth]{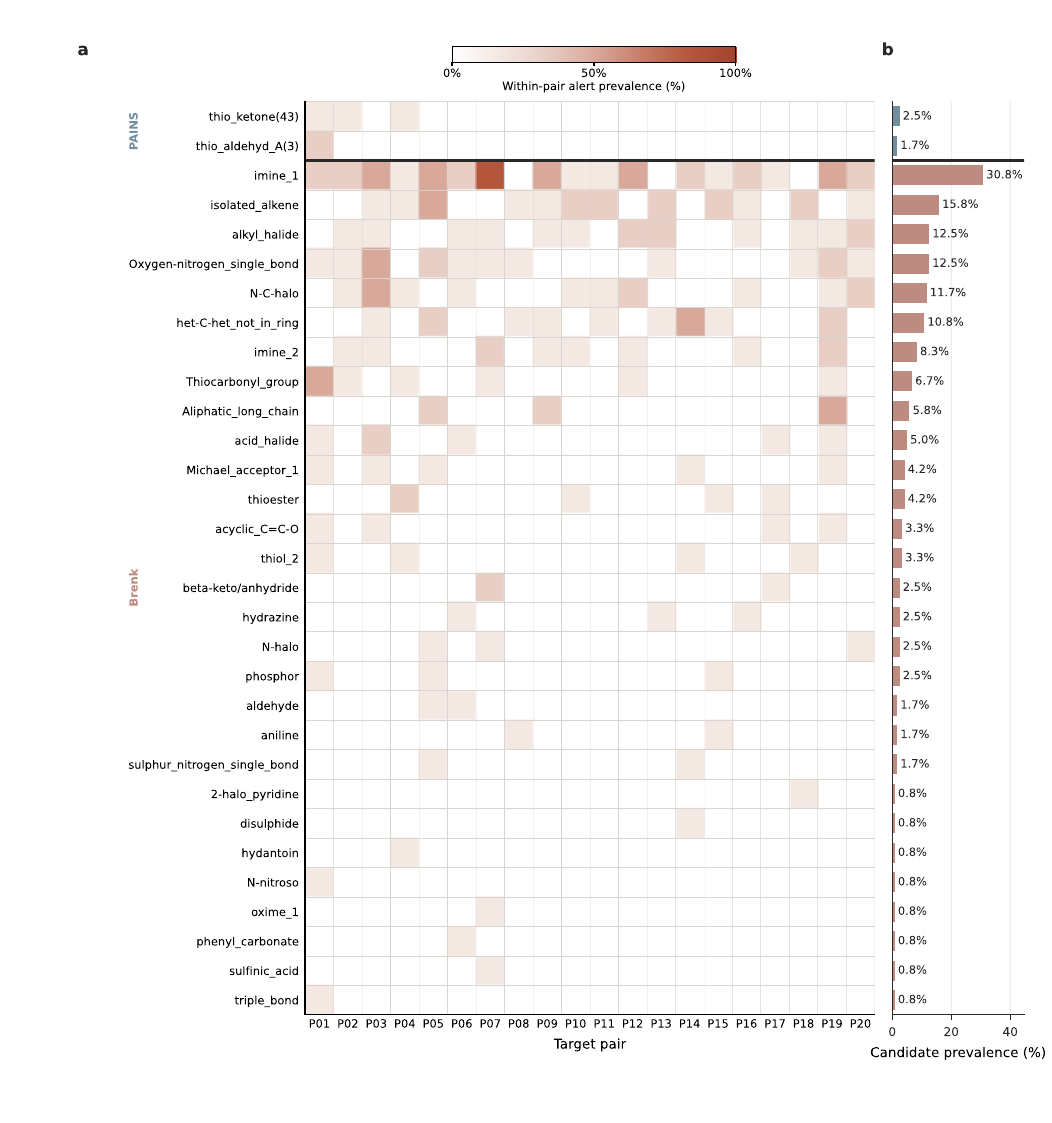}
    \caption{
PAINS and Brenk structural-alert landscape for a randomly sampled subset
of target pairs.
(a) The subset is used for pair-level visualization of the alert
distribution. Rows denote alert motifs, and columns denote the sampled
target pairs. Cell color indicates the proportion of final
\textsc{REUSE} candidates within each target pair that trigger the
corresponding motif, pooled across  random seeds. The horizontal
separator distinguishes the PAINS and Brenk catalogs. 
(b) Prevalence of each motif among candidates associated with the same
sampled target-pair subset, pooled across random seeds. PAINS and
Brenk matches are heuristic structural flags for subsequent inspection
and do not constitute experimental toxicity measurements.
}
    \label{fig:alert}
\end{figure*}

\begin{table*}[t]
\centering
\small
\setlength{\tabcolsep}{5pt}
\begin{tabular}{@{}lcccccc@{}}
\toprule
Frozen Backbone & P-1 Vina $\downarrow$ & P-2 Vina $\downarrow$ & Max Vina $\downarrow$ & Dual High $\uparrow$ & QED $\uparrow$ & SA $\uparrow$ \\
\midrule
TargetDiff & -9.41 & -9.26 & -8.64 & 58.3\% & 0.62 & 0.57 \\
DiffSBDD   & -9.57 & -8.75 & -8.44 & 61.7\% & 0.50 & 0.55 \\
DecompDiff & -8.47 & -8.44 & -7.88 & 43.3\% & 0.52 & 0.56 \\
\bottomrule
\end{tabular}
\caption{Backbone transfer results on the dual-target benchmark. REUSE is applied with three different frozen single-target diffusion backbones.}
\label{tab:backbone_transfer}
\end{table*}

\begin{figure*}[tp]
    \centering
    \includegraphics[width=0.68\textwidth]{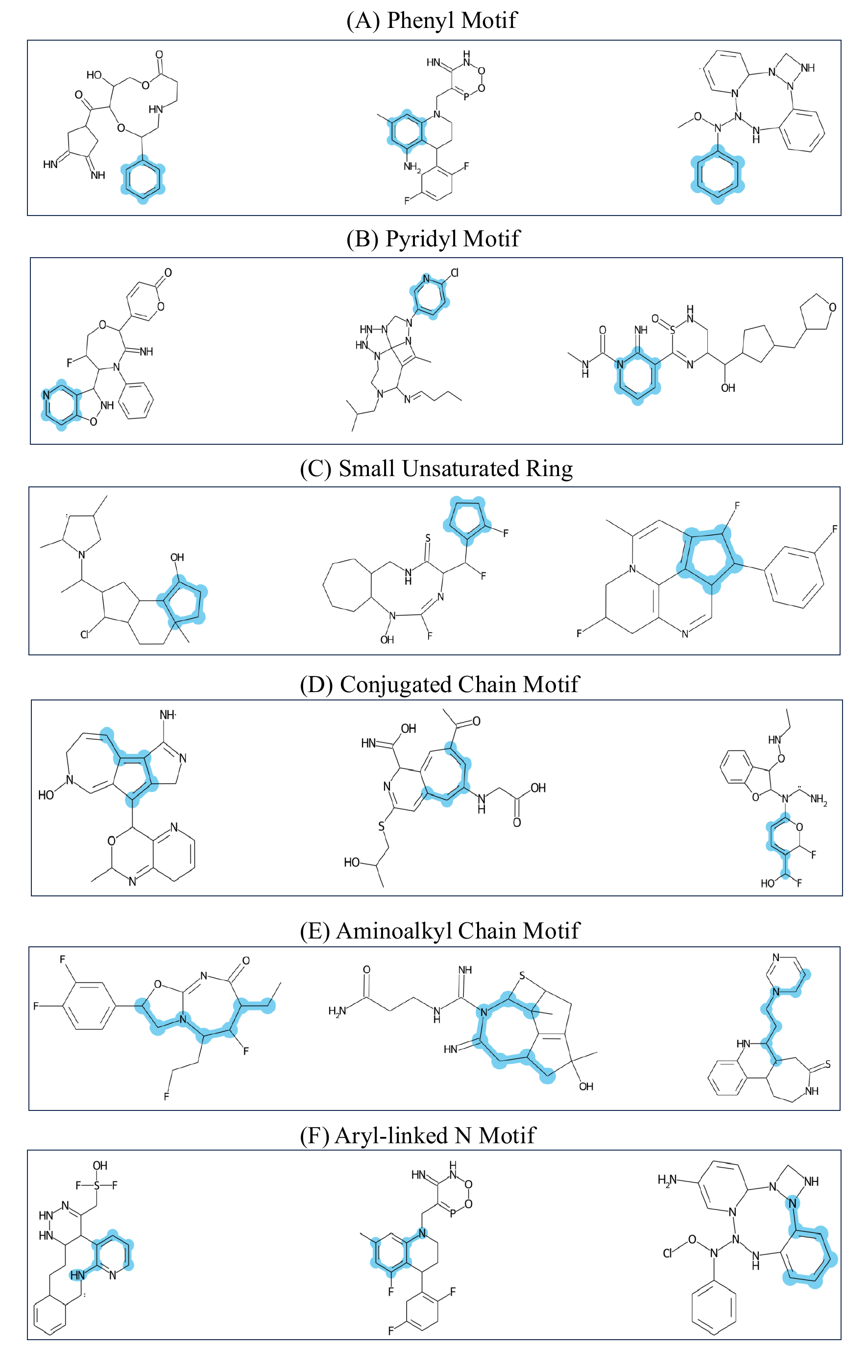}
    \caption{Representative rows illustrating local motif recurrence among globally distinct molecules. Each row contains three molecules from the same candidate pool. Although the overall molecular structures differ at the whole-molecule level, a compact shared motif recurs within the row and is highlighted in blue. The six rows include both ring-like and non-ring motifs, indicating that the phenomenon is not restricted to a single motif class. Row-level structural summaries and molecule-level SA/QED values are reported in Tables~\ref{tab:appendix_motif_row_summary} and~\ref{tab:appendix_motif_molecule_details}.}
    \label{fig:appendix_small_shared_motif}
\end{figure*}

\subsection{Backbone Transfer Analysis}
\label{app:backbone_transfer}

To examine whether \textsc{REUSE} is tied to a particular frozen generator,
we evaluate it with three pretrained single-target diffusion backbones:
TargetDiff, DiffSBDD, and DecompDiff. Table~\ref{tab:backbone_transfer}
reports the primary Vina and chemical-quality results, while
Figure~\ref{fig:backbone_radar} provides a complementary cross-scorer view.
In all cases, the backbone parameters and diffusion sampling process remain
fixed; only the input-space search and selection procedure of \textsc{REUSE}
is applied.

The primary metrics confirm that \textsc{REUSE} can operate with all three
backbones, while revealing some backbone-dependent variation. TargetDiff
achieves P-1/P-2 Vina scores of $-9.41$ and $-9.26$, a Max Vina score of
$-8.64$, and a Dual High rate of $58.3\%$. DiffSBDD obtains the highest
Dual High rate ($61.7\%$) and strongest P-1 Vina score ($-9.57$), although
its P-2 and Max Vina scores are slightly weaker than those of TargetDiff.
DecompDiff produces weaker docking outcomes overall, including a Dual High
rate of $43.3\%$, indicating that the reachable dual-target region depends
partly on the underlying frozen prior.

Despite these differences in the primary metrics, the post hoc cross-scorer
profiles in Figure~\ref{fig:backbone_radar} are broadly similar. TargetDiff
leads on GNINA and Vinardo for both targets and on T2 RTMScore, whereas
DiffSBDD leads on AutoDock4 for both targets and is essentially tied on T1
RTMScore. Together with TargetDiff's stronger T2 Vina and QED/SA results,
this consistency supports its use as the default backbone without suggesting
that \textsc{REUSE} depends on a particular generator.

Overall, \textsc{REUSE} can be instantiated with different frozen
single-target diffusion backbones without retraining or modifying their
denoising dynamics. The remaining variation across backbones is consistent
with the proposed capability-recovery formulation: the quality of the
recovered panel ultimately depends on the expressiveness and chemical
coverage of the frozen prior being searched.

\subsection{Limitations}

Our evaluation emphasizes docking-based affinity, chemical quality, and structural diversity, but does not include wet-lab validation. To partially mitigate the reliance on docking-centered evaluation, we include additional computational sanity checks in Appendix~\ref{app:oracle_structural_sanity_checks}, including cross-scorer consistency, relaxed-pose validation, geometry checks, and residue-level interaction overlap. Accordingly, the recovered molecules should be interpreted as computational candidates rather than experimentally confirmed dual-target binders. In addition, REUSE is designed as a search-based recovery framework over a frozen generator, and its performance may depend on the quality of the underlying pretrained single-target prior as well as the search budget allocated at inference time. Finally, although the method remains substantially cheaper than adaptation-heavy alternatives in our experiments, it still introduces nontrivial iterative search overhead relative to one-shot generation or lighter sampling-time intervention methods, which may limit throughput in large-scale screening settings.

\subsection{Broader Impacts}

This work may have positive impact on computational drug discovery by providing a more efficient and diagnostic way to explore whether complex multi-target behavior can already be recovered from existing generative priors, potentially reducing the need for repeated retraining and lowering the iteration cost of early-stage molecular design. More broadly, the perspective developed here may help researchers better understand what capabilities are already encoded in pretrained scientific generative models.

At the same time, methods for molecular generation can have dual-use implications. In particular, algorithms that improve the recovery of biologically active molecules could in principle be misused to accelerate the design of harmful compounds. Our work does not include experimental validation, does not release a curated set of hazardous molecules, and is evaluated in a standard research setting focused on benchmarked dual-target design. We believe responsible use of such methods should remain subject to existing institutional, legal, and domain-specific safety oversight. Future releases of code or generated assets should take into account appropriate safeguards and usage restrictions where necessary.

\begin{figure*}[tp]
    \centering
    \includegraphics[width=0.68\textwidth]{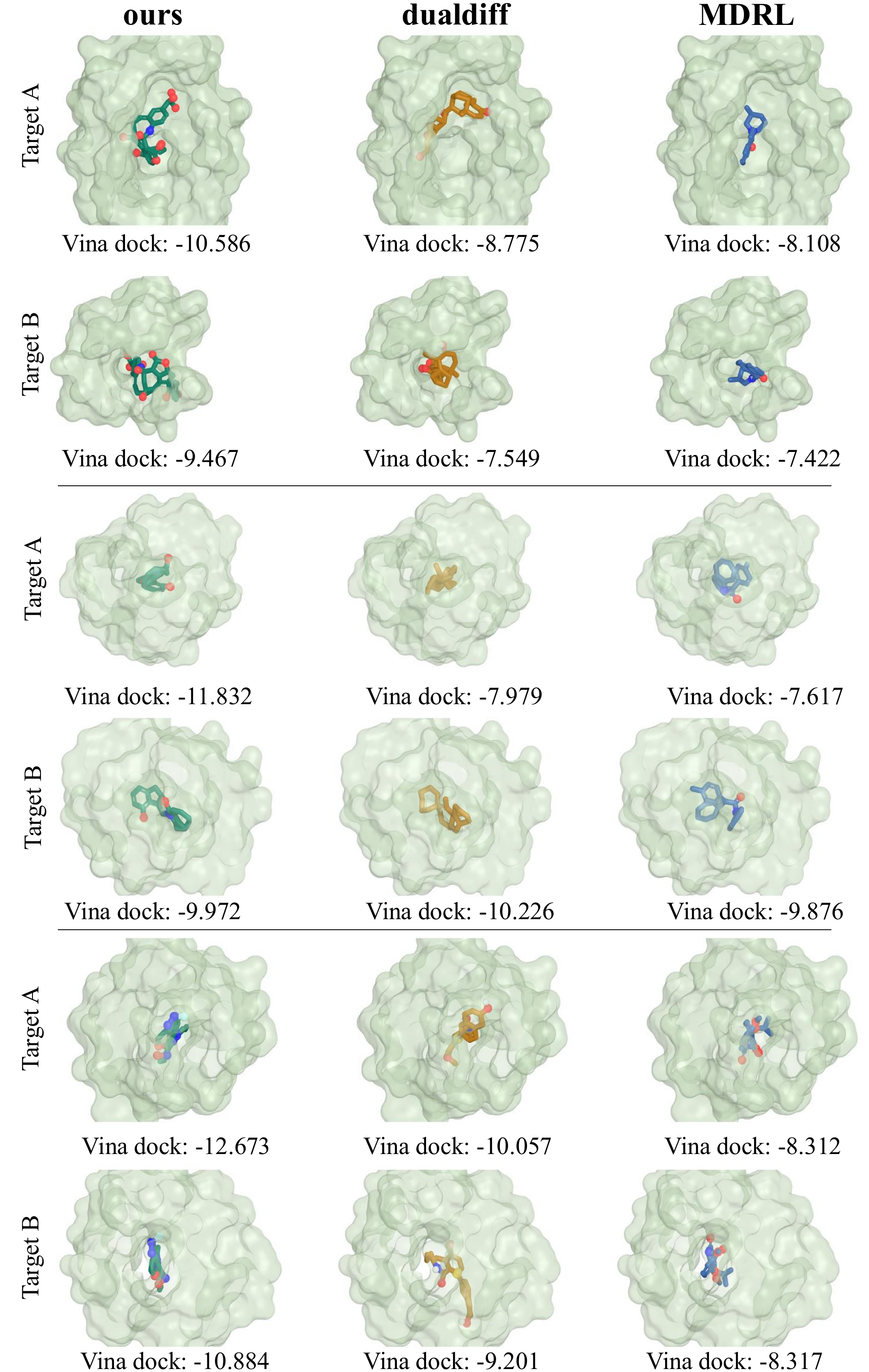}
    \caption{
    Qualitative comparison of binding poses on three representative
    dual-target pairs. For each pair, the left, middle, and right columns
    show molecules from REUSE, DualDiff, and MDRL, respectively; the top
    and bottom rows correspond to Target A and Target B. The reported value
    under each panel is the Vina docking score for the corresponding target,
    where lower is better. Across the three pairs, REUSE achieves lower
    Vina scores than DualDiff in five of the six target-wise comparisons
    and than MDRL in all six comparisons. DualDiff obtains the lower score
    on Target B of the second pair. The examples additionally provide
    qualitative comparisons of pocket placement and cross-target pose
    consistency.
    }
    \label{fig:appendix_binding_visualization}
\end{figure*}


\section{Mathematical Properties and Proofs for REUSE}
\label{app:math_proofs}

In this appendix we formalize several basic mathematical properties of REUSE.
The results are intentionally conservative: they establish well-posedness of the
algorithmic objects, preservation of feasibility/diversity constraints, monotonicity
of the incumbent set under the update rule in Algorithm~\ref{alg:REUSE}, and a
sampling-reachability guarantee induced by the immigration component in the offspring
proposal.
Whenever a statement depends on the exact panel-construction or incumbent-comparison
rule, it is understood to refer only to the deterministic exact-selector
instantiation introduced below.
We do \emph{not} claim global optimality of the frozen generator or of the overall
search procedure.

\subsection{Assumptions and auxiliary definitions}

Fix a target pair $p=(t_a,t_b)$ throughout this appendix.
Recall that the frozen generator induces a decoded molecular family
$\mathcal{M}(z,p)$ for each $z \in \mathcal{Z}$.
We denote the candidate universe for target pair $p$ by
\[
\mathcal{X}_p
:=
\bigcup_{z \in \mathcal{Z}} \mathcal{M}(z,p).
\]

\begin{assumption}[Measurable latent space and prior]
\label{ass:measurable}
The latent/input space $\mathcal{Z}$ is equipped with a $\sigma$-algebra
$\mathscr{A}$, and for each target pair $p$ the initialization prior
$\pi_0(\cdot\mid p)$ is a probability measure on $(\mathcal{Z},\mathscr{A})$.
\end{assumption}

\begin{assumption}[Deterministic finite offspring budgets and finite decoded families]
\label{ass:finite_family}
For each iteration $t$, the offspring population $\mathcal{O}_t$ is finite with
deterministic cardinality
\[
n_t := |\mathcal{O}_t| \in \mathbb{N}.
\]
Moreover, for every $z \in \mathcal{Z}$,
\[
|\mathcal{M}(z,p)| < \infty.
\]
Hence every pooled candidate set induced by finitely many offspring is finite.
\end{assumption}

\begin{assumption}[Subset-valued stage-selection operators]
\label{ass:subset_ops}
For each stage $s \in \{1,\dots,S\}$ and each finite candidate pool
$C \subseteq \mathcal{X}_p$, the operator
\[
\mathrm{Top}_{B_s}(C;\succ_s)
\]
returns a subset of $C$ with cardinality
\[
\left|\mathrm{Top}_{B_s}(C;\succ_s)\right|=\min\{|C|,B_s\}.
\]
\end{assumption}

\begin{definition}[Feasible and $\tau$-diverse set family]
\label{def:feasible_diverse_family}
For any candidate pool $C \subseteq \mathcal{X}_p$, define
\[
\begin{aligned}
\mathfrak{F}_N(C;p,\tau)
&:= \Bigl\{\, S \subseteq C \;\Bigm| \\
&\quad |S|=N,\quad c(m,p)=1\ \forall m\in S, \\
&\quad \Delta(m_i,m_j)\ge \tau\ \forall m_i\neq m_j \in S
\,\Bigr\}.
\end{aligned}
\]
Thus $\mathfrak{F}_N(C;p,\tau)$ is the family of all $N$-element subsets of $C$
that are feasible for the target pair $p$ and pairwise $\tau$-diverse.
\end{definition}

\begin{definition}[Global feasible-set family]
\label{def:global_feasible_family}
Define
\[
\begin{aligned}
\mathfrak{F}^{\mathrm{all}}_N(p,\tau)
&:= \Bigl\{\, S \subseteq \mathcal{X}_p \;\Bigm| \\
&\quad |S|=N,\quad c(m,p)=1\ \forall m\in S, \\
&\quad \Delta(m_i,m_j)\ge \tau\ \forall m_i\neq m_j \in S
\,\Bigr\}.
\end{aligned}
\]
\end{definition}

\begin{assumption}[Exact-selector instantiation of panel construction and incumbent comparison]
\label{ass:scalarized_final_selector}
For the results in this appendix involving the panel-construction and incumbent-update
steps of Algorithm~\ref{alg:REUSE}, we analyze the following deterministic
instantiation.

There exists a deterministic real-valued set utility
\[
J_p : \mathfrak{F}^{\mathrm{all}}_N(p,\tau)\cup\{\emptyset\} \to \mathbb{R}\cup\{-\infty\},
\qquad
J_p(\emptyset):=-\infty.
\]

At iteration $t$, if $\mathfrak{F}_N(\mathcal{C}_t^{(S)};p,\tau)\neq\emptyset$, the
panel-construction step returns an element of
\[
\arg\max_{S \in \mathfrak{F}_N(\mathcal{C}_t^{(S)};p,\tau)} J_p(S).
\]
If $\mathfrak{F}_N(\mathcal{C}_t^{(S)};p,\tau)=\emptyset$, it returns $\emptyset$.

The incumbent update is
\[
S_t^\star \in \arg\max_{S \in \{S_{t-1}^\star,\, S_t\}} J_p(S),
\qquad
S_0^\star := \emptyset.
\]
\end{assumption}

\begin{remark}
\label{rem:scalarized_instantiation}
The main method remains multi-objective.
Assumption~\ref{ass:scalarized_final_selector} does \emph{not} replace the
multi-objective formulation in the main text; it only specifies one deterministic
instantiation of the final panel-selection/comparison step so that the incumbent
update rule can be analyzed formally.
Accordingly, every result below that invokes
Assumption~\ref{ass:scalarized_final_selector} is an exact-selector statement:
it applies to the analyzed exact feasible-panel maximization instantiation and
need not extend to arbitrary heuristic or greedy implementations of the final
panel-construction step.
Equivalently, one may replace $J_p$ by any deterministic total preorder over feasible
candidate panels; the statements below then hold after replacing ``$\arg\max$''
by the corresponding maximal-set selection operator.
\end{remark}

\begin{assumption}[Offspring mixture with immigration]
\label{ass:mixture}
For each iteration $t$, conditional on the selected parent set $\mathcal{A}_t$,
the mutation and crossover kernels
\[
q_{\mathrm{mut}}(\cdot \mid \mathcal{A}_t,p),
\qquad
q_{\mathrm{cross}}(\cdot \mid \mathcal{A}_t,p)
\]
are probability measures on $(\mathcal{Z},\mathscr{A})$, and the offspring proposal
distribution is
\[
\begin{aligned}
q_{\mathrm{off}}(z' \mid \mathcal{A}_t,p)
={}& \alpha_{\mathrm{mut}}\, q_{\mathrm{mut}}(z' \mid \mathcal{A}_t,p) \\
&+ \alpha_{\mathrm{cross}}\, q_{\mathrm{cross}}(z' \mid \mathcal{A}_t,p) \\
&+ \alpha_{\mathrm{imm}}\, \pi_0(z' \mid p),
\end{aligned}
\]
where
\[
\begin{aligned}
&\alpha_{\mathrm{mut}},\alpha_{\mathrm{cross}},\alpha_{\mathrm{imm}} \ge 0, \\
&\alpha_{\mathrm{mut}}+\alpha_{\mathrm{cross}}+\alpha_{\mathrm{imm}}=1, \\
&\alpha_{\mathrm{imm}} > 0.
\end{aligned}
\]
Hence $q_{\mathrm{off}}(\cdot\mid \mathcal{A}_t,p)$ is also a probability measure on
$(\mathcal{Z},\mathscr{A})$.
\end{assumption}

\begin{assumption}[Conditional independence within iteration]
\label{ass:conditional_independence}
Conditional on the history available at the beginning of iteration $t$, the
$n_t$ offspring latent points sampled in iteration $t$ are independent draws from
$q_{\mathrm{off}}(\cdot \mid \mathcal{A}_t,p)$.
\end{assumption}

\begin{definition}[Good latent region]
\label{def:good_region}
For any threshold $\eta \in \mathbb{R}$, define
\[
\begin{aligned}
\mathcal{A}_{\eta,N,\tau}(p)
&:= \Bigl\{\, z \in \mathcal{Z} \;\Bigm| \\
&\quad \exists\, S \subseteq \mathcal{M}(z,p), \\
&\quad S \in \mathfrak{F}_N(\mathcal{M}(z,p);p,\tau), \\
&\quad J_p(S)\ge \eta
\,\Bigr\}.
\end{aligned}
\]
Thus $\mathcal{A}_{\eta,N,\tau}(p)$ is the set of latent points whose decoded
families already contain at least one feasible, $\tau$-diverse $N$-set with utility at least $\eta$.
\end{definition}

\begin{remark}
\label{rem:good_region_sufficient}
Definition~\ref{def:good_region} is intentionally conservative.
It characterizes latent points for which a \emph{single} decoded family
$\mathcal{M}(z,p)$ already contains a good feasible witness set.
Since REUSE operationally constructs panels from the pooled set
$\mathcal{C}_t^{(0)}=\bigcup_{z\in\mathcal{O}_t}\mathcal{M}(z,p)$, good panels may also arise
from combining candidates originating from multiple latent points.
Accordingly, the reachability results below should be interpreted as sufficient-condition
guarantees rather than as a characterization of all possible successful pooled outcomes.
\end{remark}

\begin{assumption}[Measurability of good latent regions]
\label{ass:good_region_measurable}
For every threshold $\eta \in \mathbb{R}$, the set
$\mathcal{A}_{\eta,N,\tau}(p)$ belongs to $\mathscr{A}$.
\end{assumption}

\subsection{Basic structural properties}

\begin{lemma}[Finiteness of all candidate pools]
\label{lem:finiteness}
For every iteration $t$ and every stage $s=0,\dots,S$, the candidate pool
$\mathcal{C}_t^{(s)}$ appearing in Algorithm~\ref{alg:REUSE} is finite.
\end{lemma}

\begin{proof}
By Assumption~\ref{ass:finite_family}, the offspring population
$\mathcal{O}_t$ is finite, with cardinality $n_t<\infty$.
For each $z\in \mathcal{O}_t$, the decoded family $\mathcal{M}(z,p)$ is finite.
Therefore
\[
\mathcal{C}_t^{(0)}
=
\bigcup_{z \in \mathcal{O}_t} \mathcal{M}(z,p)
\]
is a finite union of finite sets, hence finite.

We now argue by induction over the stage index.
Assume $\mathcal{C}_t^{(s-1)}$ is finite for some $s\in\{1,\dots,S\}$.
By Assumption~\ref{ass:subset_ops},
\[
\mathcal{C}_t^{(s)}
=
\mathrm{Top}_{B_s}\big(\mathcal{C}_t^{(s-1)};\succ_s\big)
\subseteq
\mathcal{C}_t^{(s-1)},
\]
so $\mathcal{C}_t^{(s)}$ is finite.
Thus all stage-wise candidate pools are finite.
\end{proof}

\begin{proposition}[Nestedness of stage-wise candidate pools]
\label{prop:nestedness}
For every iteration $t$ and every stage $s=1,\dots,S$,
\[
\mathcal{C}_t^{(s)} \subseteq \mathcal{C}_t^{(s-1)}.
\]
Consequently,
\[
\mathcal{C}_t^{(S)} \subseteq \mathcal{C}_t^{(S-1)} \subseteq \cdots \subseteq \mathcal{C}_t^{(0)}.
\]
Hence every molecule in the panel selected at iteration $t$ belongs to
$\mathcal{C}_t^{(0)}$.
\end{proposition}

\begin{proof}
For each $s\in\{1,\dots,S\}$, Assumption~\ref{ass:subset_ops} gives
\[
\mathcal{C}_t^{(s)}
=
\mathrm{Top}_{B_s}\big(\mathcal{C}_t^{(s-1)};\succ_s\big)
\subseteq
\mathcal{C}_t^{(s-1)}.
\]
Chaining these inclusions yields
\[
\mathcal{C}_t^{(S)} \subseteq \mathcal{C}_t^{(S-1)} \subseteq \cdots \subseteq \mathcal{C}_t^{(0)}.
\]
Since the panel-construction step selects $S_t$ from $\mathcal{C}_t^{(S)}$, we also have
\[
S_t \subseteq \mathcal{C}_t^{(S)} \subseteq \mathcal{C}_t^{(0)}.
\]
\end{proof}

\begin{theorem}[Well-definedness of the panel-construction step under the exact-selector instantiation]
\label{thm:final_selection_well_defined}
Fix an iteration $t$.
Under Assumption~\ref{ass:scalarized_final_selector}, if
$\mathfrak{F}_N(\mathcal{C}_t^{(S)};p,\tau)\neq \emptyset$, then the panel-construction step in
Algorithm~\ref{alg:REUSE} is well-defined and returns some
\[
S_t \in \mathfrak{F}_N(\mathcal{C}_t^{(S)};p,\tau).
\]
If $\mathfrak{F}_N(\mathcal{C}_t^{(S)};p,\tau)=\emptyset$, then it returns $\emptyset$.
\end{theorem}

\begin{proof}
By Lemma~\ref{lem:finiteness}, the set $\mathcal{C}_t^{(S)}$ is finite.
Hence its power set is finite, and therefore the subfamily
$\mathfrak{F}_N(\mathcal{C}_t^{(S)};p,\tau)$ is also finite.

If this family is nonempty, Assumption~\ref{ass:scalarized_final_selector} states that the
panel-construction step returns an element of
\[
\arg\max_{S \in \mathfrak{F}_N(\mathcal{C}_t^{(S)};p,\tau)} J_p(S).
\]
Since the domain is finite and nonempty, the maximum exists.
Thus the step is well-defined and the returned set belongs to
$\mathfrak{F}_N(\mathcal{C}_t^{(S)};p,\tau)$.

If the family is empty, the same assumption states that the step returns $\emptyset$.
\end{proof}

\begin{remark}
\label{rem:exact_selector_scope}
Theorem~\ref{thm:final_selection_well_defined} and the subsequent
utility-comparison results in this subsection analyze only the exact-selector
instantiation from Assumption~\ref{ass:scalarized_final_selector}. They do not
establish corresponding optimality guarantees for arbitrary heuristic solvers of
the final panel-construction problem.
\end{remark}

\begin{corollary}[Constraint preservation of the returned panel under the exact-selector instantiation]
\label{cor:constraint_preservation}
If $S_t \neq \emptyset$, then
\[
\begin{aligned}
|S_t| &= N, \\
c(m,p) &= 1 && \forall m\in S_t, \\
\Delta(m_i,m_j) &\ge \tau
&& \forall m_i\neq m_j \in S_t.
\end{aligned}
\]
\end{corollary}

\begin{proof}
This follows immediately from Theorem~\ref{thm:final_selection_well_defined}
and Definition~\ref{def:feasible_diverse_family}.
\end{proof}

\begin{proposition}[Recovery from the terminal survivor pool under the exact-selector instantiation]
\label{prop:recovery_from_shortlist}
Fix an iteration $t$ and a threshold $\eta \in \mathbb{R}$.
If there exists
\[
\widehat{S} \in \mathfrak{F}_N(\mathcal{C}_t^{(S)};p,\tau)
\quad \text{such that} \quad
J_p(\widehat{S}) \ge \eta,
\]
then the panel selected at iteration $t$ satisfies
\[
J_p(S_t) \ge \eta.
\]
\end{proposition}

\begin{proof}
Because $\widehat{S} \in \mathfrak{F}_N(\mathcal{C}_t^{(S)};p,\tau)$, the feasible family
$\mathfrak{F}_N(\mathcal{C}_t^{(S)};p,\tau)$ is nonempty.
By Assumption~\ref{ass:scalarized_final_selector}, the panel-construction step returns
\[
S_t \in \arg\max_{S \in \mathfrak{F}_N(\mathcal{C}_t^{(S)};p,\tau)} J_p(S).
\]
Therefore
\[
J_p(S_t) \ge J_p(\widehat{S}) \ge \eta.
\]
\end{proof}

\begin{remark}
\label{rem:terminal_pool_boundary}
Proposition~\ref{prop:recovery_from_shortlist} is a terminal-pool statement.
It says that once a good feasible witness set is already present in
$\mathcal{C}_t^{(S)}$, the exact final selector cannot return a strictly worse
utility value. It does not assert that the preceding stage-wise filtering must
preserve every witness set that may exist in the larger pooled family
$\mathcal{C}_t^{(0)}$.
\end{remark}

\begin{corollary}[Constraint preservation of the incumbent under the exact-selector instantiation]
\label{cor:incumbent_constraints}
For every iteration $t$, if $S_t^\star \neq \emptyset$, then
\[
\begin{aligned}
|S_t^\star| &= N, \\
c(m,p) &= 1 && \forall m\in S_t^\star, \\
\Delta(m_i,m_j) &\ge \tau
&& \forall m_i\neq m_j \in S_t^\star.
\end{aligned}
\]
\end{corollary}

\begin{proof}
We argue by induction on $t$.
At $t=0$, we have $S_0^\star=\emptyset$, so the claim is vacuous.

Assume the statement holds for $t-1$.
By Assumption~\ref{ass:scalarized_final_selector}, the incumbent at iteration $t$ is chosen from
\[
\{S_{t-1}^\star,\, S_t\}.
\]
If $S_t^\star = S_{t-1}^\star$, the induction hypothesis implies the claim.
If $S_t^\star = S_t$, then Corollary~\ref{cor:constraint_preservation} implies the claim.
Thus the statement holds for iteration $t$.
\end{proof}

\subsection{Monotonicity of the incumbent under the update rule}

\begin{theorem}[Monotonicity of incumbent utility under the exact-selector instantiation]
\label{thm:monotonicity}
For every iteration $t\ge 1$,
\[
J_p(S_t^\star) \ge J_p(S_{t-1}^\star).
\]
Hence the sequence $\{J_p(S_t^\star)\}_{t\ge 0}$ is nondecreasing.
\end{theorem}

\begin{proof}
By Assumption~\ref{ass:scalarized_final_selector},
\[
S_t^\star \in \arg\max_{S \in \{S_{t-1}^\star,\, S_t\}} J_p(S).
\]
Therefore,
\[
J_p(S_t^\star)
=
\max\{J_p(S_{t-1}^\star),\, J_p(S_t)\}
\ge
J_p(S_{t-1}^\star).
\]
This proves the one-step inequality.
Applying it recursively yields that the sequence
$\{J_p(S_t^\star)\}_{t\ge 0}$ is nondecreasing.
\end{proof}

\begin{corollary}[Persistence after recovery under the exact-selector instantiation]
\label{cor:persistence_after_recovery}
If for some iteration $t_0$ there exists
\[
\widehat{S} \in \mathfrak{F}_N(\mathcal{C}_{t_0}^{(S)};p,\tau)
\quad \text{such that} \quad
J_p(\widehat{S}) \ge \eta,
\]
then
\[
J_p(S_t^\star) \ge \eta
\qquad \text{for all } t \ge t_0.
\]
\end{corollary}

\begin{proof}
By Proposition~\ref{prop:recovery_from_shortlist},
\[
J_p(S_{t_0}) \ge \eta.
\]
By Assumption~\ref{ass:scalarized_final_selector},
\[
J_p(S_{t_0}^\star)
=
\max\{J_p(S_{t_0-1}^\star),\, J_p(S_{t_0})\}
\ge \eta.
\]
The conclusion for all $t\ge t_0$ follows from Theorem~\ref{thm:monotonicity}.
\end{proof}

\begin{remark}
Theorem~\ref{thm:monotonicity} is an \emph{incumbent} monotonicity result.
It does not assert that every intermediate population, frontier, or selected set
must improve at every iteration.
It only states that once the algorithm stores the best-so-far panel, the utility of
that incumbent cannot decrease under the exact-selector incumbent-update rule
specified in Assumption~\ref{ass:scalarized_final_selector}.
\end{remark}

\subsection{Sampling reachability guarantee induced by immigration}

We next formalize the role of the immigration component
$\alpha_{\mathrm{imm}}\pi_0(\cdot\mid p)$ in the offspring mixture.
The results in this subsection are sampling statements in latent space: they
quantify the probability of visiting specified measurable regions of
$\mathcal{Z}$. By themselves, they do not imply that decoded witness sets from
such regions survive later stage-wise filtering or are returned by the final
panel-construction step.

\begin{lemma}[Uniform lower bound from immigration]
\label{lem:one_step_lower_bound}
For any measurable set $A \in \mathscr{A}$, any iteration $t$, and any selected
parent set $\mathcal{A}_t$,
\[
q_{\mathrm{off}}(A \mid \mathcal{A}_t,p)
\ge
\alpha_{\mathrm{imm}}\, \pi_0(A \mid p).
\]
\end{lemma}

\begin{proof}
By Assumption~\ref{ass:mixture},
\begin{align*}
q_{\mathrm{off}}(A \mid \mathcal{A}_t,p)
&= \alpha_{\mathrm{mut}}\,
q_{\mathrm{mut}}(A \mid \mathcal{A}_t,p) \\
&\quad + \alpha_{\mathrm{cross}}\,
q_{\mathrm{cross}}(A \mid \mathcal{A}_t,p) \\
&\quad + \alpha_{\mathrm{imm}}\, \pi_0(A \mid p).
\end{align*}
Each term on the right-hand side is nonnegative because it is a nonnegative weight
times a probability measure evaluated on $A$.
Therefore,
\[
q_{\mathrm{off}}(A \mid \mathcal{A}_t,p)
\ge
\alpha_{\mathrm{imm}}\, \pi_0(A \mid p).
\]
\end{proof}

\begin{theorem}[Finite-horizon hitting probability]
\label{thm:hitting_probability}
Let $A \in \mathscr{A}$ be any measurable subset of $\mathcal{Z}$, and let
$E_T(A)$ denote the event that at least one offspring latent point sampled during the
first $T$ iterations belongs to $A$.
Then
\[
\mathbb{P}\big(E_T(A)\big)
\ge
1 - \prod_{t=1}^T \left(1-\alpha_{\mathrm{imm}}\pi_0(A\mid p)\right)^{n_t}.
\]
In particular, if the offspring budget is constant, $n_t \equiv n$, then
\[
\mathbb{P}\big(E_T(A)\big)
\ge
1 - \left(1-\alpha_{\mathrm{imm}}\pi_0(A\mid p)\right)^{nT}.
\]
\end{theorem}

\begin{proof}
Set
\[
r_A := \alpha_{\mathrm{imm}}\pi_0(A\mid p).
\]
By Assumption~\ref{ass:mixture}, we have $r_A \in [0,1]$.

Let $\mathcal{H}_t$ denote the full history available at the beginning of iteration $t$,
before the $n_t$ offspring are drawn.
For any offspring draw $z_t^{(i)}$ in iteration $t$, Lemma~\ref{lem:one_step_lower_bound} gives
\[
\mathbb{P}\big(z_t^{(i)} \in A \mid \mathcal{H}_t\big)
=
q_{\mathrm{off}}(A \mid \mathcal{A}_t,p)
\ge r_A.
\]
Hence
\[
\mathbb{P}\big(z_t^{(i)} \notin A \mid \mathcal{H}_t\big)
\le 1-r_A.
\]
By Assumption~\ref{ass:conditional_independence}, conditional on $\mathcal{H}_t$ the $n_t$
offspring draws in iteration $t$ are independent.
Therefore, if we define the event
\[
M_t(A) := \bigcap_{i=1}^{n_t} \{ z_t^{(i)} \notin A\},
\]
then
\[
\mathbb{P}\big(M_t(A)\mid \mathcal{H}_t\big)
=
\prod_{i=1}^{n_t} \mathbb{P}\big(z_t^{(i)} \notin A \mid \mathcal{H}_t\big)
\le
(1-r_A)^{n_t}.
\]

Now note that
\[
E_T(A)^c = \bigcap_{t=1}^T M_t(A).
\]
We prove by induction on $T$ that
\[
\mathbb{P}\!\left(\bigcap_{t=1}^T M_t(A)\right)
\le
\prod_{t=1}^T (1-r_A)^{n_t}.
\]
For $T=1$, this follows by taking expectation of
$\mathbb{P}(M_1(A)\mid \mathcal{H}_1)\le (1-r_A)^{n_1}$.

Assume the claim holds for $T-1$.
Then, by the tower property,
\begin{align*}
\mathbb{P}\!\left(\bigcap_{t=1}^T M_t(A)\right)
&=
\mathbb{E}\!\left[
\mathbf{1}_{\cap_{t=1}^{T-1} M_t(A)}
\,
\mathbb{P}\big(M_T(A)\mid \mathcal{H}_T\big)
\right] \\
&\le
(1-r_A)^{n_T}
\,
\mathbb{P}\!\left(\bigcap_{t=1}^{T-1} M_t(A)\right) \\
&\le
(1-r_A)^{n_T}
\prod_{t=1}^{T-1} (1-r_A)^{n_t}
\\
&=
\prod_{t=1}^T (1-r_A)^{n_t}.
\end{align*}
This completes the induction.

Finally,
\begin{align*}
\mathbb{P}\big(E_T(A)\big)
&=
1 - \mathbb{P}\big(E_T(A)^c\big)
\\
&=
1 - \mathbb{P}\!\left(\bigcap_{t=1}^T M_t(A)\right)
\\
&\ge
1 - \prod_{t=1}^T (1-r_A)^{n_t},
\end{align*}
which is the desired bound.
The constant-budget form follows immediately when $n_t \equiv n$.
\end{proof}

\begin{corollary}[Asymptotic hitting guarantee]
\label{cor:asymptotic_hitting}
If $\pi_0(A\mid p) > 0$ and
\[
\sum_{t=1}^\infty n_t = \infty,
\]
then
\[
\lim_{T\to\infty} \mathbb{P}\big(E_T(A)\big) = 1.
\]
\end{corollary}

\begin{proof}
By Theorem~\ref{thm:hitting_probability},
\[
\mathbb{P}\big(E_T(A)^c\big)
\le
\prod_{t=1}^T (1-r_A)^{n_t},
\qquad
r_A=\alpha_{\mathrm{imm}}\pi_0(A\mid p).
\]
Because $\alpha_{\mathrm{imm}}>0$ and $\pi_0(A\mid p)>0$, we have $r_A\in(0,1]$.
If $r_A=1$, then the product is identically zero and the conclusion is immediate.
If $r_A\in(0,1)$, then $0<1-r_A<1$ and
\[
\prod_{t=1}^T (1-r_A)^{n_t}
=
(1-r_A)^{\sum_{t=1}^T n_t}.
\]
Since $\sum_{t=1}^\infty n_t=\infty$, the exponent diverges to $+\infty$,
so
\[
(1-r_A)^{\sum_{t=1}^T n_t} \to 0.
\]
Hence $\mathbb{P}(E_T(A)^c)\to 0$, equivalently
$\mathbb{P}(E_T(A))\to 1$.
\end{proof}

\begin{corollary}[Sampling of good latent regions]
\label{cor:good_region_recovery}
For any threshold $\eta \in \mathbb{R}$,
\begin{align*}
\mathbb{P}\!\left(E_T\!\left(\mathcal{A}_{\eta,N,\tau}(p)\right)\right)
&\ge
1 -
\prod_{t=1}^T
\Bigl[
1-
\\
&\qquad \alpha_{\mathrm{imm}}
\pi_0\!\left(\mathcal{A}_{\eta,N,\tau}(p)\mid p\right)
\Bigr]^{n_t}.
\end{align*}
Therefore, if
\[
\pi_0\!\left(\mathcal{A}_{\eta,N,\tau}(p)\mid p\right) > 0
\quad \text{and} \quad
\sum_{t=1}^\infty n_t = \infty,
\]
then with probability tending to one, REUSE eventually samples an offspring latent
point whose decoded family contains a feasible, $\tau$-diverse $N$-molecule
witness set with utility at least $\eta$.
\end{corollary}

\begin{proof}
By Assumption~\ref{ass:good_region_measurable}, the set
$\mathcal{A}_{\eta,N,\tau}(p)$ is measurable.
Apply Theorem~\ref{thm:hitting_probability} and
Corollary~\ref{cor:asymptotic_hitting} with
\[
A = \mathcal{A}_{\eta,N,\tau}(p).
\]
\end{proof}

\begin{remark}
Corollary~\ref{cor:good_region_recovery} is a \emph{sampling/reachability} result.
It shows that immigration prevents the search from assigning zero probability to any
good latent region that already has positive prior mass under $\pi_0(\cdot\mid p)$.
It does not, by itself, imply that every such witness set must survive stage-wise
selection, nor that the final returned panel must attain the same utility level as
the witness set, nor does it imply global optimality.
\end{remark}

\subsection{From latent reachability to threshold attainment}

The preceding reachability results establish that immigration can eventually sample
latent regions whose decoded families contain good witnesses. The missing step is
algorithmic: such witnesses must still survive the downstream decode-plus-selection
pipeline before the exact final selector can recover them. We therefore introduce a
conservative iteration-level bridge assumption. It should be interpreted as an abstract
threshold-attainment assumption rather than as a mechanistic description of REUSE.

\begin{definition}[Iteration-level sampled-good event]
\label{def:sampled_good_event}
For any measurable set $A \in \mathscr{A}$ and iteration $t$, define
\[
G_t(A)
:=
\left\{
\exists\, i \in \{1,\dots,n_t\} : z_t^{(i)} \in A
\right\}.
\]
Thus $G_t(A)$ is the event that iteration $t$ samples at least one offspring latent
point in $A$.
\end{definition}

\begin{assumption}[Iteration-level survivability bridge]
\label{ass:iteration_survivability}
Fix a threshold $\eta \in \mathbb{R}$.
There exists a measurable set
\[
\mathcal{A}^{\mathrm{surv}}_{\eta,N,\tau}(p) \in \mathscr{A}
\]
and a constant $\rho_\eta \in (0,1]$ such that for every iteration $t$,
\begin{align*}
&\mathbb{P}\!\Bigl(
\exists\, \widehat{S} \in \mathfrak{F}_N(\mathcal{C}_t^{(S)};p,\tau)
\text{ with } J_p(\widehat{S})\ge \eta
\\
&\qquad \Bigm|\
\mathcal{H}_t,\,
G_t\!\left(\mathcal{A}^{\mathrm{surv}}_{\eta,N,\tau}(p)\right)
\Bigr)
\ge \rho_\eta
\end{align*}
almost surely on the event
$G_t(\mathcal{A}^{\mathrm{surv}}_{\eta,N,\tau}(p))$.
\end{assumption}

\begin{remark}
Assumption~\ref{ass:iteration_survivability} is the conservative bridge that lifts
latent-space sampling to algorithm-level attainment. It does not attempt to prove,
from first principles, why stage-wise selection preserves a witness. Instead, it
isolates the weakest additional statement needed to turn latent reachability into
threshold attainment by the exact-selector incumbent. Accordingly, it should be
read as an abstract bridge assumption rather than as an explanation of the internal
selection mechanics of REUSE.
\end{remark}

\begin{definition}[Terminal witness event and finite-horizon attainment event]
\label{def:terminal_witness_event}
For any threshold $\eta \in \mathbb{R}$ and iteration $t$, define
\[
B_t(\eta)
:=
\left\{
\exists\, \widehat{S} \in \mathfrak{F}_N(\mathcal{C}_t^{(S)};p,\tau)
\text{ with } J_p(\widehat{S})\ge \eta
\right\}.
\]
For any horizon $T \ge 1$, define
\[
F_T(\eta)
:=
\bigcup_{t=1}^T B_t(\eta).
\]
\end{definition}

\begin{lemma}[Iteration-level sampled-good probability]
\label{lem:iteration_sampled_good}
For any measurable set $A \in \mathscr{A}$ and iteration $t$,
\[
\mathbb{P}\big(G_t(A)^c \mid \mathcal{H}_t\big)
\le
\left(1-\alpha_{\mathrm{imm}}\pi_0(A\mid p)\right)^{n_t}.
\]
Equivalently,
\[
\mathbb{P}\big(G_t(A) \mid \mathcal{H}_t\big)
\ge
1-\left(1-\alpha_{\mathrm{imm}}\pi_0(A\mid p)\right)^{n_t}.
\]
\end{lemma}

\begin{proof}
By Lemma~\ref{lem:one_step_lower_bound}, each offspring draw in iteration $t$
satisfies
\[
\mathbb{P}\big(z_t^{(i)} \in A \mid \mathcal{H}_t\big)
\ge
\alpha_{\mathrm{imm}}\pi_0(A\mid p).
\]
Therefore
\[
\mathbb{P}\big(z_t^{(i)} \notin A \mid \mathcal{H}_t\big)
\le
1-\alpha_{\mathrm{imm}}\pi_0(A\mid p).
\]
By Assumption~\ref{ass:conditional_independence}, conditional on $\mathcal{H}_t$
the offspring draws are independent. Hence
\begin{align*}
\mathbb{P}\big(G_t(A)^c \mid \mathcal{H}_t\big)
&=
\mathbb{P}\!\left(\bigcap_{i=1}^{n_t}\{z_t^{(i)}\notin A\}\middle|\mathcal{H}_t\right) \\
&\le
\left(1-\alpha_{\mathrm{imm}}\pi_0(A\mid p)\right)^{n_t}.
\end{align*}
The complementary lower bound follows immediately.
\end{proof}

\begin{theorem}[Finite-horizon threshold-attainment guarantee under the survivability bridge]
\label{thm:threshold_attainment_conservative}
Fix a threshold $\eta \in \mathbb{R}$ and suppose
Assumption~\ref{ass:iteration_survivability} holds.
Then for every iteration $t$,
\begin{align*}
\mathbb{P}\big(B_t(\eta)\mid \mathcal{H}_t\big)
&\ge
\rho_\eta
\Bigl[
1-
\\
&\qquad \left(
1-\alpha_{\mathrm{imm}}
\pi_0\!\left(\mathcal{A}^{\mathrm{surv}}_{\eta,N,\tau}(p)\mid p\right)
\right)^{n_t}
\Bigr].
\end{align*}
Consequently, for every $T \ge 1$,
\begin{align*}
\mathbb{P}\big(F_T(\eta)\big)
&\ge
1-
\prod_{t=1}^T
\Bigl(
1-\rho_\eta
\Bigl[
1-
\\
&\qquad \left(
1-\alpha_{\mathrm{imm}}
\pi_0\!\left(\mathcal{A}^{\mathrm{surv}}_{\eta,N,\tau}(p)\mid p\right)
\right)^{n_t}
\Bigr]
\Bigr).
\end{align*}
\end{theorem}

\begin{proof}
By Assumption~\ref{ass:iteration_survivability},
\begin{align*}
&\mathbb{P}\Bigl(
B_t(\eta)\,\Bigm|\, \mathcal{H}_t,
G_t(\mathcal{A}^{\mathrm{surv}}_{\eta,N,\tau}(p))
\Bigr)
\ge \rho_\eta
\end{align*}
almost surely on $G_t(\mathcal{A}^{\mathrm{surv}}_{\eta,N,\tau}(p))$.
Therefore
\begin{align*}
\mathbb{P}\big(B_t(\eta)\mid \mathcal{H}_t\big)
&\ge
\mathbb{P}\Bigl(
B_t(\eta)\cap
G_t(\mathcal{A}^{\mathrm{surv}}_{\eta,N,\tau}(p))
\,\Bigm|\, \mathcal{H}_t\Bigr) \\
&=
\mathbb{P}\Bigl(
B_t(\eta)\,\Bigm|\, \mathcal{H}_t,
G_t(\mathcal{A}^{\mathrm{surv}}_{\eta,N,\tau}(p))
\Bigr) \\
&\quad \times
\mathbb{P}\Bigl(
G_t(\mathcal{A}^{\mathrm{surv}}_{\eta,N,\tau}(p))
\,\Bigm|\, \mathcal{H}_t\Bigr) \\
&\ge
\rho_\eta
\mathbb{P}\Bigl(
G_t(\mathcal{A}^{\mathrm{surv}}_{\eta,N,\tau}(p))
\,\Bigm|\, \mathcal{H}_t\Bigr).
\end{align*}
Applying Lemma~\ref{lem:iteration_sampled_good} with
$A=\mathcal{A}^{\mathrm{surv}}_{\eta,N,\tau}(p)$ gives the one-step lower bound.

Define
\[
q_t(\eta)
:=
\rho_\eta
\left[
1-\left(
1-\alpha_{\mathrm{imm}}
\pi_0\!\left(\mathcal{A}^{\mathrm{surv}}_{\eta,N,\tau}(p)\mid p\right)
\right)^{n_t}
\right].
\]
Then
\[
\mathbb{P}\big(B_t(\eta)^c\mid \mathcal{H}_t\big)\le 1-q_t(\eta).
\]
Using the tower property iteratively,
\begin{align*}
\mathbb{P}\big(F_T(\eta)^c\big)
&=
\mathbb{P}\!\left(\bigcap_{t=1}^T B_t(\eta)^c\right) \\
&=
\mathbb{E}\!\left[
\mathbf{1}_{\cap_{t=1}^{T-1} B_t(\eta)^c}
\mathbb{P}\big(B_T(\eta)^c\mid \mathcal{H}_T\big)
\right] \\
&\le
(1-q_T(\eta))
\mathbb{P}\!\left(\bigcap_{t=1}^{T-1} B_t(\eta)^c\right).
\end{align*}
Induction on $T$ yields
\[
\mathbb{P}\big(F_T(\eta)^c\big)\le \prod_{t=1}^T (1-q_t(\eta)),
\]
which proves the claim.
\end{proof}

\begin{corollary}[Finite-horizon incumbent threshold attainment under the exact-selector instantiation]
\label{cor:incumbent_threshold_attainment}
Fix a threshold $\eta \in \mathbb{R}$ and suppose
Assumption~\ref{ass:iteration_survivability} and
Assumption~\ref{ass:scalarized_final_selector} hold.
Then for every $T \ge 1$,
\[
\mathbb{P}\!\left(\exists\, t \le T : J_p(S_t^\star)\ge \eta\right)
\ge
\mathbb{P}\big(F_T(\eta)\big).
\]
In particular,
\begin{align*}
&\mathbb{P}\!\left(\exists\, t \le T : J_p(S_t^\star)\ge \eta\right)
\\
&\quad \ge
1-
\prod_{t=1}^T
\Bigl(
1-\rho_\eta
\Bigl[
1-
\\
&\qquad \left(
1-\alpha_{\mathrm{imm}}
\pi_0\!\left(\mathcal{A}^{\mathrm{surv}}_{\eta,N,\tau}(p)\mid p\right)
\right)^{n_t}
\Bigr]
\Bigr).
\end{align*}
\end{corollary}

\begin{proof}
On the event $B_t(\eta)$, the terminal survivor pool contains a feasible witness
set $\widehat{S}$ satisfying $J_p(\widehat{S})\ge \eta$. By
Proposition~\ref{prop:recovery_from_shortlist}, this implies
$J_p(S_t)\ge \eta$. Theorem~\ref{thm:monotonicity} then implies
$J_p(S_t^\star)\ge \eta$ for all later iterations. Therefore the event
$F_T(\eta)$ implies $\{\exists\, t \le T : J_p(S_t^\star)\ge \eta\}$.
\end{proof}

\begin{corollary}[Asymptotic incumbent threshold attainment under the survivability bridge]
\label{cor:asymptotic_threshold_attainment}
Fix a threshold $\eta \in \mathbb{R}$ and suppose
Assumption~\ref{ass:iteration_survivability} and
Assumption~\ref{ass:scalarized_final_selector} hold.
If
\[
\pi_0\!\left(\mathcal{A}^{\mathrm{surv}}_{\eta,N,\tau}(p)\mid p\right) > 0
\qquad \text{and} \qquad
\sum_{t=1}^\infty n_t = \infty,
\]
then
\[
\lim_{T\to\infty}
\mathbb{P}\!\left(\exists\, t \le T : J_p(S_t^\star)\ge \eta\right)=1.
\]
\end{corollary}

\begin{proof}
Let
\[
r_\eta
:=
\alpha_{\mathrm{imm}}
\pi_0\!\left(\mathcal{A}^{\mathrm{surv}}_{\eta,N,\tau}(p)\mid p\right).
\]
Because $r_\eta>0$ and $n_t\ge 1$ for every iteration,
\[
1-(1-r_\eta)^{n_t} \ge r_\eta.
\]
Therefore the one-step lower bounds from
Theorem~\ref{thm:threshold_attainment_conservative} satisfy
\[
q_t(\eta)\ge \rho_\eta r_\eta
\qquad \text{for every } t.
\]
Since $\rho_\eta r_\eta > 0$, we obtain
\[
\prod_{t=1}^T (1-q_t(\eta))
\le
(1-\rho_\eta r_\eta)^T \to 0.
\]
The claim now follows from Corollary~\ref{cor:incumbent_threshold_attainment}.
\end{proof}

\begin{remark}
Theorem~\ref{thm:threshold_attainment_conservative} and
Corollaries~\ref{cor:incumbent_threshold_attainment}--\ref{cor:asymptotic_threshold_attainment}
form the conservative theorem chain for this appendix. They provide a genuine
threshold-attainment guarantee, but only after adding the abstract bridge
Assumption~\ref{ass:iteration_survivability}. This is the safest upgrade path for
the paper because it closes the logical gap from latent reachability to incumbent
quality without introducing stage-wise independence claims or heuristic-selector
approximations.
\end{remark}

\subsection{Algorithm-specific strengthening via stage-wise ranking preservation}

We next record a stronger, more algorithm-specific refinement that makes the role
of REUSE's stage-wise feasibility-first selection explicit. The main theorem in
this subsection is deterministic: once a good witness set lies in the pooled set
and its members remain sufficiently highly ranked under every stage-wise operator
$\mathrm{Top}_{B_s}(\cdot;\succ_s)$, the witness must survive to the terminal pool.

\begin{definition}[Strict dominator set at stage $s$]
\label{def:strict_dominator_set}
Fix an iteration $t$, a stage $s\in\{1,\dots,S\}$, and a molecule
$m$. Define
\[
D_{t,s}(m)
:=
\{m' \in \mathcal{C}_t^{(s-1)} : m' \succ_s m\}.
\]
\end{definition}

\begin{assumption}[Exact stage-truncation consistency within the stronger refinement]
\label{ass:rank_consistent_top}
For each iteration $t$, stage $s\in\{1,\dots,S\}$, finite pool
$C=\mathcal{C}_t^{(s-1)}$, and molecule $m\in C$, if
\[
\big|\{m' \in C : m' \succ_s m\}\big| \le B_s-1,
\]
then
\[
m \in \mathrm{Top}_{B_s}(C;\succ_s).
\]
\end{assumption}

\begin{assumption}[Stage-wise witness feasibility preservation]
\label{ass:witness_stage_feasible}
Fix an iteration $t$, a threshold $\eta \in \mathbb{R}$, and a witness set
$W\subseteq \mathcal{C}_t^{(0)}$.
For every stage $s\in\{1,\dots,S\}$ and every molecule $m\in W$, the quantity
$h_s(m;p)$ is well-defined and $m$ remains feasible for the stage-$s$ ordering.
\end{assumption}

\begin{assumption}[Stage-wise dominator-count bound]
\label{ass:dominator_count}
Fix an iteration $t$, a threshold $\eta \in \mathbb{R}$, and a witness set
$W\subseteq \mathcal{C}_t^{(0)}$.
For every stage $s\in\{1,\dots,S\}$ and every molecule $m\in W$,
\[
|D_{t,s}(m)| \le B_s-1.
\]
\end{assumption}

\begin{theorem}[Deterministic stage-wise survival of a ranking-preserving witness]
\label{thm:deterministic_witness_survival}
Fix an iteration $t$ and a threshold $\eta \in \mathbb{R}$.
Suppose there exists a witness set
\[
W \in \mathfrak{F}_N(\mathcal{C}_t^{(0)};p,\tau)
\qquad \text{such that} \qquad
J_p(W)\ge \eta.
\]
If Assumptions~\ref{ass:witness_stage_feasible} and~\ref{ass:dominator_count}
hold for $W$, and Assumption~\ref{ass:rank_consistent_top} holds, then
\[
W \subseteq \mathcal{C}_t^{(S)}.
\]
In particular,
\[
W \in \mathfrak{F}_N(\mathcal{C}_t^{(S)};p,\tau).
\]
Under Assumption~\ref{ass:scalarized_final_selector}, this further implies
\[
J_p(S_t)\ge \eta.
\]
\end{theorem}

\begin{proof}
We prove by induction on the stage index that
$W \subseteq \mathcal{C}_t^{(s)}$ for all $s=0,\dots,S$.
The base case $s=0$ holds by assumption.

Now assume $W \subseteq \mathcal{C}_t^{(s-1)}$ for some $s\in\{1,\dots,S\}$.
Fix any molecule $m\in W$. By Assumption~\ref{ass:dominator_count},
the strict dominator set $D_{t,s}(m)$ contains at most $B_s-1$ molecules.
Therefore Assumption~\ref{ass:rank_consistent_top} implies
\[
m \in \mathrm{Top}_{B_s}(\mathcal{C}_t^{(s-1)};\succ_s)=\mathcal{C}_t^{(s)}.
\]
Because $m\in W$ was arbitrary, we obtain
$W \subseteq \mathcal{C}_t^{(s)}$.

By induction, $W \subseteq \mathcal{C}_t^{(S)}$.
Since $W$ itself is unchanged, its feasibility, cardinality, and pairwise
$\tau$-diversity are preserved, so
$W \in \mathfrak{F}_N(\mathcal{C}_t^{(S)};p,\tau)$.
Proposition~\ref{prop:recovery_from_shortlist} now implies
$J_p(S_t)\ge J_p(W)\ge \eta$.
\end{proof}

\begin{remark}[Margin-plus-capacity as a sufficient condition]
\label{rem:margin_capacity}
Assumption~\ref{ass:dominator_count} can be justified by a more interpretable
margin-style condition tied to the stage scores. For example, suppose that at stage
$s$ every witness molecule satisfies $h_s(m;p)\ge \theta_{s,\eta}+\delta_{s,\eta}$
for some threshold $\theta_{s,\eta}$ and margin $\delta_{s,\eta}>0$, and that the
number of feasible molecules in $\mathcal{C}_t^{(s-1)}$ with score at least
$\theta_{s,\eta}+\delta_{s,\eta}$ is at most $B_s$. Then every witness molecule
automatically satisfies the dominator-count bound at stage $s$. Thus the theorem
above can be read as a formal version of the intuitive claim that REUSE preserves
a high-scoring frontier when the stage budgets are large enough relative to the
number of genuinely competitive molecules.
\end{remark}

\begin{definition}[Ranking-preserving witness event]
\label{def:ranking_preserving_event}
For any threshold $\eta \in \mathbb{R}$ and iteration $t$, define
\[
\begin{aligned}
E_t^{(0)}(\eta)
:=
\Bigl\{
&\exists\, W \in \mathfrak{F}_N(\mathcal{C}_t^{(0)};p,\tau)
\text{ such that } J_p(W)\ge \eta, \\
& W \text{ satisfies Assumptions~\ref{ass:witness_stage_feasible}
and~\ref{ass:dominator_count}}
\Bigr\}.
\end{aligned}
\]
\end{definition}

\begin{assumption}[Latent-to-ranking bridge]
\label{ass:latent_to_ranking}
Fix a threshold $\eta \in \mathbb{R}$.
There exists a measurable set
\[
\mathcal{A}^{\mathrm{rank}}_{\eta,N,\tau}(p) \in \mathscr{A}
\]
and a constant $\rho_\eta^{\mathrm{rank}}\in(0,1]$ such that for every iteration $t$,
\[
\mathbb{P}\!\left(
E_t^{(0)}(\eta)
\ \middle|\
\mathcal{H}_t,\,
G_t\!\left(\mathcal{A}^{\mathrm{rank}}_{\eta,N,\tau}(p)\right)
\right)
\ge \rho_\eta^{\mathrm{rank}}
\]
almost surely on the event
$G_t(\mathcal{A}^{\mathrm{rank}}_{\eta,N,\tau}(p))$.
\end{assumption}

\begin{theorem}[Finite-horizon threshold attainment under ranking-preserving witnesses]
\label{thm:threshold_attainment_ranking}
Fix a threshold $\eta \in \mathbb{R}$ and suppose
Assumption~\ref{ass:latent_to_ranking} and
Assumption~\ref{ass:scalarized_final_selector} hold.
Then for every $T \ge 1$,
\begin{align*}
&\mathbb{P}\!\left(\exists\, t \le T : J_p(S_t^\star)\ge \eta\right)
\\
&\quad \ge
1-
\prod_{t=1}^T
\Bigl(
1-\rho_\eta^{\mathrm{rank}}
\Bigl[
1-
\\
&\qquad \left(
1-\alpha_{\mathrm{imm}}
\pi_0\!\left(\mathcal{A}^{\mathrm{rank}}_{\eta,N,\tau}(p)\mid p\right)
\right)^{n_t}
\Bigr]
\Bigr).
\end{align*}
\end{theorem}

\begin{proof}
By Assumption~\ref{ass:latent_to_ranking},
\[
\mathbb{P}\big(E_t^{(0)}(\eta)\mid \mathcal{H}_t\big)
\ge
\rho_\eta^{\mathrm{rank}}
\mathbb{P}\big(G_t(\mathcal{A}^{\mathrm{rank}}_{\eta,N,\tau}(p))\mid \mathcal{H}_t\big).
\]
Applying Lemma~\ref{lem:iteration_sampled_good} with
$A=\mathcal{A}^{\mathrm{rank}}_{\eta,N,\tau}(p)$ yields
\begin{align*}
\mathbb{P}\big(E_t^{(0)}(\eta)\mid \mathcal{H}_t\big)
&\ge
\rho_\eta^{\mathrm{rank}}
\Bigl[
1-
\\
&\qquad \left(
1-\alpha_{\mathrm{imm}}
\pi_0\!\left(\mathcal{A}^{\mathrm{rank}}_{\eta,N,\tau}(p)\mid p\right)
\right)^{n_t}
\Bigr].
\end{align*}
On the event $E_t^{(0)}(\eta)$, Theorem~\ref{thm:deterministic_witness_survival}
implies $J_p(S_t)\ge \eta$, and therefore Theorem~\ref{thm:monotonicity} yields
$J_p(S_t^\star)\ge \eta$.
The same tower-property argument used in
Theorem~\ref{thm:threshold_attainment_conservative} then gives the stated
finite-horizon product bound.
\end{proof}

\begin{remark}
Theorem~\ref{thm:threshold_attainment_ranking} is the stronger conditional refinement
in this appendix. Unlike the conservative bridge theorem, it explains success in
terms of the actual REUSE mechanism: feasibility-first ordering, stage scores,
top-$B_s$ truncation, and the size of the competitive frontier. The tradeoff is
that its assumptions are correspondingly stronger and more algorithm-specific.
In particular, Assumption~\ref{ass:rank_consistent_top} is part of this stronger
refinement only; it is not claimed as a property of every possible implementation
of the stage-selection operator.
\end{remark}

\subsection{A simple property of the local-consistency statistics}

For completeness, we also record a basic bound for the local-consistency quantities
used in the appendix analysis of the shared input space.

\begin{proposition}[Range of the local-consistency statistics]
\label{prop:consistency_range}
Suppose that $\mathcal{K}\neq\emptyset$, that $|\mathcal{I}_m|\ge 2$, and that for every
neighborhood size $k \in \mathcal{K}$ one has
\[
1 \le k \le |\mathcal{I}_m|-1.
\]
Then the local-consistency statistics defined in the appendix satisfy
\[
\begin{aligned}
0 &\le S_i^{(k)} \le 1,
&\qquad 0 &\le O_i^{(k)} \le 1, \\
0 &\le \bar S_m \le 1,
& 0 &\le \bar O_m \le 1.
\end{aligned}
\]
\end{proposition}

\begin{proof}
By definition,
\begin{align*}
S_i^{(k)}
&=
\frac{1}{k}\sum_{j\in \mathcal{N}_k(i)} \mathbf{1}[c_j=c_i],
\\
O_i^{(k)}
&=
\frac{1}{k}\sum_{j\in \mathcal{N}_k(i)} \mathbf{1}[y_j=y_i].
\end{align*}
Each summand is an indicator variable and therefore belongs to $\{0,1\}$.
Hence each average lies in the interval $[0,1]$, which proves
\[
0 \le S_i^{(k)} \le 1,
\qquad
0 \le O_i^{(k)} \le 1.
\]

Now $\bar S_m$ and $\bar O_m$ are averages of the quantities
$S_i^{(k)}$ and $O_i^{(k)}$, respectively, over indices
$i\in \mathcal{I}_m$ and $k\in \mathcal{K}$.
An average of numbers in $[0,1]$ also lies in $[0,1]$.
Therefore
\[
0 \le \bar S_m \le 1,
\qquad
0 \le \bar O_m \le 1.
\]
\end{proof}

\subsection{Summary}

The results proved above provide two layers of guarantees for REUSE.
At the conservative level:
(i) all stage-wise candidate pools are finite and therefore well-defined;
(ii) under the exact-selector instantiation of
Assumption~\ref{ass:scalarized_final_selector}, any nonempty returned panel and any
nonempty incumbent panel satisfy the stated feasibility and diversity constraints;
(iii) under the same exact-selector instantiation, if the terminal survivor pool
already contains a good feasible witness set, then the panel-construction step
returns a set with at least the same utility, and the incumbent-update step preserves
that quality thereafter;
(iv) under the exact-selector incumbent-update rule, the incumbent utility is
monotone nondecreasing; and
(v) any measurable good latent region with positive prior mass remains
sampling-reachable with nonzero probability at every iteration and is hit with
probability tending to one under an unbounded total search budget, but this
reachability statement alone does not imply survival through stage-wise filtering
or exact recovery by the final panel; and
(vi) after adding the abstract survivability bridge
Assumption~\ref{ass:iteration_survivability}, the incumbent reaches any fixed
utility threshold $\eta$ with explicitly controlled finite-horizon probability and,
under unbounded total search budget, with probability tending to one.

At the stronger algorithm-specific level:
(vii) if a witness set in the pooled candidate family remains sufficiently highly
ranked under every stage-wise feasibility-first selection operator
$\mathrm{Top}_{B_s}(\cdot;\succ_s)$, then that witness survives deterministically to
the terminal pool and is recovered by the exact final selector; and
(viii) after adding the latent-to-ranking bridge
Assumption~\ref{ass:latent_to_ranking}, this ranking-preservation mechanism yields a
stronger conditional finite-horizon threshold-attainment guarantee for the
best-so-far incumbent.

Thus the appendix now separates a safest threshold-attainment theorem chain from a
harder, more algorithm-specific refinement. Neither layer claims global optimality,
and neither should be interpreted as proving that latent reachability alone implies
final recovery.

\bibliography{aaai2027}

\end{document}